\documentclass[final]{article}

\PassOptionsToPackage{numbers, compress}{natbib}
\usepackage{neurips_2020}

\usepackage{graphicx}
\usepackage[export]{adjustbox}
\usepackage{amsmath,amssymb,amsthm,enumerate,color,ifthen,graphicx,overpic,bm}
\usepackage{xargs}

\usepackage[colorlinks=true,
linkcolor=blue,
urlcolor=blue,
citecolor=blue]{hyperref}

\usepackage[utf8]{inputenc}

\usepackage{cleveref}

\usepackage[linesnumbered,vlined,boxed,commentsnumbered]{algorithm2e}
\usepackage[shortlabels]{enumitem}
\usepackage{dashundergaps}
\usepackage{aliascnt,bbm}

\usepackage{tabularx}
\usepackage{wrapfig}

\makeatletter
\newcommand{\settitle}{\@maketitle}
\makeatother

\newtheorem{assumption}{H\hspace{-3pt}}

\newcommand{\ps}[2]{#1^{\top}#2}
\def\nset{{\mathbb{N}}}
\def\rset{\mathbb R}

\def\rmd{\mathrm{d}}
\def\argmin{\operatorname{Argmin}}

\def\max{\mathrm{max}}
\def\1{\mathbbm{1}}
\newcommand{\un}{\ensuremath{\mathbbm{1}}}
\def\PE{\mathbb{E}} 
\def\PVar{\mathbb{Var}} 
\newcommand{\pscal}[2]{\left\langle#1,#2\right\rangle}
\newcommand{\eqdef}{\ensuremath{\stackrel{\mathrm{def}}{=}}}
\newcommand{\kmax}{k_\mathrm{max}}
\newcommand{\kouter}{k_\mathrm{out}}
\newcommand{\kin}{k_\mathrm{in}}

\newcommand{\kswitch}{\mathrm{kswitch}}

\newcommand\curr{\mathrm{curr}}
\newcommand\init{\mathrm{init}}
\def\PVar{\operatorname{Var}}
\def\PCov{\operatorname{Cov}}
\newcommand{\R}{\mathsf{R}}
\def\Zset{\mathsf{Z}}
\def\Rset{\mathbb{R}}
\newcommand\Sset{{\Rset^q}}
\def\Zsigma{\mathcal{Z}}

\newcommand{\loss}[1]{\ensuremath{\mathcal{L}_{#1}}}
\newcommand{\Q}[1]{\ensuremath{\mathsf{Q}_{#1}}}
\newcommand{\bars}{\bar{s}}

\newcommand{\s}{s}
\newcommand{\hatS}{\widehat{S}}
\newcommand{\Sronde}{\widetilde{S}}
\newcommand{\Smem}{\mathsf{S}}

\newcommand{\param}{\theta}
\newcommand{\Param}{\Theta}
\newcommand{\map}{\mathsf{T}}
\newcommand{\F}{\mathcal{F}} 
\newcommand{\G}{\mathcal{G}} 

\newcommand\sequence[3] {\ifthenelse{\equal{#3}{}}{\ensuremath{\{
#1_{#2}\}}}{\ensuremath{\{ #1^{#2}, \eqsp #2 \in #3 \}}}}
\newcommand\sequencedown[3] {\ifthenelse{\equal{#3}{}}{\ensuremath{\{
#1_{#2}\}}}{\ensuremath{\{ #1_{#2}, \eqsp #2 \in #3 \}}}}
\newcommand{\lyap}{\operatorname{W}}

\def\Id{\mathrm{I}}

\def\eqsp{\;}

\newcommand{\ie}{i.e.}

\newcommand{\batch}{\mathcal{B}}
\newcommand{\lbatch}{\mathsf{b}}
\def\A{\mathsf{A}}

\newcommand{\pas}{\gamma}
\def\param{\theta}

\def\pa{\mathsf{a}}

\def\pc{\mathsf{c}}
\def\pd{\mathsf{d}}

\newcommand{\ooint}[1]{\left(#1\right)}
\newcommand{\ccint}[1]{\left[#1\right]}

\newtheorem{theorem}{Theorem}
\newaliascnt{proposition}{theorem}
\newtheorem{proposition}[proposition]{Proposition}
\aliascntresetthe{proposition}
\newaliascnt{lemma}{theorem}
\newtheorem{lemma}[lemma]{Lemma}
\aliascntresetthe{lemma}
\newaliascnt{corollary}{theorem}
\newtheorem{corollary}[corollary]{Corollary}
\aliascntresetthe{corollary}

\newaliascnt{definition}{theorem}

\aliascntresetthe{definition}
\newaliascnt{remark}{theorem}

\aliascntresetthe{remark}

\newcommand{\beq}{\begin{equation}}
\newcommand{\eeq}{\end{equation}}

\def\grad{{\bm g}}

\def\grad{\nabla}

\title{A Stochastic Path-Integrated Differential EstimatoR Expectation Maximization Algorithm}

\author{
Gersende Fort \\
Institut de Mathématiques de Toulouse \\
Université de Toulouse; CNRS \\
UPS, Toulouse, France \\
\texttt{gersende.fort@math.univ-toulouse.fr} \\
\And
Eric Moulines \\
Centre de Math\'{e}matiques Appliqu\'{e}es  \\
Ecole Polytechnique, France\\
CS Dpt, HSE University, Russian Federation \\
\texttt{eric.moulines@polytechnique.edu} \\
\And
Hoi-To Wai \\
Department of SEEM\\
The Chinese University of Hong Kong\\
Shatin, Hong Kong\\
\texttt{htwai@cuhk.edu.hk} \\
}

\begin{document}

\maketitle

\begin{abstract}
  The Expectation Maximization (EM) algorithm is of key importance for
  inference in latent variable models including mixture of regressors
  and experts, missing observations. This paper introduces a novel EM
  algorithm, called \texttt{SPIDER-EM}, for inference from a training
  set of size $n$, $n \gg 1$. At the core of our algorithm is an
  estimator of the full conditional expectation in the {\sf E}-step,
  adapted from the stochastic path-integrated differential estimator
  ({\tt SPIDER}) technique.  We derive finite-time complexity bounds
  for smooth non-convex likelihood: we show that for convergence to an
  $\epsilon$-approximate stationary point, the complexity scales as
  $K_{\operatorname{Opt}} (n,\epsilon )={\cal O}(\epsilon^{-1})$ and
  $K_{\operatorname{CE}}( n,\epsilon ) = n+ \sqrt{n} {\cal O}(
  \epsilon^{-1} )$,
  where $K_{\operatorname{Opt}}( n,\epsilon )$ and
  $K_{\operatorname{CE}}(n, \epsilon )$ are respectively the number of
  {\sf M}-steps and the number of per-sample conditional expectations
  evaluations.  This improves over the state-of-the-art
  algorithms. Numerical results support our findings.
\end{abstract}

\hrule
\medskip

This paper is close to the final version accepted for
  publication in the Conference on Neural Information Processing
  Systems (NeurIPS 2020). The final version can be found at \\
  https://papers.nips.cc/paper/2020/hash/c589c3a8f99401b24b9380e86d939842-Abstract.html

  \medskip
  
  \hrule 

\section{Introduction}
Expectation Maximization (EM) is a key algorithm in machine-learning and statistics \cite{maclachlan:2008}. Applications are numerous including clustering, natural language processing, parameter estimation in mixed models, missing data, to give just a few.  The common feature of all these applications is the introduction of latent variables: the ``incomplete" likelihood $p(y;\param)$ where $\param \in \Param \subseteq \rset^d$ is defined by marginalizing the ``complete-data" likelihood $p(y,z;\param)$ defined as the joint distribution of the observation $y$ and a non-observed latent variable $z \in \Zset$, \ie\ $p(y;\param) = \int p(y,z;\param) \mu(\rmd z)$ where $\Zset$ is the latent space and $\mu$ is a measure on $\Zset$. We focus in this paper on the case where $p(y,z;\theta)$ belongs to a curved exponential family, given by
\begin{equation}
p (y, z;\param) \eqdef \rho( y,z) \exp \big\{ \pscal{ s(y,z) }{ \phi( \param) } - \psi( \param) \big\} \eqsp;
\end{equation}
where $s(y,z) \in \rset^q$ is the complete data sufficient statistics, $\phi: \Param \to \rset^q$ and $\psi : \Param \rightarrow \rset$, $\rho: {\sf Y} \times \Zset \rightarrow \rset^+$ are vector/scalar functions.
Given a training set of $n$ independent
observations $\{ y_i \}_{i=1}^n$, our goal is to minimize the negated penalized log-likelihood
with respect to $\param \in \Param$:
\begin{equation} \label{eq:intro:F}
{\displaystyle \min_{\param \in \Param}} ~F(\param) \eqdef \frac{1}{n} \sum_{i=1}^n \loss{i}(\param) + \R(\param),~~ \loss{i}(\param) \eqdef - \log p(y_i;\theta),
  \end{equation}
such that $\R(\param)$ is a  regularizer. A popular solution approach to \eqref{eq:intro:F} is the EM algorithm \citep{Dempster:em:1977} which is a special instance of the Majorize-Minimization (MM) algorithm. It alternates between two steps: in the Expectation ({\sf
  E}) step, using the current value of the iterate $\param_\curr$, we compute a majorizing function $\param \mapsto
\Q{}(\param,\param_\curr)$ given up to an additive constant by
\begin{equation} \label{eq:majorize}
\Q{}(\param, \param_\curr) \eqdef - \pscal{ \bars( \param_\curr) }{
  \phi( \param ) } + \psi(\param) + \R( \param ) \quad \text{where} \quad \bars(
\param ) \eqdef \frac{1}{n} \sum_{i=1}^n \bars_i( \param ) \eqsp;
\end{equation}
and  $\bars_i( \param )$ is the $i$th sample conditional expectation of the complete data sufficient statistics:
\begin{equation}
\label{eq:definition-bar-s}
\bars_i( \param ) \eqdef \int_{\Zset} s(y_i,z) p( z \vert y_i; \param ) \mu( \rmd z ) \eqsp, \quad p(z|y_i; \param) \eqdef p(y_i,z;\param)/ p(y_i; \param) \eqsp.
\end{equation}
As for the Maximization ({\sf M}) step, a new value of $\param_\curr$
is computed as a minimizer of
$\param \mapsto \Q{}(\param, \param_\curr)$. The majorizing function is then updated with the new $\param_\curr$. This process is iterated until convergence. One of the distinctive advantage of EM algorithms with respect to (w.r.t.) first-order methods stems from the fact that it is invariant by change of parameterization and that EM is, by construction, monotone; see \cite{maclachlan:2008}.

The conventional EM algorithm is not suitable for analyzing the  increasingly large data sets, such as those that
could be considered as big data in volumes  \cite{buhlmann2016handbook,hardle2018handbook}: in such case, the explicit computation of $\bars(\param_\curr)$ in \emph{each {\sf E}-step} of the EM algorithm involves evaluating $n$ conditional expectations \cite{maclachlan:2008}.
As a remedy, \emph{incremental} methods were designed which reduce the number of samples used per iteration to a mini-batch.
Among the incremental methods, the first approach to cope with large-scale EM setting is the incremental EM ({\tt iEM}) algorithm \citep{Neal:hinton:1998} (also see \citep{Ng:mclachlan:2003} for a refined algorithm). At each iteration, {\tt iEM} selects a minibatch ${\cal B}_\curr$ of size $\lbatch$ and updates the associated statistic $\bars_i(\param_\curr), i \in {\cal B}_\curr$, in the current estimate $\hatS_\curr$ of $\bars(\param_\curr)$; and then updates the parameters by a classical {\sf M}-step. Later, an alternative approach was proposed in \citep{cappe:moulines:2009} as the {\tt Online EM} algorithm, which shares some similarities with stochastic gradient descent \cite{bottou:lecun:2003} even though {\tt Online EM} is not a first-order method. Recent papers have proposed improvements to {\tt Online EM} by combining it with variance reduction techniques.  For instance, \citep{chen:etal:2018} and \citep{karimi:etal:2019} proposed respectively the stochastic EM with variance reduction ({\tt sEM-vr}) and the fast incremental EM ({\tt FIEM}) algorithms. These methods are extensions to the EM algorithm of the SVRG \citep{johnson:zhang:2013} and the SAGA \cite{Defazio:bach:2014} techniques.

The complexity of these algorithms have been analyzed under the assumption that $F(\param)$ is smooth but possibly non-convex. They are expressed as the number of {\sf M}-steps updates, $K_{\operatorname{Opt}}(n, \epsilon)$, and the number of per-sample  conditional expectations evaluations $K_{\operatorname{CE}}(n,\epsilon)$, in order to find an $\epsilon$-approximate stationary point of $F(\param)$; see \eqref{eq:epsilon-stationary} for the definition.  It was established in \citep{karimi:etal:2019} that $K_{\operatorname{Opt}}(n, \epsilon) = K_{\operatorname{CE}}(n,\epsilon) = n + n^{2/3} {\cal O}( \epsilon^{-1} )$ updates/evaluations are needed for the {\tt sEM-vr} and {\tt FIEM} algorithms (the rate for {\tt FIEM} can be sharpened, see \citep{fort:gach:moulines:2020}).  These complexity bounds match those of the SVRG and the SAGA algorithms for smooth non-convex optimization \citep{reddi:etal:2016}.

For smooth non-convex problems, the Stochastic Path-Integrated Differential EstimatoR {\tt (SPIDER)} technique has recently been introduced by \cite{fang:etal:2018} (see also \cite{wang:etal:nips:2019} for {\tt SPIDER-BOOST} and \citep{nguyen:liu:etal:2017} for {\tt SARAH}), which established an $n + \sqrt{n} {\cal O}( \epsilon^{-1} )$ bound of calls to first order oracles to find an $\epsilon$-approximate stationary solution of a general finite sum optimization problem. Furthermore, the $\sqrt{n}$-dependence was proven to be optimal. This motivates the current work to explore new EM algorithms with reduced complexity. Our contributions are:
\begin{itemize}[leftmargin=5mm, itemsep=0pt]
    \item We propose a novel {\tt SPIDER-EM} algorithm, inspired by the {\tt SPIDER} estimator in
        \citep{fang:etal:2018} and tailored to the EM framework for
        curved exponential family class of distributions. The {\tt
        SPIDER-EM} uses an outer loop to maintain a {\em control
        variate} that requires a full scan of the dataset to compute
      $\bars( \param_\curr )$, and inner loops which perform low
      complexity updates by drawing random minibatches of samples.
    \item We introduce a unified framework of \emph{stochastic
      approximation (SA) within EM} which covers the convergence
      analysis of {\tt Online EM}, {\tt sEM-vr}, {\tt FIEM}, {\tt
        SPIDER-EM}. In this general framework, {\tt SPIDER-EM} may be seen as a
      stochastic approximation algorithm using variance reduced
      estimate $\hatS_\curr$.
    \item Using the SA analysis framework, we prove that the
      complexity bounds for {\tt SPIDER-EM} are
      $K_{\operatorname{Opt}}(n,\epsilon) = {\cal O}(\epsilon^{-1})$,
      $K_{\operatorname{CE}}(n,\epsilon) = n + \sqrt{n}{\cal
        O}(\epsilon^{-1})$.
      Among the incremental-EM techniques, we
        provide state of the art complexity bounds that overpass all the previous ones.
    \item The EM is \text{not} a first-order method contrary to {\tt SPIDER}.
    Therefore, the convergence analysis of {\tt SPIDER-EM} methods require specific mathematical developments which differ significantly from the original {\tt SPIDER} analysis. In addition, the analysis of {\tt SPIDER-EM} differs from previous ones for incremental EM algorithms, since it involves {\em biased} approximations, which makes the proof more challenging (see \autoref{sec:proof:maintheo}, Lemma~\autoref{lem:field:bias}).
    \item We provide a new perspective to interpret {\tt SPIDER-EM} as an equivalent algorithm to a perturbed {\tt Online-EM} where the perturbation acts as a control variate to reduce variance - see \autoref{algo:SPIDER-EMbis}.
\end{itemize}
Furthermore, the {\tt SPIDER-EM} algorithm operates with a
significantly lower memory footprint than {\tt iEM} and {\tt FIEM},
and the memory footprint is on par with {\tt sEM-vr} and {\tt Online
  EM}. To our best knowledge, the proposed algorithm offers the best
of both worlds -- having a low complexity bounds and a low memory
footprint. Lastly, we support the theoretical findings with numerical
experiments and show that {\tt SPIDER-EM} performs favorably compared
to existing algorithms.

{\bf Notations.} For two vectors $a,b \in \rset^r$, $\pscal{a}{b}$
denotes the usual Euclidean product and $\|a \|$ the associated norm. By convention, vectors are
column vectors. For a vector $x$ with components $(x_1, \ldots, x_r)$,
$x_{i:j}$ denotes the sub-vector with components $(x_i, x_{i+1},
\ldots, x_{j-1}, x_j)$.
For two matrices $A \in \rset^{r_1 \times r_2}$ and $B \in \rset^{r_3 \times r_4}$, $A \otimes B$ denotes the Kronecker product. $\Id_r$ is the $r \times r$ identity matrix. $A^T$ is the transpose of $A$.

\section{EM Algorithm and its Variants using Stochastic Approximation}
We formulate the model assumptions and introduce the {\tt SPIDER-EM} algorithm. Recall the definition of the negated penalized log-likelihood $F(\param)$ from \eqref{eq:intro:F} and consider a few regulatory assumptions:
\begin{assumption} \label{hyp:model} $\Param \subseteq \rset^d$ is a measurable convex set.  $(\Zset, \Zsigma)$ is a measurable space and $\mu$ is a $\sigma$-finite positive measure on $\Zsigma$. The functions $\R: \Param \to \rset$, $\phi : \Param \to \rset^q$, $\psi: \Param \to \rset$,  and $\rho(y_i,\cdot): \Zset \to \rset_+$, $\s(y_i,\cdot): \Zset \to \rset^q$ for $i \in \{1, \ldots, n\}$ are measurable functions. For any $\param \in \Param$ and $i \in \{1,\ldots, n\}$, the log-likelihood is bounded as $-\infty<\loss{i}(\param) < \infty$.
\end{assumption}
\begin{assumption} \label{hyp:bars} For all $\param \in \Param$ and
  $i \in \{1, \ldots, n\}$, the conditional expectation
  $\bars_i(\param)$ is well-defined.
\end{assumption}
\begin{assumption} \label{hyp:Tmap}  For any $s \in \Sset$, the map $s \mapsto \argmin_{\param \in \Param} \ \left\{  \psi(\param) + \R(\param) -  \pscal{s}{\phi(\param)} \right\}$ exists and is unique; the singleton is denoted by $\{\map(s)\}$.
\end{assumption}

As discussed in the Introduction, the EM algorithm is an MM algorithm associated with the majorization functions $\{\param \mapsto \Q{}(\param, \param_\curr), \param_\curr \in \Param\}$. Thus, the EM algorithm defines a sequence $\{\param_k, k \geq 0 \}$ that can be computed recursively as $ \param_{k+1} = \map \circ \bars (\param_k)$,
where the map $\map$ is defined in \Cref{hyp:Tmap} and $\bars$ is defined in \eqref{eq:majorize}.
On the other hand, the EM algorithm can be defined through a mapping in the complete data sufficient statistics, referred to as the  {\em expectation space}. In this setting, the EM iteration defines a sequence in
$\rset^q$ $\{\hatS_k, k \geq 0 \}$ given by $ \hatS_{k+1} = \bars \circ \map(\hatS_k)$. To summarize, we observe that the EM algorithm admits two equivalent representations:
\begin{equation}
\label{eq:EM-equiv-param-stat}
    \text{(Parameter space)}~~\param_{k+1} = \map \circ \bars (\param_k); \quad \text{(Expectation space)}~~\hatS_{k+1} = \bars \circ \map(\hatS_k).
\end{equation}
In this paper, we focus on the expectation space representation.
Let $\param_\star \eqdef \map ( s_\star )$ where $s_\star \in \rset^q$. It has been shown in \citep{Delyon:etal:1999} that if $s_\star$ is a fixed point to the EM algorithm in the expectation space, then $\param_\star = \map(s_\star)$ is a fixed point of the EM algorithm in the parameter space, i.e., $\param_\star = \map \circ \bars( \param_\star )$. Note that the converse is also true.
The limit points of the EM algorithm in the expectation space are the  roots of the {\em mean field}
\begin{equation}
  \label{eq:def:h:meanfield}
{h(s) \eqdef \bars \circ \map(s) - s, \quad s \in \rset^q}
  \eqsp.
\end{equation}
Consider the following assumption.
\begin{assumption} \label{hyp:regV}
\begin{enumerate}[itemsep=0pt, leftmargin=5mm]
\item \label{hyp:model:C1} The functions $\phi, \psi$ and $\R$ are continuously differentiable on $\Param^v$. If $\Param$ is open, then $\Param^v = \Param$, otherwise $\Param^v$ is a neighborhood of $\Param$. $\map$ is continuously differentiable on $\Sset$.
\item \label{hyp:model:F:C1} The function $F$ is continuously differentiable on $\Param^v$ and for any $\param \in \Param$, $\grad F(\param) = -  \ps{\grad{\phi}(\param)}{\bars(\param)} + \grad \psi(\param) + \grad \R(\param)$.
\item \label{hyp:regV:C1} For any $s \in \Sset$, $B(s) \eqdef \grad (\phi \circ \map)(s)$ is a symmetric matrix with positive minimal eigenvalue.
\end{enumerate}
\end{assumption}
These assumptions are classical, see for example, \citep{karimi:etal:2019} and the references therein.

A key property of the EM algorithm is that it is \emph{monotone}: in the parameter space $\theta_{k+1}= \map \circ \bars(\theta_k)$ decreases the objective function with $F(\param_{k+1}) \leq F(\param_k)$. The same monotone property also holds in the expectation space. Define
\begin{equation} \label{eq:wdef}
\lyap (s) \eqdef F \circ \map(s) = \frac{1}{n} \sum_{i=1}^n {\cal L}_i( \map(s) ) +  \R(\map(s)), \quad \s \in \rset^q \eqsp.
\end{equation}
It can be shown that $F(\param_{k+1}) \leq F(\param_k)$ implies $\lyap (\hat{S}_{k+1}) \leq \lyap (\hat{S}_k)$. In addition, \cite{Delyon:etal:1999} showed that:
\begin{proposition}
  \label{prop:fixed:to:stationary}
  Under \Cref{hyp:model}, \Cref{hyp:bars}, \Cref{hyp:Tmap}
  and \Cref{hyp:regV}, $\lyap(s)$ is continuously differentiable on $\rset^q$ and for any $s \in \rset^q$, $\grad \lyap(s) = - B(s) \,
  h(s)$.
\end{proposition}
Hence, $s_\star$ is a fixed point to the EM algorithm in expectation space, with $s_\star = \bars \circ \map(s_\star)$ and $h(s_\star)=0$ if and only if $s_\star$ is a stationary point satisfying $\grad \lyap (s_\star) = 0$.  This property has made it possible to develop a new class of algorithms that preserve  desirable properties of the EM (e.g, invariant in the choice of parameterization) while replacing the computation of $\bars( \param)$ by a stochastic approximation (SA) scheme; see \cite{robbins1951stochastic,benveniste:etal:1990,borkar:2008} for a survey on SA. This scheme has been exploited in \cite{Delyon:etal:1999} to deal with the case where the computation of the conditional expectation $\bars(\param)$ is intractable.

We consider yet another form of intractability in this work which is linked with the size of the dataset $n \gg 1$.  To alleviate this problem, the {\tt Online EM} algorithm \citep{cappe:moulines:2009} defines a sequence $\{\hatS_k, k \geq 0\}$ with the recursion:
\begin{equation}\label{eq:onlineEM}
  \hatS_{k+1} = \hatS_k + \pas_{k+1} \left(  \bars_{\batch_{k+1}} \circ \map(\hatS_k) - \hatS_k \right)
  \eqsp,
\end{equation}
where $\{\pas_{k+1}, k \geq 0 \}$ is a deterministic sequence of step sizes, $\batch_{k+1}$ is a mini-batch of $\lbatch$ examples sampled at random in $\{1, \dots, n\}$ and for a mini-batch $\batch$ of size $\lbatch$, we set $\bars_\batch  \eqdef \lbatch^{-1} \sum_{i \in \batch}  \bars_i$.

The {\tt Online EM} algorithm can be viewed as an SA scheme designed for finding the roots of the mean-field $h$; indeed, the mean-field of {\tt Online EM} satisfies $\PE [\bars_{\batch_{k+1}} \circ \map(\hatS_k) - \hatS_k ] = h( \hatS_k )$.
Hence, the possible limiting points of {\tt Online EM} are the roots of $h(s)$,  such a root $s_\star$ is a stationary point of $\lyap$ (see \Cref{prop:fixed:to:stationary} and \eqref{eq:wdef}), and $\map(s_\star)$ corresponds to a stationary point of the penalized likelihood \eqref{eq:intro:F}; see \cite{cappe:moulines:2009} for a precise statement and \cite{Karimi:miasojedow:2019} for a detailed convergence analysis.

\noindent \textbf{Variance Reduction for SA with EM Algorithm.}
For the finite-sum problem \eqref{eq:intro:F}, more efficient algorithms can be developed by introducing a control variate in order to achieve variance reduction. Suppose that we have a random variable (r.v.) $U$ and our aim is to estimate $u \eqdef \PE[U]$. For any zero-mean r.v. $V$, the sum $U+V$ is an unbiased estimator of $u$. Now, if $V$ is negatively correlated with $U$ and $\PVar(V^2) \leq -2 \PCov(U,V)$, then the variance of $U+V$ will be lower than that of the standalone estimator $U$; $V$ is a {\em control variate}.

This approach has been proven to be effective for stochastic gradient algorithms: emblematic examples are Stochastic Variance Reduced Gradient (SVRG) introduced by \cite{johnson:zhang:2013} and SAGA introduced by  \cite{Defazio:bach:2014}. Whereas control variates have been originally designed to the stochastic gradient framework,  similar ideas can be applied to SA procedures for finite-sum optimization. For {\tt Online EM}, variance reduction amounts to expressing the mean-field as $h(s) = \PE\left[ \bars_\batch \circ \map(s) - s + V \right]$ where $V$ is a control variate. These methods differ in the way the control variate is constructed. The efficiency of such variance reduction methods improves with the correlation of $V$ with  $\bars_\batch \circ \map(s) - s$.

An SVRG-like algorithm is the {\tt Stochastic EM with Variance Reduction (sEM-vr)} algorithm \cite{chen:etal:2018}. In {\tt sEM-vr}, the control variate is reset in an outer loop every
$\kin$ iterations: in the outer loop $\#t$ for $t \in \{1, \dots, \kouter\}$,  and the
inner loop  $\#(k+1)$ for $k \in \{ 0, \dots, \kin -2\}$, the complete
data sufficient statistic is updated using {\tt Online EM} and a
recursively defined control variate
\begin{align}
 \hatS_{t,k+1} &= \hatS_{t,k} + \pas_{t,k+1} (  \bars_{\batch_{t,k+1}} \circ \map(\hatS_{t,k}) - \hatS_{t,k} +
    V_{t,k+1}) \eqsp,  \label{eq:SVREMandFIEM}  \\
    V_{t,k+1} &  =\bars \circ \map (\hatS_{t-1,\kin-1}) -\bars_{\batch_{t,k+1}}  \circ
\map (\hatS_{t-1,\kin-1}) \eqsp.  \label{eq:sEMVR_ctrl}
\end{align}
When $k = 0$, the complete data sufficient statistic $\hatS_{t,0}$
is obtained by performing first a full-pass on the dataset
$\Sronde_{t,0} = \bars \circ \map(\hatS_{t-1,\kin-1})$ and then
updating $ \hatS_{t,0} = \hatS_{t-1, \kin-1} + \pas_{t,0} (
\Sronde_{t,0} - \hatS_{t-1, \kin-1} )$.
An SAGA-like version is the {\tt Fast Incremental EM (FIEM)} algorithm proposed in \cite{karimi:etal:2019}. The construction of the control variate for {\tt FIEM} is more involved; for details, see \autoref{algo:FIEM} in the supplementary material.

In \citep{karimi:etal:2019}, the {\tt sEM-VR} and {\tt FIEM}
algorithms have been analyzed with a randomized terminating iteration
$(\tau,\xi)$, uniformly selected from $\{1,\dots, \kouter\} \times \{0, \dots, \kin-1\}$ where
$\kin$ (resp. $\kouter$) is the number of inner loops per outer one,
and $\kouter$ is the total number of outer loops. The random
termination  is inspired by \cite{ghadimi:lan:2013} which enables
one to show non-asymptotic convergence of stochastic gradient methods to a stationary point. Consider
first {\tt sEM-VR}. For any $n, \epsilon$, we define
$\mathcal{K}(n,\epsilon) \subset \nset^3$ such that, for any
$(\kin,\kouter,\lbatch) \in \mathcal{K}(n,\epsilon)$,
\begin{equation}
\label{eq:epsilon-stationary}
\textstyle{\PE[ \|h(\hatS_{\tau,\xi})\|^2 ]
\eqdef {\kmax}^{-1} \sum_{t=1}^{\kouter} \sum_{k=0}^{\kin-1} \PE[ \|h(\hatS_{t,k})\|^2 ] \leq \epsilon \,,}
\end{equation}
where $\kmax= \kin \kouter$. In words, the randomly terminated
algorithm computes a solution $\hatS_{\tau,\xi}$ such that the expected
squared norm of the mean field is less than $\epsilon$; see
\cite{ghadimi:lan:2013}. The finite sample complexity in terms of the
number of {\sf M}-steps is
$K^{\texttt{sEM-VR}}_{\operatorname{Opt}}(n,\epsilon)=
\inf_{\mathcal{K}(n,\epsilon)} \kin \kouter$.

The complexity in terms of the total number of per-sample conditional
expectations evaluations, is defined as
$K^{\texttt{sEM-VR}}_{\operatorname{CE}}(n,\epsilon,\lbatch)=
\inf_{\mathcal{K}(n,\epsilon)} \{ n + \kouter n + \lbatch \kin \kouter
+ (n \wedge (\lbatch \kin)) \kouter\}$.  Similar results can be
derived for \texttt{FIEM} and other incremental EM algorithms (see
\autoref{sec:table:comparison}).  In such case, define by $\kmax=
\kmax(n,\epsilon)$ the minimal number of iterations such that
\eqref{eq:epsilon-stationary} is satisfied and set
$K^{\texttt{FIEM}}_{\operatorname{Opt}}(n,\epsilon)=
\kmax(n,\epsilon)$ and
$K^{\texttt{FIEM}}_{\operatorname{CE}}(n,\epsilon) = 2
\kmax(n,\epsilon) \lbatch$.  It can be shown (see
\citep{karimi:etal:2019} and the supplementary material) that
$K^{\texttt{sEM-VR}}_{\operatorname{Opt}}(n,\epsilon)=K^{\texttt{FIEM}}_{\operatorname{Opt}}(n,\epsilon)
= n^{2/3} {\cal O}( \epsilon^{-1})$ and $
K^{\texttt{sEM-VR}}_{\operatorname{CE}}(n,\epsilon) =
K^{\texttt{FIEM}}_{\operatorname{CE}}(n,\epsilon) = n+ n^{2/3} {\cal
  O}(\epsilon^{-1})$.  These bounds exhibit an ${\cal
  O}(\epsilon^{-1})$ growth as the stationarity requirement $\epsilon$
decreases. Such a rate is comparable to a deterministic gradient
method for smooth and non-convex objective functions. However, the
complexity of {\sf M}-step computations as well as of conditional
expectations evaluations grow at the rate of $n^{2/3}$, which can be
undesirable if $n \gg 1$. Hereafter, we aim to design a novel
algorithm with better finite-time complexities.

\section{The {\tt SPIDER-EM} Algorithm}
\label{sec:main:result}
To reduce the dependence on $n$ and the overall complexity, we propose to design a \emph{new control variate}, and to optimize the size of the  \emph{minibatch}. To this regard, we borrow from \citep{fang:etal:2018,wang:etal:nips:2019} (see also \citep{nguyen:liu:etal:2017} and the algorithm {\tt SARAH}) a new technique called Stochastic Path-Integrated Differential Estimator ({\tt SPIDER}) to generate the control variates for estimating the conditional expectation of the complete data for the full dataset.

\textbf{Algorithm Description.} We propose the {\tt SPIDER-EM} algorithm formulated in the expectation space. The outer loop is the same as that of {\tt sEM-vr}. The difference lays in the update of $\hatS_k$ as follows:
\setlength{\algomargin}{1.5em}
\begin{algorithm}[htbp]
  \KwData{$\kin \in \nset_\star$, $\kouter \in \nset_\star$,
    $\hatS_\init \in \Sset$, $\{\pas_{t,k+1}, t \geq 1, k \geq 0\}$
    positive sequence.}  \KwResult{The {\tt SPIDER-EM} sequence:
    $\hatS_{t,k}, t=1,\ldots, \kouter$ and $ k=0, \ldots, \kin-1$}{
    $\hatS_{1,0} = \hatS_{1,-1} = \hatS_\init$, \quad $\Smem_{1,0} =
    \bars \circ \map(\hatS_{1,-1})$} \;\For{$t=1, \ldots, \kouter$}{
    \For{$k=0,\ldots, \kin-2$} { Sample a mini-batch $\batch_{t,k+1}$
      in $\{1, \ldots, n\}$ of size $\lbatch$, with or without
      replacement\label{line:algo2:startInner}\; $ \Smem_{t,k+1} =
      \Smem_{t,k} + \bars_{\batch_{t,k+1}} \circ \map( \hatS_{t,k} ) -
      \bars_{\batch_{t,k+1}} \circ \map( \hatS_{t,k-1}
      )$ \label{line:algo2:updateSmem} \; $\hatS_{t,k+1} = \hatS_{t,k}
      + \pas_{t,k+1} \big(\Smem_{t,k+1}-
      \hatS_{t,k}\big)$ \label{line:algo2:endInner}} $\hatS_{t+1,-1} =
    \hatS_{t, \kin-1}$ \; $\Smem_{t+1, 0} = \bars \circ
    \map(\hatS_{t+1, -1})$ \; $\hatS_{t+1,0} = \hatS_{t,\kin-1} +
    \gamma_{t,\kin} \big( \Smem_{t+1,0} - \hatS_{t,\kin-1}
    \big)$ \label{line:algo2:updateShatbis}}
    \caption{The {\tt SPIDER-EM} algorithm.\label{algo:SPIDER-EM}}
\end{algorithm}

We discuss the design considerations of the {\tt SPIDER-EM} algorithm and provide insights on how it can accelerate convergence as follows.

\noindent \textbf{Control Variate and Variance Reduction.} We shall analyze {\tt SPIDER-EM} as an SA scheme with control variate to reduce variance. While the description of {\tt SPIDER-EM} algorithm in the above does not present the control variates explicitly, it is possible to re-interpret the inner loop (line \ref{line:algo2:startInner}--line \ref{line:algo2:endInner}) with a control variate defined, for $t \in \nset_\star$ and  $k \in \{0, \dots, \kin - 2 \}$, as\vspace{-.1cm}
\begin{equation} \label{eq:CV:accumulated}
\begin{split}
  V_{t,k+1} & = V_{t,k} + \bars_{\batch_{t,k}} \circ \map ( \hatS_{t,k-1} ) - \bars_{\batch_{t,k+1}} \circ \map( \hatS_{t,k-1} ) \\
  & \textstyle = \sum_{j=0}^{k}  \{ \bars_{\batch_{t,j}} \circ \map ( \hatS_{t,j-1} ) - \bars_{\batch_{t,j+1}} \circ \map( \hatS_{t,j-1} ) \},
\end{split}
\end{equation}
where $V_{t,0} = 0$ is reset at every outer iteration and, by convention, $\batch_{t,0} \eqdef \{1,\ldots,n\}$.
It is seen that line~\ref{line:algo2:endInner} can be rewritten as (see \Cref{lem:equivalent:SPIDER-EM} in the supplementary material)
\begin{equation} \label{eq:spider_ctrl}
\hatS_{t,k+1} = \hatS_{t,k} + \gamma_{t,k+1} \big( \bars_{\batch_{t,k+1}} \circ \map (\hatS_{t,k} ) - \hatS_{t,k} +  V_{t,k+1} \big) \eqsp.
\end{equation}
Note that, by construction, the control variate  $V_{t,k}$ is zero mean because, $\PE[\bars_{\batch_{t,j}} \circ \map ( \hatS_{t,j-1} ) ] =\PE[\bars_{\batch_{t,j+1}} \circ \map( \hatS_{t,j-1} )] = \PE[\bars \circ \map(\hatS_{t,j-1})]$. Eq.~\eqref{eq:CV:accumulated} shows how {\tt SPIDER-EM} constructs a control variate by accumulating information -- similar to {\tt SPIDER} and {\tt SARAH} in the gradient descent setting.

Comparing \eqref{eq:CV:accumulated}-\eqref{eq:spider_ctrl} to \eqref{eq:SVREMandFIEM}-\eqref{eq:sEMVR_ctrl}, the {\tt SPIDER-EM} algorithm differs from  {\tt sEM-vr}  only in the construction of the control variate.
To obtain insights about their performance, let us denote the filtration as $\F_{t,k} \eqdef \sigma(\hatS_\init, \batch_{1,1}, \ldots,  \batch_{1,\kin-1}, \ldots, \batch_{t,1}, \ldots, \batch_{t,k})$. Observe that the conditional variances (given $\F_{t,k}$) of $\hatS_{t,k+1}$  of the {\tt sEM-VR} and {\tt SPIDER-EM} algorithms are:
\[
\begin{split}
\textstyle{\PVar\big[\hatS_{t,k+1}^{\tt sEM-vr} \vert \F_{t,k} \big]} & \textstyle = \pas_{t,k+1}^2 \PVar [ \bars_{\batch_{t,k+1}} \circ \map (\hatS_{t,k}) - \bars_{\batch_{t,k+1}} \circ \map(\hatS_{t-1,\kin-1}) \vert \F_{t,k}] \eqsp,\\[.1cm]
\textstyle{\PVar\big[\hatS_{t,k+1}^{\tt SPIDER-EM}  \vert \F_{t,k} \big]} & \textstyle = \pas_{t,k+1}^2 \PVar [ \bars_{\batch_{t,k+1}} \circ \map (\hatS_{t,k}) - \bars_{\batch_{t,k+1}} \circ \map(\hatS_{t,k -1}) \vert \F_{t,k}] \eqsp.
\end{split}
\]
As a comparison, the variance of $\hatS_{(t-1) \kin +k+1}$ for the
{\tt Online EM} is given by
\[
\gamma_{(t-1)\kin+k+1}^2 \PVar\big[
  \bars_{\batch_{(t-1) \kin +k+1}} \circ \map (\hatS_{(t-1)\kin+k})
  \vert \F^{{\tt O-EM}}_{(t-1)\kin + k} \big].
\]
Here, $\F^{{\tt O-EM}}_{\tau} \eqdef \sigma(\hatS_\init, \batch_1, \ldots,
\batch_\tau)$.  In this sense, both {\tt sEM-vr} and {\tt SPIDER-EM}
are variance-reduced versions of the {\tt Online EM}. Additionally,
{\tt SPIDER-EM} and {\tt sEM-VR} are designed to exploit two values
$\hatS_{t,k}, \hatS_{t,k-1}$ and $\hatS_{t,k}, \hatS_{t-1,\kin - 1}$,
respectively. The former thus takes the benefit of a stronger
correlation between two successive values of $\{\hatS_{t,k}, k \geq 1\}$
than between $\hatS_{t,k}$ and $\hatS_{t-1,\kin -1}$ in the variance
reduction step. As a result,  {\tt SPIDER-EM} should inherit a
better rate of convergence -- an intuition which is established will be \Cref{theo:rate:sqrtn}.

\textbf{Step Size and Memory Footprint.}  The {\tt SPIDER-EM}
algorithm is described with a positive step size sequence
$\{\pas_{t,k+1}, t \geq 1, k \geq 0\}$. Different strategies are
allowed: {\sf (a)} a constant step size $\pas_{t,k+1} = \pas$ for any
$k \geq 0$, or {\sf (b)} a random sequence.  We focus on case {\sf
  (a)} in the following, while we refer the readers to
\cite{fang:etal:2018} for such a strategy in the gradient setting.
Lastly, we observe that the {\tt SPIDER-EM} algorithm has the same
memory footprint requirement as the {\tt sEM-vr} algorithm.

\textbf{Convergence Analysis.}
Let $(\tau,\xi)$ be uniform r.v. on $\{1, \dots, \kouter\} \times \{0,
\dots, \kin -1\}$, independent of the {\tt SPIDER-EM} sequence
$\{\hatS_{t,k}, t=1, \cdots, \kouter; k= -1, \cdots, \kin-1\}$.  Our
goal is to derive explicit upper bounds for $\PE[ \|
  h(\hatS_{\tau,\xi-1}) \|^2 ]$ for the {\tt SPIDER-EM} sequence given
by \autoref{algo:SPIDER-EM} with a constant step size ($\pas_{t,k+1}
=\pas$ for any $t \geq 1$, $k \geq 0$).  We strengthen the assumption
\Cref{hyp:regV} as follows:
\begin{assumption} \label{hyp:regV:bis}
  \begin{enumerate}[label=(\alph*),leftmargin=5mm, itemsep=0pt]
   \item \label{hyp:regV:C1:vmax} There exist $0 < v_{\min} \leq
     v_{\max}< \infty $ such that for all $s\in \Sset$, the spectrum of
     $B(s)$ is in $\ccint{v_{\min}, v_{\max}}$; $B(s)$ is defined in
     \Cref{hyp:regV}.
        \item \label{hyp:Tmap:smooth} For any $i \in \{1, \ldots, n\}$, the map
  $\bars_i \circ \map$ is globally Lipschitz on $\Sset$ with constant
  $L_i$.
      \item \label{hyp:regV:DerLip} The function $s \mapsto \grad
        \lyap(s) = - B(s) h(s)$ is globally Lipschitz on $\Sset$ with
        constant $L_{\grad \lyap}$.
        \end{enumerate}
\end{assumption}
From \Cref{hyp:regV:bis}-\ref{hyp:regV:C1:vmax} and
\Cref{prop:fixed:to:stationary}, we have $\PE\big[ \|h(\hatS_{\tau,\xi-1})\|^2
  \big] \geq v_{\max}^{-2} \PE\big[ \|\grad \lyap(\hatS_{\tau,\xi-1})\|^2
  \big]$ so that a control of $\PE\big[ \|h(\hatS_{\tau,\xi-1})\|^2 \big]$
provides a control of $\PE\big[ \|\grad \lyap(\hatS_{\tau,\xi-1})\|^2
  \big]$. The convergence result for {\tt SPIDER-EM} is summarized
below:

\fbox{\begin{minipage}{.98\linewidth}{
\begin{theorem}
  \label{theo:rate:sqrtn}
  Assume \Cref{hyp:model}, \Cref{hyp:bars},
  \Cref{hyp:Tmap}, \Cref{hyp:regV} and \Cref{hyp:regV:bis} and set $L^2 \eqdef
  n^{-1} \sum_{i=1}^n L_i^2$.
  Fix $\kouter, \kin \in \nset_\star$, $\lbatch \in \nset_\star$ and
  set $\pas_{t,k} \eqdef \alpha/L$ for any $t,k > 0$ where $\alpha \in \ooint{0,v_{\min}/\mu_\star(\kin, \lbatch)}$ with
  \begin{equation}
  \label{eq:mustar-kin=}
  \textstyle{\mu_\star(\kin,\lbatch) \eqdef  v_{\max} \sqrt{\kin/\lbatch}}+ L_{\grad \lyap}/(2L) \,.
  \end{equation}
 The {\tt SPIDER-EM} sequence $\{\hatS_{t,k}, t \geq 1, k \geq 0\}$ given by
\autoref{algo:SPIDER-EM} satisfies
\begin{equation}
  \label{eq:upperbound:theo}
  \notag
{\PE\left[ \|h(\hatS_{\tau,\xi-1})\|^2 \right] \leq \left( \frac{1}{\kin}+
\frac{\alpha^2}{\lbatch}\right) \frac{2 L}{\alpha  \{v_{\min} - \alpha \mu_\star(\kin,\lbatch)\}}
\frac{1}{\kouter} \ \left( \PE[\lyap(\hatS_\init)] -
\min \lyap \right)} \eqsp.
\end{equation}
\end{theorem}}
\end{minipage}
}\vspace{.1cm}

Our analysis, whose detail can be found in the supplementary material, shares some similarities with the one in {\tt SPIDER-Boost} \citep{wang:etal:nips:2019}. Nevertheless, there are a number of differences because {\sf (a)}  {\tt SPIDER-EM} algorithm recursion uses two spaces (the expectation space and the parameter space) which are connected by the maps $\bars$ and $\map$; {\sf (b)} {\tt SPIDER-EM} is not a gradient algorithm in the expectation space,
but an SA scheme to obtain a root for $h$;
{\sf (c)} there is a Lyapunov function $\lyap(s)$ where $\grad{\lyap}(s) \neq -h(s)$, but which  satisfies $\pscal{\grad{\lyap}(s)}{h(s)} \leq -v_{\min} \| h(s) \|^2$.
In addition, in relation to the above points, our analysis took insights from \cite{karimi:lavielle:moulines:2019,Karimi:miasojedow:2019} to analyze {\tt SPIDER-EM} as a biased SA scheme. Our challenge lies in carefully controlling the bias/variance of the {\tt SPIDER} estimator employed, which is not reported in the prior literature.

\textbf{Proof Sketch.}
While we shall omit the proof details, an outline of the proof is provided. Set $H_{t,k+1} \eqdef \pas_{t,k+1}^{-1} (\hatS_{t,k+1} - \hatS_{t,k})$.
A key property is the following descent condition for the Lyapunov function $\lyap$. There exist positive sequences $\Lambda_{t,k}, \beta_{t,k}$ such that  for any $t \geq 1$, $k \geq 0$,
\[
\textstyle{\lyap(\hatS_{t,k+1}) \leq \lyap(\hatS_{t,k})  - \Lambda_{t,k+1} \|H_{t,k+1}\|^2  + \pas_{t,k+1} \frac{v_{\max}^2}{2 \beta_{t,k+1}^2}  \|H_{t,k+1} - h(\hatS_{t,k}) \|^2 }\eqsp.
\]
It holds for any $t \geq 1$ and $0 \leq k \leq \kin-2$,
\begin{equation}
  \label{eq:sketch:perturbationSA}
\textstyle \PE\left[ \|H_{t,k+1} - h(\hatS_{t,k})\|^2 \vert \F_{t-1,\kin-1}  \right]  \leq \frac{L^2}{\lbatch} \sum_{j=0}^k \pas_{t,j}^2  \PE\left[ \|H_{t,j}\|^2 \vert \F_{t-1,\kin-1}  \right] \eqsp.
\end{equation}
The above conditions can be combined to yield
\[
\textstyle \sum_{t=1}^{\kouter} \sum_{k=0}^{\kin-1} A_{t,k} \PE\left[ \|H_{t,k} \|^2 \right]  \leq  \PE\left[ \lyap(\hatS_\init) \right] - \min \lyap
\]
 where the $A_{t,k}$'s are positive. Dividing both sides of the
 inequality by $\sum_{t=1}^{\kouter} \sum_{k=0}^{\kin-1} A_{t,k} $
 leads to a bound on $\PE [ \|H_\Xi\|^2 ]$ for some r.v. $\Xi$ on
 $\{1,\dots,\kouter\} \times \{0, \dots, \kin-1 \}$. For the concerned
 case when $\pas_{t,k} = \pas$, we have $A_{t,k} =A$ and
 $\Xi=(\tau,\xi)$ is the uniform distribution, thus the convergence
 rate for $\PE [ \|H_{\tau,\xi}\|^2 ]$ is ${\cal O}(1/{\kin
   \kouter})$. Lastly, we obtain a bound for the mean field $\| h(
 \hatS_{\tau,\xi-1}) \|^2$ using the standard inequality $(a+b)^2 \leq 2
 a^2 + 2 b^2$ and \eqref{eq:sketch:perturbationSA} again.

\textbf{Choice of $\kin,\lbatch,\kouter$ and Complexity Bounds.} The maximum of  $\alpha \{v_{\min} - \alpha \mu_\star(\kin,\lbatch)\}$ on $\ooint{0,v_{\min}/\mu_\star(\kin,\lbatch)}$ is $\alpha_\star(\kin, \lbatch) \eqdef  v_{\min} / \{2 \mu_\star(\kin,\lbatch)\}$  which yields $ \pas = v_{\min} / \{2 \mu_\star(\kin,\lbatch) L\}$ and the upper bound
\[
\PE\left[ \|h(\hatS_{\tau,\xi-1})\|^2 \right] \leq
\left(\frac{\mu_\star(\kin,\lbatch)}{v_{\min}^2}+ \frac{\kin }{4
  \mu_\star(\kin,\lbatch) \lbatch} \right) \frac{8 L}{\kin \kouter}
(\PE[ \lyap(\hatS_\init)] - \min \lyap) \eqsp.
\]
The number of parameter updates is $1 + \kouter +
\kin \kouter$. The number of per-sample conditional expectation
computations is $n + \kouter n + 2 \lbatch \kin \kouter$. Assume that
$n$ and $\epsilon > 0$ are given. Set for simplicity $\lbatch= \kin=
\lceil \sqrt{n}\rceil$ which means that the number of per-sample
conditional expectations evaluations in the inner loop is equal to
$n$, i.e., is an epoch (see \autoref{sec:choice:designparameters} for a discussion on other strategies). With this choice, we get $
\mu_\star(\kin,\lbatch)=m_\star \eqdef v_{\max} + L_{\grad \lyap}/(2L)$. Taking
\[
\textstyle{
\kouter \geq \left( \frac{m_\star}{v_{\min}^2}+
\frac{1}{4 m_\star} \right)
\frac{8L}{\sqrt{n} \epsilon} \ ( \PE[\lyap(\hatS_\init)] -
\min \lyap )} \eqsp,
\]
then we have $\PE[\| h(\hatS_{\tau,\xi-1}) \|^2] \leq \epsilon$.
With these choices of $\kin, \kouter, \lbatch$, the complexity in terms
of the number of per-sample conditional expectations evaluations
$\bars_i$ is  $K_{\operatorname{CE}}(n,\epsilon) = n +
\sqrt{n} L {\cal O}( \epsilon^{-1} )$. The number of parameter updates is  $ K_{\operatorname{Opt}} (n,\epsilon) = {\cal O}(\epsilon^{-1})$.
Note that the step size is chosen to be $\gamma =  \alpha_\star( \kin, \lbatch )/L$, which is independent of the targeted accuracy $\epsilon$.

\textbf{Linear convergence rate.}
In \autoref{sec:linear:cvgrate}, we provide a modification of {\tt SPIDER-EM} which exhibits a linear convergence rate when $\lyap$ satisfies a Polyak-Lojasiewicz inequality.
Note that the latter condition (or its variants) has been used in a few recent works, e.g., \cite{balakrishnan2017, chen:etal:2018}.

\begin{figure}[t]
    \centering
    \includegraphics[width=.475\linewidth]{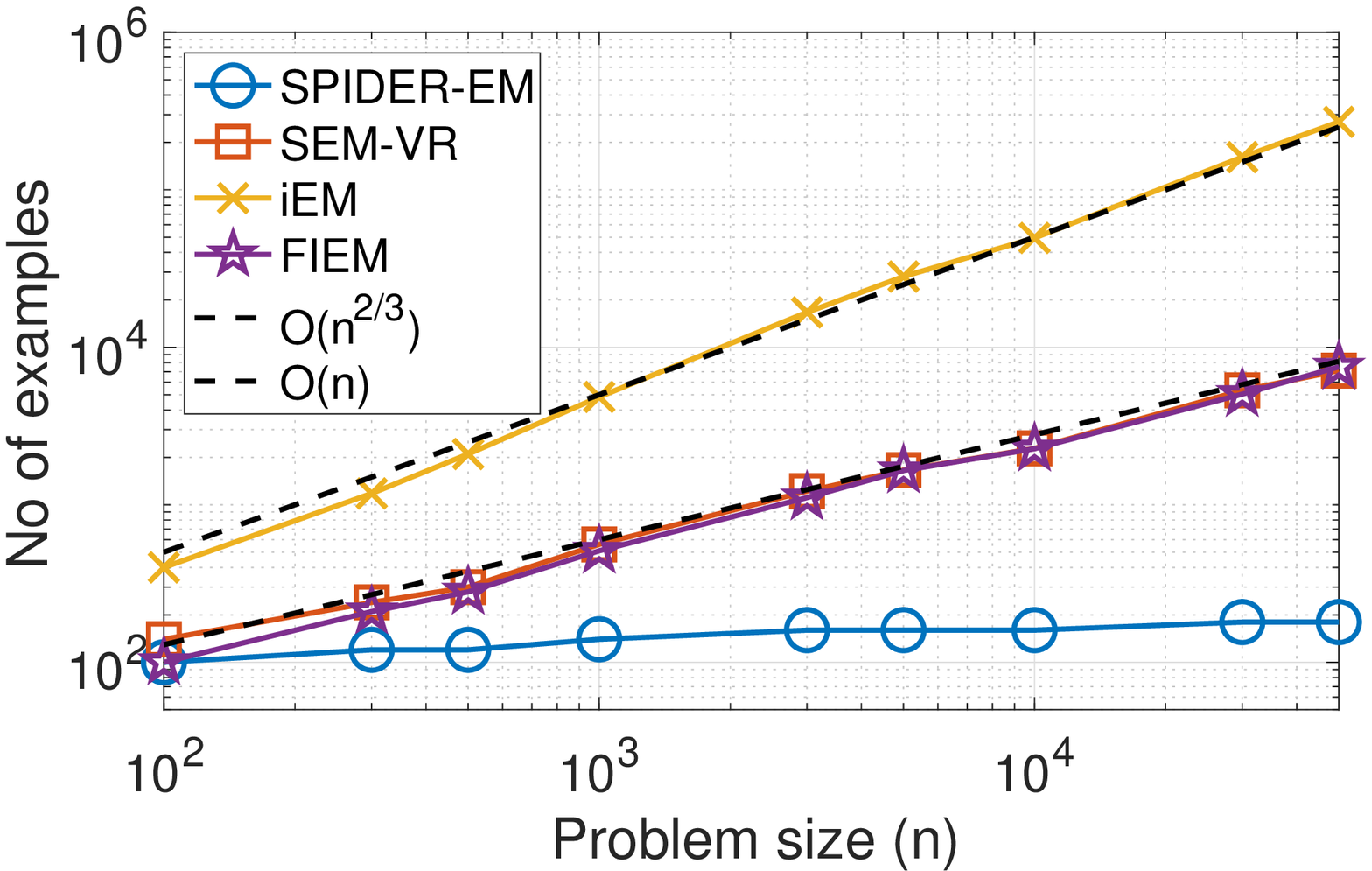}\includegraphics[width=.475\linewidth]{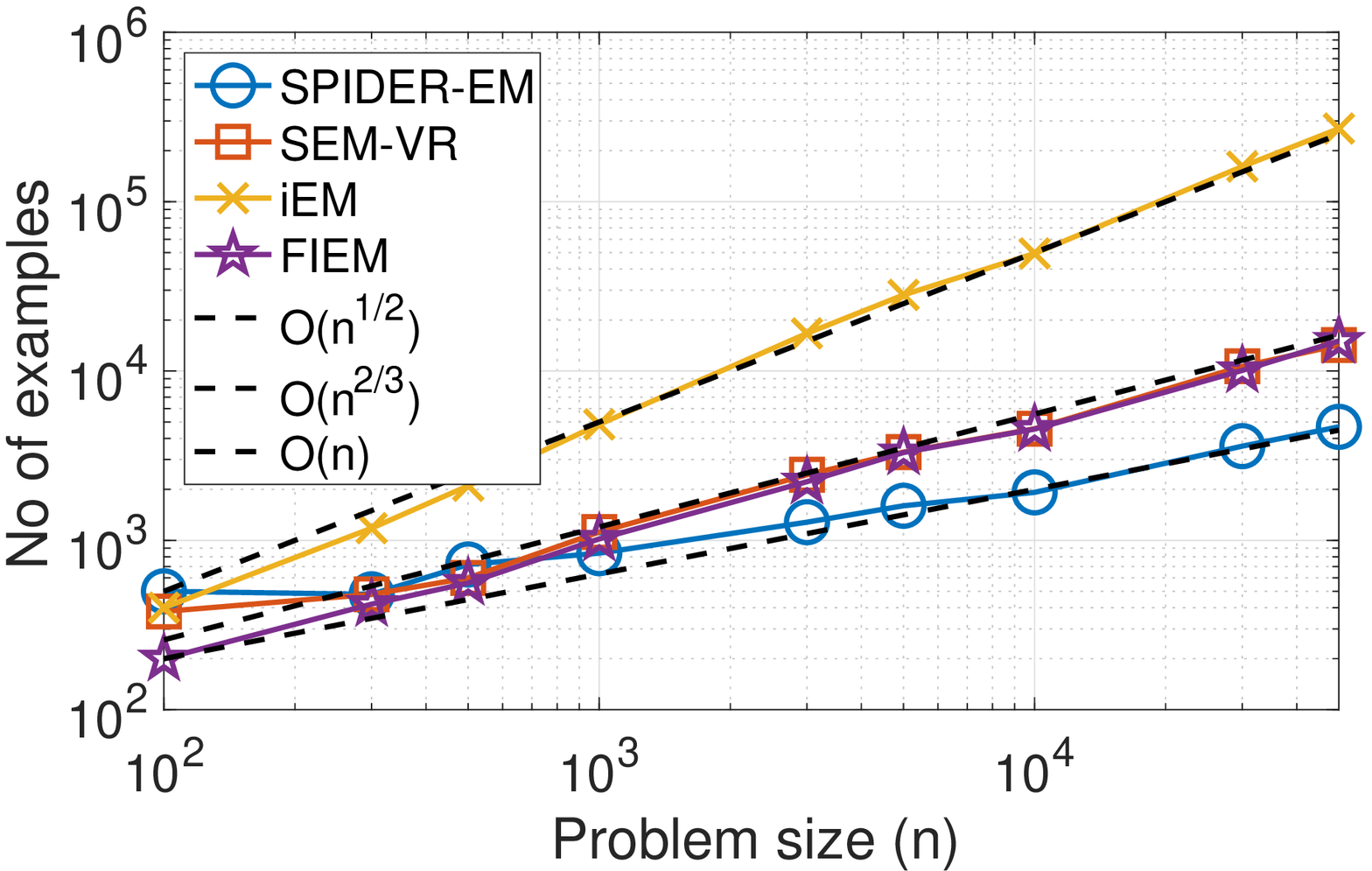}\vspace{-.1cm}
    \caption{[Left] Median estimated number of parameter updates $K_{\operatorname{Opt}}(n,\epsilon)$ needed to reach an accuracy of $2.5 \times 10^{-5}$ [Right] Median estimated number of per-sample conditional expectations $K_{\operatorname{CE}}(n,\epsilon)-n$ needed to reach an accuracy of $2.5 \times 10^{-5}$. The median is taken from a Monte-Carlo simulation among 50 trials.}
    \label{fig:complexityvsn}\vspace{-.4cm}
\end{figure}

\section{Numerical illustration}
\textbf{Synthetic Data.} We evaluate the efficiency of {\tt SPIDER-EM} against the problem size. We generate a synthetic dataset with $n$ observations from a scalar two-components Gaussian mixture model (GMM) with $0.2 {\cal N}(0.5,1) + 0.8 {\cal N}(-0.5,1)$. The variances and the weights are assumed known.  We fit the means $\mu_1, \mu_2$ of a GMM to the observed data. For {\tt SPIDER-EM}, we set $\lbatch = \lceil \sqrt{n} / 20 \rceil$, $\kin = \lceil n / \lbatch \rceil$ and a fixed step size $\gamma_{k} = 0.01$.
We define $\tau_{\rm emp} = t_{\rm emp}\kin + k_{\rm emp}$ as the total number of updates of $\hatS_k$ evaluated, such that $t_{\rm emp}$, $k_{\rm emp}$ are the indices of outer, inner iteration, respectively.
To estimate $K_{\operatorname{Opt}}(n,\epsilon)$ and $K_{\operatorname{CE}}(n,\epsilon)$, we run the {\tt SPIDER-EM} algorithm until the first iteration $\tau_{\rm emp}$
when the solution satisfies $\| h(\hatS_{t_{\rm emp}, k_{\rm emp}}) \|^2 \leq \epsilon = 2.5 \times 10^{-5}$. We take the median of $\tau_{\rm emp}$ over 50 runs to give an estimate of $K_{\operatorname{Opt}}( n, \epsilon )$; similarly, we take the median of $n t_{\rm emp} + 2 \lbatch \tau_{\rm emp}$ to give an estimate of $K_{\operatorname{CE}}(n, \epsilon )$. Note that the conditional expectations computed during the initialization step are ignored.

\Cref{fig:complexityvsn} compares {\tt SPIDER-EM} to the state-of-the-art incremental EM algorithms for different settings of $n$.
The results illustrate that the empirical performance of {\tt SPIDER-EM} agrees with the theoretical analysis. In particular, we observe that for {\tt SPIDER-EM}, the estimated $K_{\operatorname{Opt}} (n,\epsilon)$ is independent of the problem size $n$ while $K_{\operatorname{CE}} (n,\epsilon)-n$ grows at the rate of $\sqrt{n}$.

\textbf{MNIST Dataset.}  We perform experiment on the MNIST dataset to
illustrate the effectiveness of {\tt SPIDER-EM} on real data; this
example is taken from \cite[Section~5]{nguyen:etal:2020}.  The dataset
consists of $n = 6 \times 10^4$ images of handwritten digits, each
with $784$ pixels. We pre-process the dataset as follows.  First, we
eliminate the uninformative pixels ($67$ pixels are always zero)
across all images to obtain a dense representation with
$d_{\operatorname{dense}} = 717$ pixels per image.  Second, we apply
principal component analysis (PCA) to further reduce the data
dimension. We keep the $d_{\operatorname{PC}} = 20$ principal
components (PCs) of each observation.

We estimate a multivariate GMM model with $g = 12$ components. Unlike in the previous experiment, here the parameter $\param$ collects the mixture's weights $\{ \alpha_\ell, 1 \leq \ell \leq g \}$, the expectations of each component and a pulled full covariance matrix. {\tt SPIDER-EM} is compared to {\tt iEM} \cite{Neal:hinton:1998}, {\tt Online EM} \citep{cappe:moulines:2009}, {\tt FIEM} \cite{karimi:etal:2019}, and {\tt sEM-vr} \cite{chen:etal:2018}. Details on the multivariate Gaussian mixture model are given in the supplementary material, \autoref{sec:MNIST}, where we give technical conditions required to verify the assumptions of \autoref{theo:rate:sqrtn}.

\begin{figure}[t]
\centering
\includegraphics[width=0.475\textwidth]{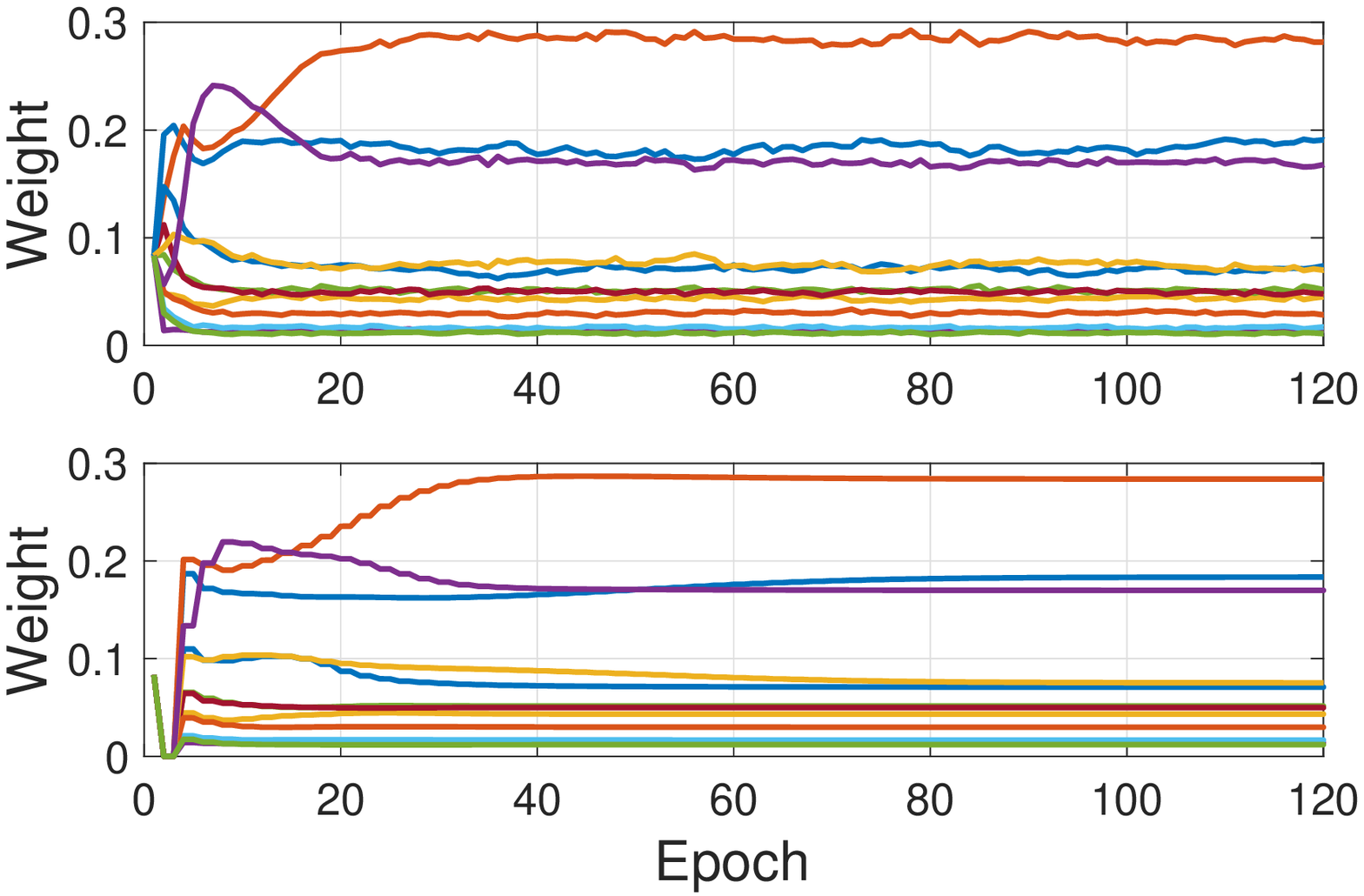}
\includegraphics[width=0.475\textwidth]{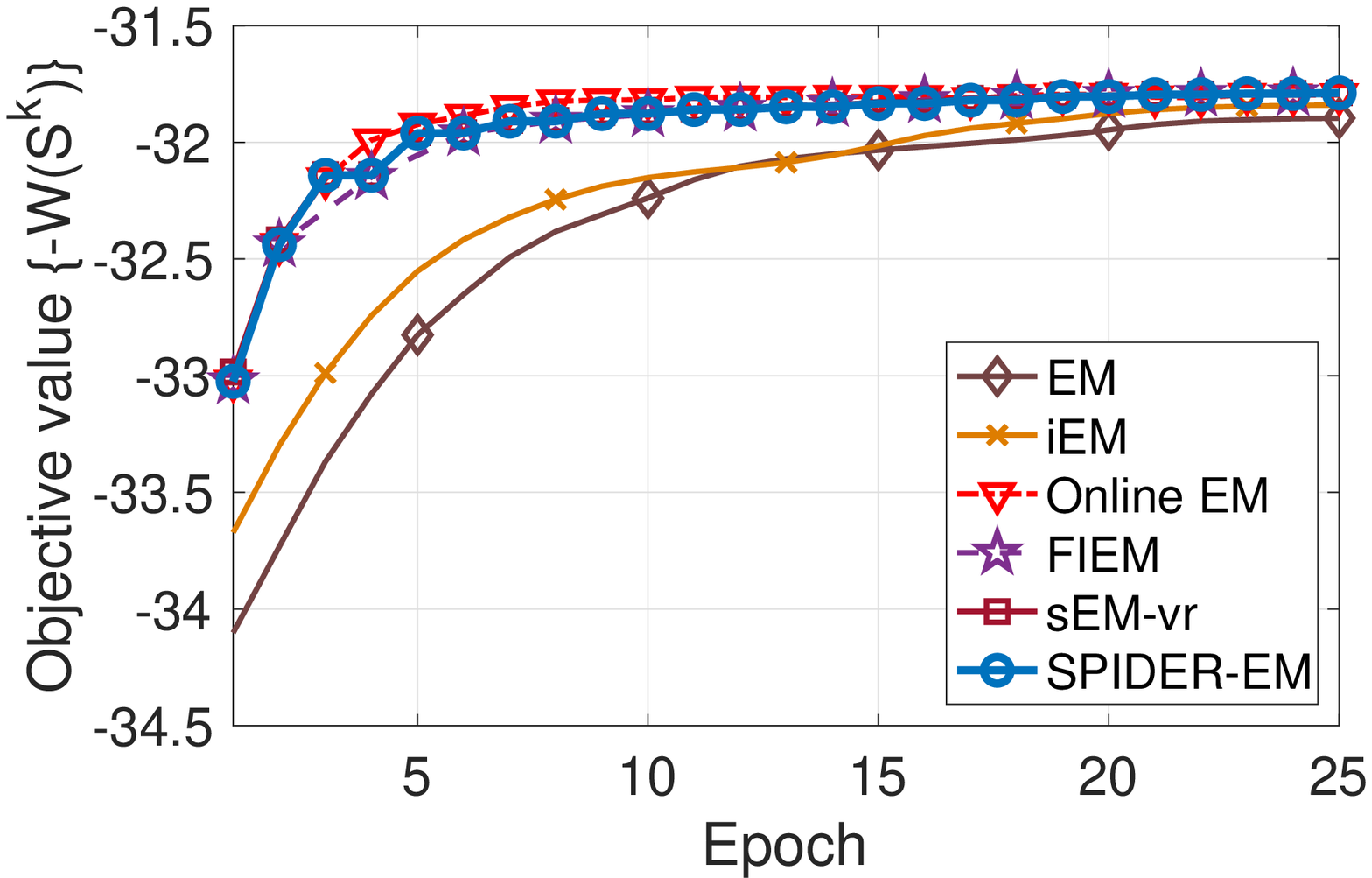}
\caption{[Left] Evolution of the estimates of the weights $\alpha_\ell$ for $\ell=1, \ldots, g$ by {\tt Online EM} (top) and {\tt SPIDER-EM} (bottom) vs the number of epochs. [Right] Evolution of the objective function  $-\lyap(\hatS_k)$ vs the number of epochs.}
\label{fig:estimation:weight}\vspace{-.2cm}
\end{figure}

In \Cref{fig:estimation:weight}, we display the sequence  of parameter estimates $\{\param_\tau\}$, the objective function $\{-\lyap(\hatS_\tau)\}$ and the squared norm of the mean field $\{ \| h(\hatS_\tau) \|^2 \}$.  \Cref{fig:quantiles} gives insights on the distribution of $\|h(\hatS_{t,k})\|^2$ along {\tt SPIDER-EM} paths.
The mini-batches $\{\batch_{\tau} \}_{\tau}$ are
independent, and sampled at random in $\{1, \ldots, n\}$ with
replacement.  For a fair comparison, we use the same seed to sample the minibatches $\{ \batch_{k} \}$; another seed is used for {\tt FIEM} which requires a second sequence of minibatches $\{
\overline{\batch}_{\tau} \}_{\tau}$. The minibatch size is set to be $\lbatch= 100$ and the stepsize $\pas_{\tau}= 5 \times 10^{-3}$ except for {\tt iEM} where $\pas_{\tau}=1$. The same initial value $\hatS_\init$ is used for all experiments. We have implemented the procedure of \cite{kwedlo:2015} in order to obtain the initialization $\param_\init$ and then we set $\hatS_\init \eqdef \bars(\param_\init)$ ( $-\lyap(\hatS_\init) = -58.3$).
The plots illustrate that {\tt SPIDER-EM} reduces the variability of {\tt Online EM} and compares favorably to {\tt iEM} and {\tt FIEM}.
Additional details and results are given in the Supplementary material.

\begin{figure}[t]
\centering
\includegraphics[width=0.475\textwidth]{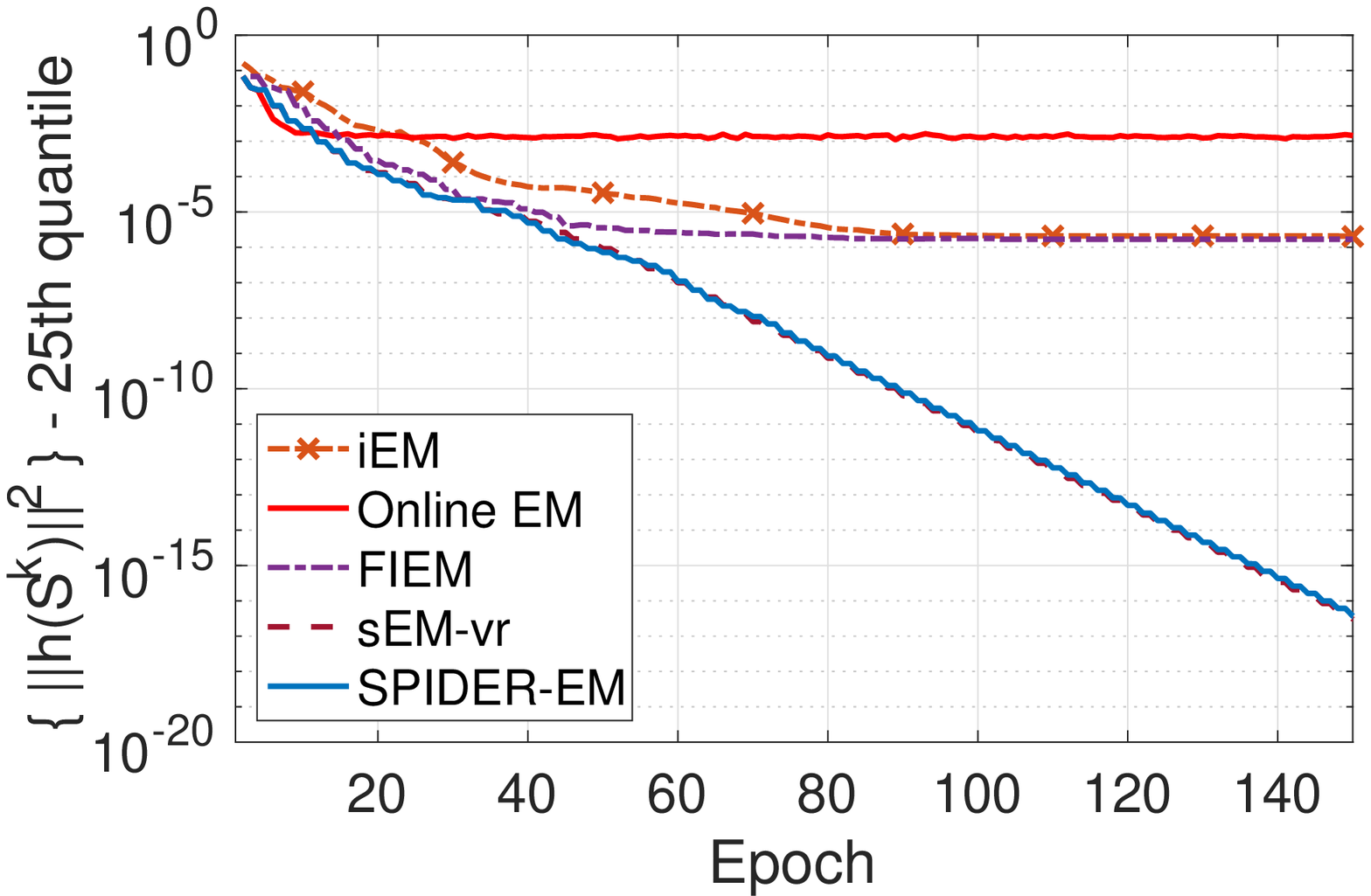}
 \includegraphics[width=0.475\textwidth]{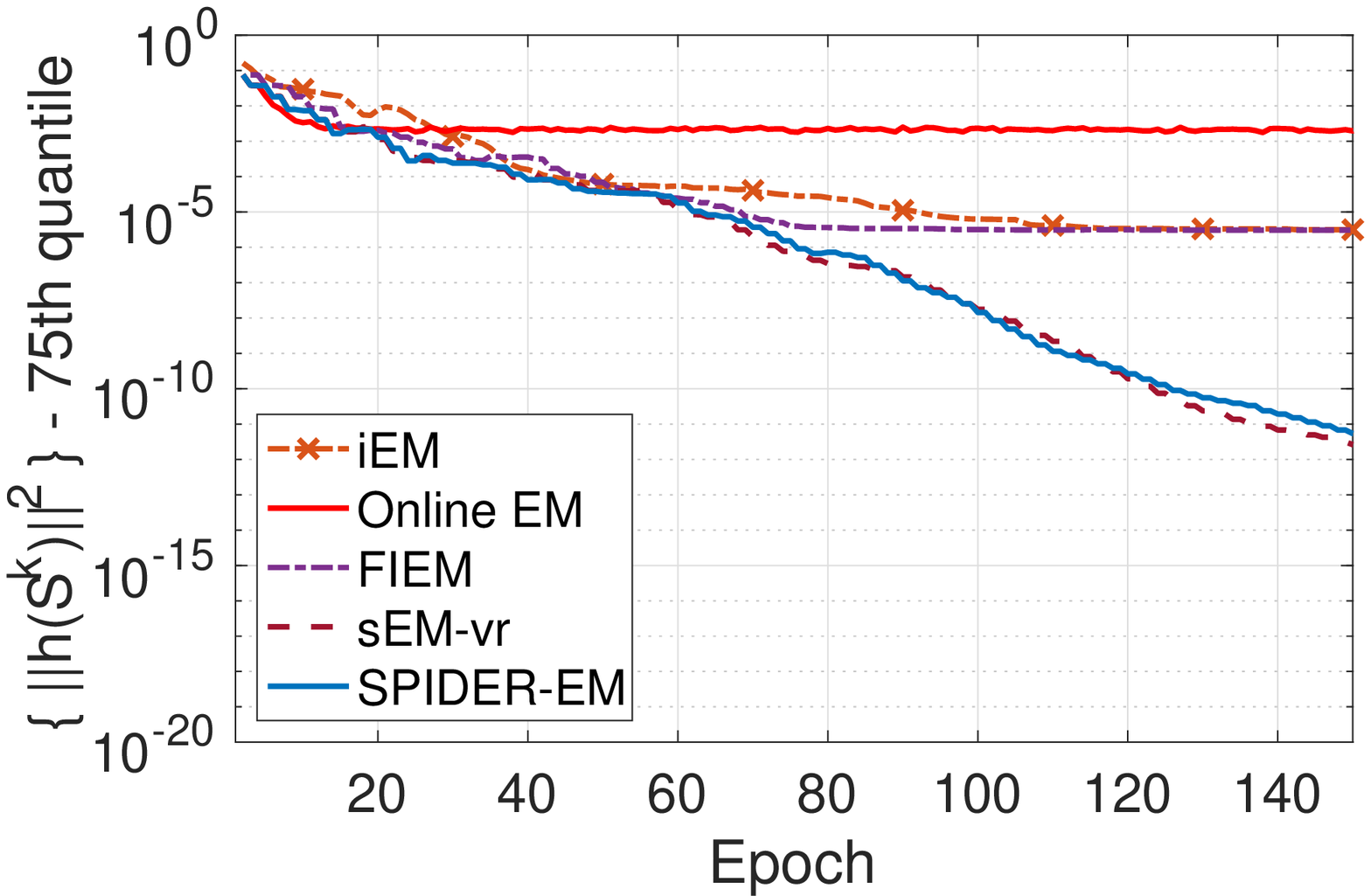}
 \caption{[Left] Quantile $0.25$ and [Right] quantile $0.75$ of the distribution of $\|h(\hatS_{t,-1}) \|^2$ vs the number of epochs $t$; the quantiles are estimated from $40$ independent samples of this distribution. }
\label{fig:quantiles}\vspace{-.2cm}
\end{figure}

\section{Conclusions} We have introduced the {\tt SPIDER-EM} algorithm for large-scale inference. The algorithm offers low memory footprint and improved complexity bounds compared to the state-of-the-art, which is verified by theoretical analysis and numerical experiments.

\clearpage

\paragraph{Broader Impact} This work does not present any foreseeable societal consequence.

\begin{ack}
The work of G. Fort is partially supported by the {\em Fondation
  Simone et Cino del Duca} under the project OpSiMorE. The work of E.~Moulines is partially supported by
ANR-19-CHIA-0002-01 / chaire SCAI. It was partially prepared within
the framework of the HSE University Basic Research Program. The work of H.-T. Wai is partially supported by the CUHK Direct Grant \#4055113.
\end{ack}


\newpage

\title{Supplementary materials for ``A Stochastic Path-Integrated Differential EstimatoR Expectation Maximization Algorithm''}

\author{
Gersende Fort \\
Institut de Mathématiques de Toulouse \\
Université de Toulouse; CNRS \\
UPS, F-31062 Toulouse Cedex 9, France \\
\texttt{gersende.fort@math.univ-toulouse.fr} \\
\And
Eric Moulines \\
Centre de Math\'{e}matiques Appliqu\'{e}es  \\
Ecole Polytechnique, France\\
CS Departement \\
HSE University, Russian Federation\\
\texttt{eric.moulines@polytechnique.edu} \\
\And
Hoi-To Wai \\
Department of SEEM\\
The Chinese University of Hong Kong\\
Shatin, Hong Kong\\
\texttt{htwai@cuhk.edu.hk} \\
}

 \settitle

{\bf Notations.} For two vectors $a,b \in \rset^r$, $\pscal{a}{b}$
denotes the usual Euclidean product and $\|a \|$ the associated norm. By convention, vectors are
column vectors. For a vector $x$ with components $(x_1, \ldots, x_r)$,
$x_{i:j}$ denotes the sub-vector with components $(x_i, x_{i+1},
\ldots, x_{j-1}, x_j)$.

For two matrices $A \in \rset^{r_1 \times r_2}$ and $B \in \rset^{r_3
  \times r_4}$, $A \otimes B$ denotes the Kronecker product. $\Id_r$
is the $r \times r$ identity matrix. $A^T$ is the transpose of $A$.

\section{Complexity of incremental EM-based methods for smooth non-convex finite sum optimization}
\label{sec:table:comparison}

We first compare the complexities of the incremental EM based methods using the following table which summarizes the state-of-the-art results. 

  \begin{table}[h]
{\footnotesize
\begin{tabular}{|l|c|c|c||c|}
  \hline
      {\tt algorithm} & $\pas$ & $K_{\operatorname{Opt}}$ & $K_{\operatorname{CE}}$ & Optimal $K_{\operatorname{CE}}$ \\
      \hline
          {\tt EM} \citep{Dempster:em:1977} & -  & $1 + \kmax$ & $n + n \kmax$ & N/A \\
          \hline
      {\tt online-EM} \citep{cappe:moulines:2009} & decaying; ${\cal O}( L^{-1} k^{-1/2} )$ & $ 1+ \kmax$ & $n + \lbatch \kmax$  & $\epsilon^{-2}$ \\
      {\tt iEM} \citep{Neal:hinton:1998} & 1 & $1+ \kmax$ & $n+ \lbatch \kmax$ & $\epsilon^{-1} n$ \\
     {\tt sEM-vr} \citep{chen:etal:2018,karimi:etal:2019} & $O(L^{-1} n^{-2/3})$  & $1 + \kin \kouter$  & $n (1 + \kouter) + 
     (\lbatch \kin + n) \kouter
    $
     & $\epsilon^{-1} n^{2/3}$ \\
     {\tt FIEM} \citep{karimi:etal:2019} & $O(L^{-1} n^{-2/3})$ & $1+ \kmax$  & $n + 2 \lbatch \kmax$ & $\epsilon^{-1} n^{2/3}$ \\
     {\tt FIEM} \citep{fort:gach:moulines:2020} & $O(L^{-1} n^{-1/3} \kmax^{-1/3})$ & $1+ \kmax$  & $n + 2 \lbatch \kmax$ & $\epsilon^{-3/2} \sqrt{n}$ \\
     \hline
      {\tt SPIDER-EM} & $O(L^{-1})$ & $1+ \kin \kouter$ & $n + \kouter n + 2 \lbatch \kin \kouter$ & $\epsilon^{-1} \sqrt{n}$ \\
      \hline
\end{tabular}}\vspace{.2cm}

\caption{Comparison between different EM-based algorithms for smooth
    non convex finite sum optimization.  Except {\tt sEM-vr} and {\tt
      SPIDER-EM} which have nested loops ($\kouter$ is the maximal
    number of outer loops and $\kin$ is the number of inner loops per
    outer loop), $\kmax$ is the maximal number of iterations. The last
    column is the optimal complexity to reach an
    $\epsilon$-approximate stationary point.}
\end{table}

Next, we provide the psuedo-codes of several existing incremental EM-based algorithms, following the notations defined in the main paper. 

\begin{algorithm}[htbp]
  \KwData{$\kmax \in \nset_\star$, $\hatS_\init \in \Sset$} \KwResult{The EM
    sequence: $\hatS_k, k=0, \ldots, \kmax$} $\hatS_0 =  \bars \circ \map(\hatS_\init)$ \; \For{$k=0, \ldots,
    \kmax-1$}{$\hatS_{k+1}= \bars \circ \map(\hatS_k)$ }
    \caption{The EM algorithm in the expectation space. } \label{algo:EM}
\end{algorithm}

\begin{algorithm}[htbp]
  \KwData{$\kmax \in \nset_\star$, $\hatS_\init \in \Sset$, $\gamma_{k} \in
    \ooint{0,\infty}$ for $k=1, \ldots, \kmax$} \KwResult{The SA
    sequence: $\hatS_k, k=0, \ldots, \kmax$} $\hatS_0 =  \bars \circ \map (\hatS_\init)$ \; \For{$k=0, \ldots,
    \kmax-1$}{Sample a mini-batch $\batch_{k+1}$ in $\{1, \ldots, n\}$
    of size $\lbatch$, with replacement \; $\hatS_{k+1}= \hatS_k +
    \pas_{k+1} \left(  \bars_{\batch_{k+1}}
    \circ \map(\hatS_k) - \hatS_k \right)$. \label{line:SA} }
  \caption{The Online EM algorithm. } \label{algo:SA}
\end{algorithm}

\begin{algorithm}[htbp]
  \KwData{$\kmax \in \nset_\star$, $\hatS_\init \in \Sset$, $\gamma_{k} \in
    \ooint{0,\infty}$ for $k=1, \ldots, \kmax$} \KwResult{The iEM
    sequence: $\hatS_k, k=0, \ldots, \kmax$} $\Smem_{0,i} = \bars_i
  \circ \map(\hatS_\init)$ for all $i=1, \ldots,n$\; $\hatS_0 = \Sronde_0 = n^{-1}
  \sum_{i=1}^n \Smem_{0,i}$\; \For{$k=0, \ldots, \kmax-1$}{Sample a
    mini-batch $\batch_{k+1}$ in $\{1, \ldots, n\}$ of size $\lbatch$,
    with replacement  \; $ \Smem_{k+1,i} = \Smem_{k,i}$ for $i
    \notin \batch_{k+1}$ \; $\Smem_{k+1,i} = \bars_{i} \circ
    \map(\hatS_k)$ for $i \in \batch_{k+1}$\; $\Sronde_{k+1} =
    \Sronde_k + n^{-1} \sum_{i \in \batch_{k+1}} \left( \Smem_{k+1,i}
    - \Smem_{k,i}\right)$ \; $\hatS_{k+1} = \hatS_k + \gamma_{k+1} (
    \Sronde_{k+1} - \hatS^k)$ }
    \caption{The Incremental EM (iEM) algorithm.  \label{algo:iEM}}
\end{algorithm}

\begin{algorithm}[htbp]
  \KwData{$\kmax \in \nset_\star$, $\hatS_\init \in \Sset$, $\gamma_{k} \in
    \ooint{0,\infty}$ for $k=1, \ldots, \kmax$} \KwResult{The FIEM
    sequence: $\hatS_k, k=0, \ldots, \kmax$} $\Smem_{0,i} = \bars_i
  \circ \map(\hatS_\init)$ for all $i=1, \ldots,n$\; $\hatS_0 = \Sronde_0 = n^{-1}
  \sum_{i=1}^n \Smem_{0,i}$\; \For{$k=0, \ldots, \kmax-1$}{Sample a
    mini-batch $\batch_{k+1}$ in $\{1, \ldots, n\}$ of size $\lbatch$,
    with replacement \; $ \Smem_{k+1,i} = \Smem_{k,i}$ for $i \notin
    \batch_{k+1}$ \; $\Smem_{k+1,i} = \bars_{i} \circ \map(\hatS_k)$
    for $i \in \batch_{k+1}$ \; $\Sronde_{k+1} = \Sronde_k + n^{-1}
    \sum_{i \in \batch_{k+1}} \left(\Smem_{k+1,i} -
    \Smem_{k,i}\right)$ \label{line:FIEM:finaux} \; Sample a
    mini-batch $\batch_{k+1}'$ in $\{1, \ldots, n\}$ of size
    $\lbatch$, with replacement \; $V_{k+1} = \Sronde_{k+1} -
    \lbatch^{-1}\sum_{i \in \batch'_{k+1}} \Smem_{k+1,i} $ \;
    $\hatS_{k+1} = \hatS_k + \gamma_{k+1} ( \bars_{\batch'_{k+1}} \circ \map(\hatS_k) - \hatS_k +
    V_{k+1})$ \label{line:FIEM:update}}
    \caption{The Fast Incremental EM (FIEM) algorithm.  \label{algo:FIEM}}
\end{algorithm}

\begin{algorithm}[htbp]
  \KwData{$\kin \in \nset_\star$, $\kouter \in \nset_\star$,
    $\hatS_\init \in \Sset$, $\gamma_{t,k} \in \ooint{0,\infty}$ for
    $t \geq 1, k \geq 1$ } \KwResult{The sEM-vr sequence: $\hatS_{t,k},
    t=1, \dots, \kouter$ and $k=0, \ldots, \kin -1$}{ $\Smem_{1, 0} =
    \bars \circ \map(\hatS_{\init})$ \; $\hatS_{1,0} = \hatS_\init$} \;\For{$t=1, \ldots,
    \kouter$}{ \For{$k=0,\ldots, \kin-2$}{Sample a
      mini-batch $\batch_{t,k+1}$ in $\{1, \ldots, n\}$ of size
      $\lbatch$, with replacement \; $ V_{t,k+1} = \Smem_{t,0} -   \bars_{\batch_{t,k+1}} \circ
      \map(\hatS_{t-1,\kin-1})$ \; $\hatS_{t,k+1} = \hatS_{t,k} +
      \gamma_{t,k+1} \left(
      \bars_{\batch_{t,k+1}} \circ \map(\hatS_{t,k}) - \hatS_{t,k} + V_{t,k+1}\right)$} $\Smem_{t+1,0} = \bars \circ \map(\hatS_{t,\kin-1})$ \; $\hatS_{t+1, 0} = \hatS_{t, \kin-1} + \gamma_{t,\kin}
    \left( \Smem_{t+1, 0} - \hatS_{ t, \kin-1} \right)$}
    \caption{The sEM-vr algorithm.   \label{algo:SEMVR}}
\end{algorithm}

\section{An equivalent definition of the {\tt SPIDER-EM} algorithm}

Using Lemma~\ref{lem:equivalent:SPIDER-EM} below this page, we deduce that {\tt SPIDER-EM} can be equivalently described by the following \autoref{algo:SPIDER-EMbis}.

\begin{algorithm}[h]
  \KwData{$\kin \in \nset_\star$, $\kouter \in \nset_\star$,
    $\hatS_\init \in \rset^q$, a  positive sequence  $\{ \pas_{t,k},  t,k \geq 1\}$.} \KwResult{The
    SPIDER-EM sequence: $\hatS_{t,k}$, $t=1, \dots, \kouter$, $k= 0, \dots, \kin-1$}{
    $\hatS_{1,-1} = \hatS_\init$ \; $\Sronde_{1,0} =  \bars \circ
    \map(\hatS_\init)$} \; \For{$t=1, \ldots, \kouter$}{ $V_{t,0} = 0$ \label{line:outer:CV} \; \For{$k=0,\ldots, \kin-2$}{Sample a
      mini-batch $\batch_{t,k+1}$ in $\{1, \ldots, n\}$ of size
      $\lbatch$, with or without replacement \; $ V_{t,k+1} = V_{t,k} + \Sronde_{t,k} -
      \bars_{\batch_{t,k+1}} \circ \map(\hatS_{t,k-1})$
      \label{line:Vupdate} \;
      $\Sronde_{t,k+1} = \bars_{\batch_{t,k+1}} \circ
      \map(\hatS_{t,k})$
      \label{line:tildeSupdate} \;
      $\hatS_{t,k+1} = \hatS_{t,k} + \gamma_{t,k+1} \left(\Sronde_{t,k+1}-
        \hatS_{t,k} + V_{t,k+1}\right)$
      \label{line:SAupdate}}
    $\Sronde_{t+1,0 } = \bars \circ \map(\hatS_{t, \kin-1})$ \;
    $\hatS_{t+1,0} = \hatS_{t,\kin-1} + \gamma_{t, \kin} \left(
    \Sronde_{t+1,0} - \hatS_{t,\kin-1}
    \right)$ \label{line:outer:hatS}}
    \caption{The SPIDER-EM algorithm (equivalent description) \label{algo:SPIDER-EMbis}}
\end{algorithm}

\begin{lemma}
  \label{lem:equivalent:SPIDER-EM}
   Let $\{\pas_{k}, k \geq 1\}$ be a positive deterministic sequence
   and $\{\batch_{k}, t,k \geq 1\}$ be a family of mini-batches
   sampled from $\{1, \ldots, n\}$.  Fix $\hatS_{-1}, \hatS_{0}$ and
   $\Smem_0$. Define for $k=0, \cdots, \kin-2$
   \begin{align*}
     \left\{
     \begin{array}{l}
     \Smem_{k+1}  \eqdef \Smem_k + \bars_{\batch_{k+1}} \circ \map (\hatS_k) -
\bars_{\batch_{k+1}} \circ \map (\hatS_{k-1}) \eqsp, \\ \hatS_{k+1}
\eqdef \hatS_k + \pas_{k+1} \left( \Smem_{k+1} - \hatS_k\right) \eqsp.
   \end{array}
   \right.
   \end{align*}
  Set $\Sronde_{-1} \eqdef \hatS_{-1}$, $\Sronde_{0} \eqdef \hatS_0$,
  $V_0 \eqdef 0$ and define for $k=0, \ldots, \kin-2$,
\begin{align*}
     \left\{
     \begin{array}{l}
   V_{k+1} \eqdef V_k + \bars_{\batch_{k}} \circ \map (\Sronde_{k-1})
   - \bars_{\batch_{k+1}} \circ \map (\Sronde_{k-1}) \eqsp,
   \\ \Sronde_{k+1} \eqdef \Sronde_k + \pas_{k+1} \left(
   \bars_{\batch_{k+1}} \circ \map (\Sronde_k) - \Sronde_k + V_{k+1}
   \right) \eqsp;
   \end{array}
   \right.
\end{align*}
by convention, set $\bars_{\batch_0} \circ \map(\Sronde_{-1}) =
\Smem_0$.

Then for any $k =-1, \ldots, \kin-1$, $\Sronde_k = \hatS_k$.
\end{lemma}
\begin{proof}
  We prove by induction that for any $k \geq 1$, $V_k = \Smem_k -
  \bars_{\batch_k} \circ \map (\hatS_{k-1})$ and $\Sronde_k =
  \hatS_k$. We have by definition of $V_0$, $\bars_{\batch_0}\circ \map(\Sronde_{-1})$,  $\Sronde_{-1}$ and $\Smem_1$,
  \begin{align*}
V_1 & = \Smem_0 - \bars_{\batch_{1}} \circ \map (\Sronde_{-1}) = \Smem_0 - \bars_{\batch_1} \circ \map(\hatS_{-1}) =
\Smem_1 - \bars_{\batch_{1}} \circ \map (\hatS_{0}) \eqsp.
  \end{align*}
  In addition, by definition of $\Sronde_0$, $\Sronde_1$ and $V_1$, we have
  \begin{align*}
\Sronde_1 & = \hatS_0 + \pas_1 \left(\bars_{\batch_1} \circ
\map(\hatS_0) - \hatS_0 + \Smem_1  - \bars_{\batch_1} \circ \map(\hatS_{-0}) \right) \eqsp.
\end{align*}
Assume that
the property holds for any $0 \leq j \leq k$. Then, by definition of
$V_{k+1}$, the induction assumption on $V_k$ and the definition of
$\Smem_{k+1}$, it holds
\begin{align*}
V_{k+1} & = V_k + \bars_{\batch_{k}} \circ \map (\Sronde_{k-1}) -
\bars_{\batch_{k+1}} \circ \map (\Sronde_{k-1}) \\ & = \Smem_k -
\bars_{\batch_{k+1}} \circ \map (\Sronde_{k-1}) = \Smem_{k+1} -
\bars_{\batch_{k+1}} \circ \map (\Sronde_{k}) \eqsp.
\end{align*}
This concludes the induction for the property on $\{V_k, k \geq 0\}$.
In addition, by the induction assumption on $\Sronde_k$, the
definition of $V_{k+1}$, the induction assumption on $V_k$ and the
definition of $\Smem_{k+1}$, we have
 \begin{align*}
\Sronde_{k+1} & = \hatS_k + \pas_{k+1} \left(\bars_{\batch_{k+1}}
\circ \map(\hatS_k) - \hatS_k + V_k + \bars_{\batch_k} \circ
\map(\hatS_{k-1}) - \bars_{\batch_{k+1}} \circ \map(\hatS_{k-1})
\right) \\ & = \hatS_k + \pas_{k+1} \left(\bars_{\batch_{k+1}} \circ
\map(\hatS_k) - \hatS_k + \Smem_k - \bars_{\batch_{k+1}} \circ \map(\hatS_{k-1})
\right) \\
& = \hatS_k + \pas_{k+1} \left( \Smem_{k+1} - \hatS_k \right) = \hatS_{k+1} \eqsp.
    \end{align*}
This concludes the proof.
  \end{proof}

\section{General convergence results} \label{sec:proof:general}
The purpose of this section is to show the general convergence results of a {\tt SPIDER-EM} like algorithm, and these results will be specialized in \cref{sec:proof:maintheo}. 
For all $i=1, \ldots, n$, $\bars_i \circ \map$ is a
function from $\rset^q$ to $\rset^q$; for a selection of $\lbatch$
indices $\batch$ in $\{1, \ldots, n \}$ with or without replacement,
we set $\bars_{\batch} \circ \map \eqdef \lbatch^{-1} \sum_{i \in
  \batch} \bars_i \circ \map$. More generally, $\bars \circ \map
\eqdef n^{-1} \sum_{i=1}^n \bars_i \circ \map$. For some results
below, specific assumptions may be introduced on $\bars_t \circ \map$.

Let $\{\pas_k, k \geq 1 \}$ be a positive deterministic sequence. Let
$\{\batch_k, k\geq 1 \}$ be a family of independent random mini
batches sampled in $\{1, \ldots, n\}$ of size $\lbatch$, (either with
replacement or without replacement). Finally, let $U_{-1}, U_0$ be
random variables. Assume that $(U_{-1},U_0)$ are independent from the
sequences $\{\batch_k, k \geq 1 \}$ and set
\begin{equation}\label{eq:def:W0}
\widetilde{U}_0 \eqdef \bars \circ \map(U_{-1}) = \PE\left[\bars_{\batch_1} \circ \map(U_{-1}) \vert U_{-1} \right] \eqsp.
\end{equation}
Consider the recursive definition for $k
\geq 0$,
\begin{align*}
\widetilde{U}_{k+1} &= \widetilde{U}_k + \bars_{\batch_{k+1}} \circ \map(U_k) -
\bars_{\batch_{k+1}} \circ \map(U_{k-1}) \eqsp, \\
U_{k+1} & = U_k + \pas_{k+1} \left( \widetilde{U}_{k+1} - U_k\right) \eqsp.
\end{align*}
Finally, define the filtration
\[
\G_0 \eqdef \sigma(U_{-1}, U_0), \qquad \qquad \text{for $k\geq 0$,}
\ \ \G_{k+1} \eqdef \sigma\left( \G_k \cup \batch_{k+1} \right) \eqsp,
\]
and define the sequence of random variables
\[
\Delta_0 \eqdef h(U_{-1}), \qquad \qquad \text{for $k\geq 0$,} \ \ \Delta_{k+1} \eqdef \widetilde{U}_{k+1}-U_k = \pas_{k+1}^{-1}(U_{k+1} -
U_k) \eqsp.
\]

\begin{lemma}\label{lem:minibatch}
  For any $k \geq 0$, $\batch_{k+1}$ and $\G_k$ are independent. For any $u \in \rset^q$,
  \[
\PE\left[ \bars_{\batch_{k+1}} \circ \map(u) \right] =
\bars \circ \map(u) \eqsp.
\]
Assume that $\bars_i \circ \map$ is globally Lipschitz with constant
$L_i$; set $L^2 \eqdef n^{-1} \sum_{i=1}^n L_i^2$. For any $u,u'
\in \rset^q$,
\begin{multline*}
\PE\left[ \| \bars_{\batch_{k+1}} \circ \map(u) - \bars_{\batch_{k+1}}
  \circ \map(u') - \bars \circ \map (u) + \bars \circ \map(u') \|^2
  \right]  \\
\leq  \frac{1}{\lbatch} \left( L^2 \|u-u'\|^2 - \| \bars
\circ \map (u) - \bars \circ \map(u') \|^2 \right) \eqsp.
\end{multline*}
  \end{lemma}
\begin{proof}
  By assumption, $\batch_{k+1}$ and $(U_0, U_{-1})$ are independent,
  and therefore $\batch_{k+1}$ and $\G_0$ are also. In addition,
  $\batch_{k+1}$ is independent of $\batch_\ell$ for any $\ell \leq k$
  so $\batch_{k+1}$ is independent of $\G_k$.

  $\bullet$ Case: sampling with replacement. We write $\batch_{k+1} = \{I_1,
  \cdots, I_{\lbatch} \}$ where the random variables are independent,
  and uniformly distributed on $\{1, \cdots, n\}$.  Then
  \[
\PE\left[ \bars_{\batch_{k+1}} \circ \map(u)\right] =
\frac{1}{\lbatch}\sum_{\ell=1}^\lbatch \PE\left[ \bars_{I_\ell} \circ
  \map(u)\right] = \PE\left[ \bars_{I_1} \circ \map(u)\right] = \bars
\circ \map(u) \eqsp.
\]
In addition, since the variance of the sum is the sum of the variance
for independent r.v.
\begin{align*}
&\PE\left[ \| \bars_{\batch_{k+1}} \circ \map(u) - \bars_{\batch_{k+1}}
  \circ \map(u') - \bars \circ \map (u) + \bars \circ \map(u') \|^2
  \right] \\
& \qquad = \frac{1}{\lbatch^2}\sum_{\ell=1}^\lbatch \PE\left[ \|
  \bars_{I_\ell} \circ \map(u) - \bars_{I_\ell} \circ \map(u') - \bars
  \circ \map (u) + \bars \circ \map(u') \|^2 \right]
\end{align*}
Then we have
\begin{align}
& \PE\left[ \| \bars_{I_\ell} \circ \map(u) - \bars_{I_\ell} \circ
  \map(u') - \bars \circ \map (u) + \bars \circ \map(u') \|^2 \right] \nonumber \\
& \qquad = \frac{1}{n} \sum_{i=1}^n \PE\left[ \| \bars_{i} \circ \map(u) -
  \bars_{i} \circ \map(u') \|^2 \right] -  \| \bars \circ \map (u) + \bars \circ
\map(u') \|^2 \nonumber  \\ &  \qquad \leq \|u - u'\|^2 \frac{1}{n} \sum_{i=1}^n L_i^2 - \|
\bars \circ \map (u) + \bars \circ \map(u') \|^2 \label{eq:replacement}
\end{align}
which concludes the proof.

$\bullet$ Case: sampling with no replacement.  $I_1$ is a uniform
random variable on $\{1, \cdots, n\}$ so that $\PE\left[\bars_{I_1}
  \circ \map(u) \right] = \bars \circ \map(u)$. Conditionally to
$I_1$, $I_2$ is a uniform random variable on $\{1, \cdots, n\}
\setminus \{I_1 \}$. Therefore
\[
\PE\left[\bars_{I_2} \circ \map(u) \right] = \ \frac{1}{n-1} \left(
\sum_{j=1}^n \bars_{j} \circ \map(u) - \PE\left[\bars_{I_1} \circ \map(u) \right]
\right) = \frac{n}{n-1} \bars \circ \map(u) - \frac{1}{n-1} \bars
\circ \map(u) \eqsp.
\]
By induction, for any $\ell \geq 2$,
\begin{align*}
\PE\left[\bars_{I_\ell} \circ \map(u) \right]& = \ \frac{1}{n-\ell+1} \left(
\sum_{j=1}^n \bars_{j} \circ \map(u) - \sum_{q=1}^{\ell-1} \PE\left[\bars_{I_q} \circ \map(u) \right]
\right)  \\
& = \frac{n}{n-\ell+1} \bars \circ \map(u) - \frac{\ell-1}{n-\ell+1} \bars
\circ \map(u) \eqsp.
\end{align*}
As a conclusion, $\lbatch^{-1} \, \sum_{\ell=1}^\lbatch
\PE\left[\bars_{I_\ell} \circ \map(u) \right] = \bars \circ \map(u)$.
Let $u,u' \in \rset^q$; set $\phi({I_\ell}) \eqdef \bars_{I_\ell}
\circ \map (u) - \bars \circ \map(u) + \bars_{I_\ell} \circ \map (u')
- \bars \circ \map(u')$.  Then $\PE\left[\phi(I_\ell) \right] =0$. We
first prove by induction that $\PE\left[ \| \phi(I_\ell) \|^2 \right]
= \PE\left[ \| \phi(I_1) \|^2 \right]$.  Upon noting that $I_1$ is a
uniform random variable on $\{1, \cdots, n\}$,
\begin{align*}
\PE\left[ \| \phi(I_\ell) \|^2 \right] & = \frac{1}{n-\ell+1}\left(
\sum_{i=1}^n \|\phi(i) \|^2 - \PE\left[\|\phi(I_1)\|^2 + \cdots +
  \|\phi(I_{\ell-1})\|^2 \right] \right) \\ & =\frac{n}{n-\ell+1}
\PE\left[\|\phi(I_1)\|^2 \right] - \frac{1}{n-\ell+1}
\sum_{p=1}^{\ell-1} \PE\left[\|\phi(I_p)\|^2 \right]
\end{align*}
which concludes the induction. Second, let us prove that for any $\ell \geq 0$,
\begin{equation}\label{eq:replacement:subadditivity}
\PE\left[ \|\sum_{p=1}^{\ell+1} \phi(I_p) \|^2 \right] \leq (\ell+1)
\PE\left[ \|\phi(I_1) \|^2 \right] \eqsp.
\end{equation}
Since $n^{-1} \sum_{i=1}^n \phi(i) = \PE\left[\phi(I_1) \right]=0$,
\[
\PE\left[ \pscal{\sum_{p=1}^\ell \phi(I_p)}{\phi(I_{\ell+1})} \right]
= \frac{1}{n-\ell} \PE\left[ \pscal{\sum_{p=1}^\ell
    \phi(I_p)}{\sum_{i=1}^n \phi(i) - \sum_{p=1}^\ell \phi(I_p)}
  \right] = - \frac{1}{n-\ell}\PE\left[\|\sum_{p=1}^\ell \phi(I_p)\|^2
  \right] \eqsp,
\]
so that
\[
\PE\left[ \|\sum_{p=1}^{\ell+1} \phi(I_p) \|^2 \right] = \left(1 -
\frac{2}{n-\ell} \right) \PE\left[\|\sum_{p=1}^\ell \phi(I_p)\|^2
  \right] + \PE\left[ \| \phi(I_{\ell+1}) \|^2 \right] \leq (\ell+1)
\PE\left[ \|\phi(I_1) \|^2 \right] \eqsp.
\]
The proof follows from \eqref{eq:replacement:subadditivity} and
\eqref{eq:replacement} since here again, $I_1$ is uniformly
distributed on $\{1, \cdots, n\}$.
\end{proof}

 \begin{lemma}
    \label{lem:field:bias:general}
 For any $k \geq 0$,
  \[
  \PE\left[\Delta_{k+1} \vert \G_k \right] - h(U_k) = \Delta_k -
  h(U_{k-1}) \eqsp.
  \]
  \end{lemma}
   \begin{proof}
Let $k \geq 0$. Since
conditionally to $\G_k$, $\batch_{k+1} = \{I_1, \ldots, I_\lbatch \}$
where the random variables $I_k$'s are independent and uniformly
distributed on $\{1, \ldots, n\}$, we have
    \[
\PE\left[\widetilde{U}_{k+1} \vert \G_k \right] =  \widetilde{U}_k  + \bars \circ \map(U_{k})
- \bars \circ \map(U_{k-1}) \eqsp.
\]
In the case $k=0$, we have by using \eqref{eq:def:W0}
\[
\PE\left[\Delta_1 - h(U_0) \vert \G_0 \right] = \PE\left[\widetilde{U}_1 \vert \G_0 \right] - \bars \circ \map(U_0) =0 = \Delta_0 - h(U_{-1}) \eqsp;
\]
the last equality explains  the convention for $\Delta_0$. In the case $k>0$,
\begin{align*}
\PE\left[\Delta_{k+1} \vert \G_k \right] & =
\PE\left[\widetilde{U}_{k+1} - U_k \vert \G_k \right] =
\widetilde{U}_k + h(U_{k}) - \bars \circ \map(U_{k-1}) \\ & = \Delta_k
+ U_{k-1} + h(U_k) - \bar \circ \map(U_{k-1}) = h(U_k) + \Delta_k -
h(U_{k-1}) \eqsp.
\end{align*}
   \end{proof}

\begin{proposition}
  \label{prop:biasedfield:general}
  Assume that for all $i=1, \cdots, n$, $\bars_i \circ \map$ is
  globally Lipschitz, with constant $L_i$; set $L^2 \eqdef n^{-1}
  \sum_{i=1}^n L_i^2$.  Then $\Delta_0 - \PE\left[\Delta_0 \vert \G_0
    \right]= 0$,
   \begin{multline*}
   \PE[ \|\Delta_{1} - \PE\left[\Delta_{1} \vert \G_0 \right]\|^2
     \vert \G_0] = \PE[ \|\Delta_{1} - h(U_0)\|^2 \vert \G_0] \\
\leq -  \frac{1}{\lbatch}\| \bars \circ \map (U_0) -\bars
 \circ \map (U_{-1}) \|^2  + \frac{L^2}{\lbatch} \|U_0 - U_{-1}\|^2 \eqsp.
  \end{multline*}
and for any $k\geq 1$,
  \begin{align*}
 \PE[ \|\Delta_{k+1} - \PE[\Delta_{k+1} \vert \G_k ]\|^2 \vert \G_k] &
 \leq - \frac{1}{\lbatch}\| \bars \circ \map (U_k) -\bars \circ \map
 (U_{k-1}) \|^2 + \frac{L^2}{\lbatch} \pas_k^2 \, \| \Delta_k\|^2
 \eqsp; \\ \PE[ \|\Delta_{k+1} - h(U_k) \|^2 \vert \G_{0}] & \leq -
 \frac{1}{\lbatch} \sum_{j=0}^k \PE\left[\| \bars \circ \map (U_j)
   -\bars \circ \map (U_{j-1}) \|^2 \vert \G_{0}\right] \\ & +
 \frac{L^2}{\lbatch} \left( \sum_{j=1}^k \pas_j^2 \, \PE\left[ \|
   \Delta_j\|^2 \vert \G_{0} \right] + \|U_0 - U_{-1}\|^2
 \right)\eqsp.
  \end{align*}
\end{proposition}
\begin{proof}
  The statement on $\Delta_0$ is trivial since $\Delta_0 = h(U_{-1})
  \in \G_0$. By definition of $\Delta_1$, by \Cref{lem:minibatch} and
  by \eqref{eq:def:W0}
  \[
\PE\left[\Delta_1 \vert \G_0 \right] = \PE\left[\widetilde U_1 \vert
  \G_0\right] - U_0 = \widetilde{U}_0 + \bars \circ \map(U_0) - \bars
\circ \map(U_{-1}) - U_0 = h(U_0) \eqsp.
\]
The equation
\[
\Delta_1 - \PE\left[\Delta_1 \vert \G_0 \right] = \bars_{\batch_1}
\circ \map (U_0) - \bars_{\batch_1} \circ \map (U_{-1}) - \left( \bars
\circ \map (U_0) - \bars \circ \map (U_{-1}) \right)
\]
and \Cref{lem:minibatch} provides the upper bound for $\Delta_1$.  Let
$k \geq 1$. By definition of $\Delta_{k+1}$ and by
\Cref{lem:minibatch},
  \begin{align*}
\Delta_{k+1} - \PE\left[\Delta_{k+1} \vert \G_k \right] & =
\widetilde{U}_{k+1} - \PE\left[\widetilde{U}_{k+1} \vert \G_k \right] \\
& = \bars_{\batch_{k+1}} \circ \map(U_k) - \bars_{\batch_{k+1}} \circ
\map(U_{k-1}) + \bars \circ \map(U_k) - \bars \circ \map(U_{k-1})
  \end{align*}
  and we then conclude by \Cref{lem:minibatch} again.  For the second
  statement, since we have $\PE\left[ \|U\|^2 \right] = \PE\left[ \|U
    - \PE[U \vert V] \|^2 \right] + \PE\left[ \| \PE[U \vert
      V]\|^2\right]$ for any random variables $U,V$, it holds for any
  $k \geq 0$,
  \begin{align*}
  &  \PE\left[ \|\Delta_{k+1} - h(U_k) \|^2 \vert \G_k \right]  =
    \PE\left[ \|\Delta_{k+1} - \PE\left[\Delta_{k+1} \vert \G_k \right] \|^2
      \vert \G_k \right] + \| \PE\left[\Delta_{k+1} \vert \G_k \right] -
    h(U_k) \|^2 \\ &  \qquad = \PE\left[ \|\Delta_{k+1} - \PE\left[\Delta_{k+1}
        \vert \G_k \right] \|^2  \vert \G_k\right] +  \| \Delta_k -
      h(U_{k-1}) \|^2
  \end{align*}
  where we used \Cref{lem:field:bias:general} in the last equality. By induction, this yields
  \[
  \PE\left[ \|\Delta_{k+1} - h(U_k) \|^2 \vert \G_{0} \right] = \sum_{j=0}^k  \PE\left[ \PE\left[ \|\Delta_{j+1} - \PE\left[\Delta_{j+1}
      \vert \G_j \right] \|^2 \vert \G_j \right] \vert \G_0 \right]
  \]
where we have used that $\Delta_0 - h(U_{-1}) =0$ (by definition). We
then conclude with the first statement.
\end{proof}

\begin{lemma}
  \label{lem:majo:gradient}
  For any $h,s,S \in \rset^q$ and any $q \times q$ symmetric matrix
  $B$, it holds
  \begin{align*}
- 2 \pscal{B h }{S} &= -\pscal{B S}{S} - \pscal{B h}{h} + \pscal{B \{h
  - S\}}{h - S} \eqsp.
  \end{align*}
\end{lemma}

\begin{proposition} \label{prop:linearrate:general}
 Assume \Cref{hyp:model}, \Cref{hyp:bars}, \Cref{hyp:Tmap} and
 \Cref{hyp:regV} and \Cref{hyp:regV:bis}.  It holds for any $K \geq
 2$,
\begin{align*}
  \sum_{\ell = 1}^{K -1} & \delta_\ell \, \PE\left[\| U_\ell -
    U_{\ell-1} \|^2 \vert \G_{0} \right] +
  \frac{v_{\min}}{2}\sum_{k=0}^{K-2} \pas_{k+1} \PE\left[ \|h(U_k)\|^2
    \vert \G_{0} \right] \\
  &\leq \lyap(U_{0}) - \PE\left[ \lyap(U_{K-1})
    \vert \G_0 \right] + \frac{L^2 v_{\max}}{2
    \lbatch}\left(\sum_{k=1}^{K-1} \pas_k \right) \|U_0 -
  U_{-1}\|^2\eqsp,
\end{align*}
where (by convention, $\sum_{\ell =K-1}^{K-2} = 0$)
\[
\delta_\ell \eqdef \left( \frac{v_{\min}}{2\pas_{\ell}} - \frac{L_{\grad \lyap}}{2} - \frac{v_{\max}}{2} \frac{L^2}{\lbatch} \sum_{k=\ell}^{ K-2}  \pas_{k+1}   \right) \eqsp
\]
\end{proposition}
\begin{proof}
Let
$k \in \{0, \cdots,K -2\}$.  By
\Cref{prop:fixed:to:stationary} and
\Cref{hyp:regV:bis}-\Cref{hyp:regV:DerLip}, $\lyap$ is continuously
differentiable with globally Lipschitz gradient, which implies
\[
\lyap(U_{k+1}) - \lyap(U_k) \leq \pscal{\nabla \lyap(U_k)}{
  U_{k+1} - U_k} + \frac{L_{\grad \lyap}}{2} \| U_{k+1} -
U_k \|^2 \eqsp.
\]
By \Cref{prop:fixed:to:stationary}, we have
$\nabla \lyap(U_k) = - B(U_k) h(U_k)$; hence,
\begin{align*}
\pscal{\nabla \lyap(U_k)}{ U_{k+1} - U_k} & = - \pscal{B(U_k) h(U_k)}{
  U_{k+1} - U_k} \eqsp.
\end{align*}
We apply \Cref{lem:majo:gradient} with $B \leftarrow B(U_k)$, $h
\leftarrow h(U_k)$ and $S \leftarrow \Delta_{k+1} = (U_{k+1} -
U_k)/\pas_{k+1}$. This yields by
\Cref{hyp:regV:bis}-\Cref{hyp:regV:C1:vmax},
\[
\pscal{\nabla \lyap(U_k)}{ U_{k+1} - U_k} \leq -\frac{\pas_{k+1}
  v_{\min}}{2 } \|\Delta_{k+1}\|^2 - \frac{ v_{\min} \pas_{k+1}}{2}
\|h(U_k)\|^2 + \frac{ v_{\max} \pas_{k+1}}{2} \|h(U_k) - \Delta_{k+1}
\|^2
\]
and since $\Delta_{k+1} = (U_{k+1} - U_k)/\pas_{k+1}$, we obtain
\begin{align*}
  &\pscal{\nabla \lyap(U_k)}{ U_{k+1} - U_k} \leq -
  \frac{v_{\min}}{2\pas_{k+1}} \| U_{k+1} - U_k \|^2 - \frac{v_{\min}
    \pas_{k+1}}{2} \|h(U_k)\|^2 + \frac{v_{\max}\pas_{k+1}}{2}
  \|\Delta_{k+1} - h(U_k) \|^2 \eqsp.
\end{align*}
Therefore, we established
\begin{align*}
\left(\frac{v_{\min}}{2\pas_{k+1}} - \frac{L_{\grad \lyap}}{2} \right) \| U_{k+1} - U_k \|^2  + \frac{v_{\min} \pas_{k+1}}{2} \|h(U_k)\|^2  & \leq \frac{v_{\max}\pas_{k+1}}{2} \|\Delta_{k+1} - h(U_k) \|^2 \\
& + \lyap(U_{k}) - \lyap(U_{k+1}) \eqsp.
\end{align*}
Applying the conditional expectation and using \Cref{prop:biasedfield:general}
(and again $\pas_{j}^2 \| \Delta_{j}\|^2 = \|U_{j} - U_{j-1}\|^2$ for $j \geq 1$),
this yields
\begin{align*}
&\left(\frac{v_{\min}}{2\pas_{k+1}} - \frac{L_{\grad \lyap}}{2} \right) \PE\left[\| U_{k+1} - U_k \|^2 \vert \G_0 \right]  + \frac{v_{\min} \pas_{k+1}}{2} \PE\left[ \|h(U_k)\|^2 \vert \G_{0} \right]   \\
& \qquad\leq \frac{v_{\max}\pas_{k+1}}{2} \frac{L^2}{\lbatch} \sum_{j = 0}^k \PE\left[ \| U_{j} - U_{j-1} \|^2 \vert \G_0\right] + \PE\left[\lyap(U_{k}) - \lyap(U_{k+1})  \vert \G_0 \right]\eqsp.
\end{align*}
We now sum from $k = 0$ to $k = K-2$ and obtain by using
\Cref{lem:toybis} with $\bar \Delta_j \leftarrow \PE\left[\| U_{j} -
  U_{j-1} \|^2 \vert \G_0 \right]$,
 \begin{align*}
   & \left(\frac{v_{\min}}{2\pas_{K-1}} - \frac{L_{\grad \lyap}}{2}
   \right) \PE\left[\| U_{K-1} - U_{K-2} \|^2 \vert \G_0 \right] \\ &
   + \sum_{\ell = 1}^{ K-2} \left( \frac{v_{\min}}{2\pas_{\ell}} -
   \frac{L_{\grad \lyap}}{2} - \frac{v_{\max}}{2} \frac{L^2}{\lbatch}
   \sum_{k=\ell}^{ K-2} \pas_{k+1} \right) \PE\left[\| U_{\ell} -
     U_{\ell-1} \|^2 \vert \G_0 \right] \\ & +
   \frac{v_{\min}}{2}\sum_{k=0}^{K-2} \pas_{k+1} \PE\left[
     \|h(U_k)\|^2 \vert \G_0 \right] \leq \PE\left[\lyap(U_{0}) -
     \lyap(U_{K-1}) \vert \G_0 \right] \\ & + \|U_0 - U_{-1}\|^2
   \left( \sum_{k=1}^{K-1} \pas_k\right) \frac{L^2 v_{\max}}{2
     \lbatch}\eqsp.
\end{align*}
 This concludes the proof.
\end{proof}

\begin{lemma}
  \label{lem:toybis}
  For any real numbers $a_i, b_i, \bar \Delta_i$ and $K \geq 2$,
\begin{align*}
\sum_{k = 1}^{ K -1} & \left( a_{k} \bar \Delta_{k} - b_k
\sum_{\ell =0}^{k-1}  \bar \Delta_\ell \right) = a_{K -1} \bar
\Delta_{ K-1} - \bar \Delta_{0} \sum_{k = 1}^{ K-1} b_{k} +
\sum_{\ell = 1}^{ K -2} \left( a_\ell - \sum_{k = \ell+1}^{ K-1}
b_{k}\right) \bar \Delta_\ell \eqsp.
\end{align*}
  \end{lemma}


\begin{lemma}
  \label{lem:reorder}
  For any $k \geq (t-1) \kin$,
  \begin{multline*}
    \sum_{q = (t-1) \kin}^k \left( - a_{q+1} X_{q+1} + b_{q+1}
    \sum_{j=(t-1) \kin}^q Y_j + c_{q+1} \sum_{j=(t-1) \kin}^q d_j X_j
    \right) \\ = - a_{k+1} X_{k+1} + d_{(t-1) \kin} \left( \sum_{q
      =(t-1)\kin}^k c_{q+1} \right) X_{(t-1) \kin} \\ + \sum_{j= (t-1)
      \kin+1}^{k} \left( d_j \left( \sum_{q =j}^k c_{q+1} \right) -
    a_j \right) X_j + \sum_{j=(t-1) \kin}^k \left( \sum_{q=j}^k
    b_{q+1}\right) Y_j \eqsp.
  \end{multline*}
\end{lemma}

\section{Proof of Main Results in \autoref{sec:main:result}}
\label{sec:proof:maintheo}
For $t=1,\cdots, \kouter$ and $k = 0, \cdots, \kin-2$, define the
$\sigma$-field $\F_{t,k}$:
\[
\F_{0,\kin-1} \eqdef \sigma(\hatS_\init) \eqsp, \qquad \F_{t,0}
\eqdef \F_{t-1,\kin-1}\eqsp, \qquad \F_{t,k+1} \eqdef \sigma \left(
\F_{t,k} \cup \batch_{t,k+1} \right) \eqsp.
\]
With these definitions, we have for $t=1,\cdots, \kouter$ and $k = 0, \cdots, \kin-2$,
\[
\hatS_{t,k+1} \in \F_{t,k+1} \eqsp, \qquad \Smem_{t,k+1} \in \F_{t,k+1} \eqsp, \qquad \batch_{t,k+1} \in \F_{t,k+1} \eqsp;
\]
and $\hatS_{t,0} \in \F_{t,0}$, $\Smem_{t,0} \in \F_{t,0}$.  For
$t=1,\cdots, \kouter$ and $k = 0, \cdots, \kin-2$ set
  \begin{equation}
    \label{eq:def:fieldH}
    H_{t,k+1} \eqdef \pas_{t,k+1}^{-1} \left(\hatS_{t,k+1} - \hatS_{t,k} \right) = \Smem_{t,k+1} - \hatS_{t,k}  \in \F_{t,k+1} \eqsp;
  \end{equation}
  and choose the convention $H_{1,0} \eqdef h(\hatS_{1,-1})$, and
  \begin{equation}\label{eq:defH0:convention}
H_{t+1,0} = H_{t,\kin} \eqdef \pas_{t,\kin}^{-1} (\hatS_{t+1,0} -
\hatS_{t,\kin-1}) = \Smem_{t+1,0} - \hatS_{t, \kin-1} =
h\left(\hatS_{t, \kin-1} \right)\eqsp.
  \end{equation}

\subsection{Preliminary lemmas}
The following results are consequences of the general analysis in \cref{sec:proof:general}. 

  \begin{lemma}
    \label{lem:field:bias} Assume \Cref{hyp:model}, \Cref{hyp:bars}, \Cref{hyp:Tmap}.
    Let $\{\hatS_{t,k}, t =1, \cdots, \kouter, k =0, \cdots, \kin-1\}$ be the sequence given
    by \autoref{algo:SPIDER-EM}.  For $t=1,\cdots, \kouter$ and $k =
    0, \cdots, \kin-2$
  \begin{align*}
 &  \PE\left[H_{t,k+1} \vert \F_{t,k} \right] - h(\hatS_{t,k}) = H_{t,k} -
    h(\hatS_{t,k-1}) \eqsp,  \\
    & H_{t,0} -  h(\hatS_{t,-1})  = 0  =H_{t, \kin} - h(\hatS_{t,\kin-1}) \eqsp.
  \end{align*}
  \end{lemma}
  \begin{proof}
    Let $t \geq 1$: apply \Cref{lem:field:bias:general} with $U_0
    \leftarrow \hatS_{t,0}$, $U_{-1} \leftarrow \hatS_{t,-1}$, $\pas_{k+1}
    \leftarrow \pas_{t,k+1}$, $\batch_{k+1} \leftarrow
    \batch_{t,k+1}$. Then $\widetilde{U}_0 \leftarrow \Smem_{t,0}$ satisfies the
    condition \eqref{eq:def:W0} and for any $ k\geq 0$, we have
    $U_{k+1} =\hatS_{t,k+1}$, $\widetilde{U}_{k+1} = \Smem_{t,k+1}$, $\Delta_{k+1}
    = H_{t,k+1}$ and $\G_{k+1} = \F_{t,k+1}$. This yields the result.
   \end{proof}
  \begin{corollary}[of \Cref{lem:field:bias}]
    \label{coro::field:bias}
  For $t=1,\cdots, \kouter$ and $k = 0, \cdots, \kin$
    \[
\PE[ H_{t,k} - h(\hatS_{t,k-1}) \vert \F_{t,0}] =0 \eqsp.
    \]
  \end{corollary}
  \begin{proof}
 Let $t \geq 1$.  If $k=0$ then by
  \Cref{lem:field:bias}, the property holds. Let $k \in \{0, \ldots,  \kin-2\}$.
    We write by using \Cref{lem:field:bias}
    \begin{align*}
\PE[ H_{t,k+1} - h(\hatS_{t,k}) \vert \F_{t,0}] & = \PE[ \PE[
    H_{t,k+1} - h(\hatS_{t,k}) \vert \F_{t,k} ] \vert \F_{t,0}] = \PE[
  H_{t,k} - h(\hatS_{t,k-1} ) \vert \F_{t,0}] \eqsp.
    \end{align*}
   The proof is concluded by induction:
    \[
\PE[ H_{t,k+1} - h(\hatS_{t,k}) \vert \F_{t,0}] =
\PE[ H_{t,0} - h(\hatS_{t,-1}) \vert
  \F_{t,0}]  =0 \eqsp.
\]
    \end{proof}

\begin{proposition}
  \label{prop:biasedfield}
  Assume \Cref{hyp:model}, \Cref{hyp:bars}, \Cref{hyp:Tmap},
  \Cref{hyp:regV:bis}-\ref{hyp:Tmap:smooth} and set $L^2 \eqdef n^{-1}
  \sum_{i=1}^n L_i^2$.  For any $t =1, \cdots, \kouter$, $H_{t,0} -
  h(\hatS_{t,-1}) = 0$,
and  \[
   \PE[ \|H_{t,1} - \PE\left[H_{t,1} \vert \F_{t,0} \right]\|^2 \vert
   \F_{t,0}] \leq - \frac{1}{\lbatch}\| \bars \circ \map (\hatS_{t,0}) -\bars
   \circ \map (\hatS_{t,-1}) \|^2  + \frac{L^2}{\lbatch} \|\hatS_{t,0} - \hatS_{t,-1}\|^2 \eqsp.
  \]
In addition, for $k = 1, \cdots, \kin-2$,
\begin{align*}
      \PE[ \|H_{t,k+1} - h(\hatS_{t,k}) \|^2 \vert \F_{t,0}] & \leq -
      \frac{1}{\lbatch} \sum_{j=0}^k \PE\left[\| \bars \circ \map
        (\hatS_{t,j}) -\bars \circ \map (\hatS_{t,j-1}) \|^2 \vert
        \F_{t,0}\right] \\ & + \frac{L^2}{\lbatch} \left( \sum_{j=1}^k
      \pas_{t,j}^2 \, \PE\left[ \| H_{t,j}\|^2 \vert \F_{t,0} \right]  + \|\hatS_{t,0} - \hatS_{t,-1}\|^2 \right) \eqsp, \\
      \PE[ \|H_{t,k+1} - \PE[H_{t,k+1} \vert \F_{t,k} ]\|^2 \vert
        \F_{t,k}] & \leq - \frac{1}{\lbatch}\| \bars \circ \map
      (\hatS_{t,k}) -\bars \circ \map (\hatS_{t,k-1}) \|^2 +
      \frac{L^2}{\lbatch} \pas_{t,k}^2 \, \| H_{t,k}\|^2 \eqsp.
\end{align*}
Finally,
\[
  \|H_{t,\kin} - h(\hatS_{t,\kin-1}) \| = \|H_{t,\kin} - \PE\left[
    H_{t,\kin} \vert \F_{t, \kin-1}\right] \| =0 \eqsp.
\]
\end{proposition}
\begin{proof}
Let $t \geq 1$. Apply \Cref{prop:biasedfield:general} with $\pas_{k}
\leftarrow \pas_{t,k}$, $\batch_{k+1} \leftarrow \batch_{t,k+1}$, $U_0
\leftarrow \hatS_{t,0}$, $U_{-1} \leftarrow \hatS_{t,-1}$, $\G_k
\leftarrow \F_{t,k}$.  Since $\Smem_{t,0} = \bars \circ
\map(\hatS_{t,-1})$, then the condition \eqref{eq:def:W0} is satisfied
with $\widetilde{U}_0 = \Smem_{t,0}$. Conclude by observing that
$\widetilde{U}_k = \Smem_{t,k}$ and $\Delta_{k+1} = H_{t,k+1}$.
  \end{proof}
\subsection{Proof of \Cref{theo:rate:sqrtn}}
\begin{proposition}
  \label{prop:general:upperbound}
Assume \Cref{hyp:model}, \Cref{hyp:bars}, \Cref{hyp:Tmap},
\Cref{hyp:regV} and \Cref{hyp:regV:bis}. Set $L^2 \eqdef n^{-1}
\sum_{i=1}^n L_i^2$. For any positive numbers $\beta_{t,k}$, set for
$t=1, \cdots, \kouter$ and $k=0, \cdots, \kin-1$
  \begin{align*}
    A_{t,k} & \eqdef \pas_{t,k} v_{\min} \left(1 -
    \frac{\beta_{t,k}^2}{2 v_{\min}} - \pas_{t,k} \frac{L_{\grad
        \lyap}}{2 v_{\min}} - \frac{L^2 v_{\max}^2}{2 v_{\min}
      \lbatch} \pas_{t,k} \left(\sum_{\ell=k}^{\kin-2}
    \frac{\pas_{t,\ell+1}}{\beta_{t,\ell+1}^2} \right) \right) \\ B_{t,k}
    & \eqdef \frac{v^2_{\max}}{2 \lbatch } \sum_{k=0}^{\kin-2} \left(
    \sum_{\ell = k}^{\kin-2}
    \frac{\pas_{t,\ell+1}}{\beta_{t,\ell+1}^2} \right) \eqsp;
  \end{align*}
  by convention $\beta_{t,0} =0$, $\pas_{t,0} = \pas_{t-1,\kin}$,
  $\pas_{0,\kin}=0$ and $B_{t, \kin-1} =0$.

  Let $\{\hatS_{t,k}, t =1, \cdots, \kouter; k = 0, \cdots, \kin-1\}$
  be the sequence given by $\autoref{algo:SPIDER-EM}$. For any $t=1,
  \cdots, \kouter$,
  \begin{equation}\label{eq:strict:contraction}
  \lyap(\hatS_{t,0}) \leq \lyap(\hatS_{t,-1}) -  \pas_{t-1,\kin} v_{\min} \left(1 -
  \pas_{t-1,\kin} \frac{L_{\nabla \lyap}}{2
   v_{\min}} \right) \|h(\hatS_{t, -1})\|^2 \eqsp;
  \end{equation}and
  \begin{multline*}
\sum_{t=1}^{\kouter} \sum_{k=0}^{\kin-1} \left( A_{t,k} \PE[
  \|H_{t,k}\|^2 ] + B_{t,k} \PE[ \| \bars \circ \map(\hatS_{t,k}) -
  \bars \circ \map(\hatS_{t,k-1})\|^2 ] \right) \leq
\PE[\lyap(\hatS_\init) ] - \min \lyap \eqsp.
  \end{multline*}
  \end{proposition}
\begin{proof}
Let $t \geq 1$.  By \Cref{hyp:regV:bis}-\ref{hyp:regV:DerLip}, we have
for any $k = -1, \cdots, \kin-1$,
\begin{equation}
  \label{eq:tool1}
\lyap(\hatS_{t,k+1}) \leq \lyap(\hatS_{t,k}) + \pas_{t,k+1} \pscal{\grad
  \lyap(\hatS_{t,k})}{H_{t,k+1}} + \pas_{t,k+1}^2 \frac{L_{\grad \lyap}}{2}\, \|
H_{t,k+1} \|^2 \eqsp;
\end{equation}
by convention, we set $\hatS_{t,\kin} \eqdef \hatS_{t+1,0}$.  By
\Cref{prop:fixed:to:stationary},
\Cref{hyp:regV:bis}-\ref{hyp:regV:C1:vmax} and
\eqref{eq:defH0:convention}, we have
\[
\pscal{\grad \lyap(\hatS_{t,\kin-1})}{H_{t,\kin}} \leq - v_{\min}
\|h(\hatS_{t,\kin-1})\|^2 =  - v_{\min} \|H_{t,\kin} \|^2 \eqsp,
\]
so that
\begin{equation}
\lyap(\hatS_{t,\kin}) \leq \lyap(\hatS_{t,\kin-1}) - \pas_{t,\kin}  v_{\min} \|H_{t,\kin} \|^2 + \pas_{t,\kin}^2 \frac{L_{\grad \lyap}}{2}\, \|
H_{t,\kin} \|^2 \eqsp.
\end{equation}
This concludes the proof of \eqref{eq:strict:contraction} since
$\hatS_{t,\kin} = \hatS_{t+1,0}$ and $\hatS_{t,\kin-1} =
\hatS_{k+1,-1}$. Now, let us fix $k \in \{0, \cdots, \kin-2\}$. We
write
\begin{align} \label{eq:tool2}
\pscal{\grad \lyap(\hatS_{t,k})}{H_{t,k+1}} & = - \pscal{B(\hatS_{t,k})
  h(\hatS_{t,k})}{H_{t,k+1}}  \nonumber \\ & = - \pscal{B(\hatS_{t,k}) \left( h(\hatS_{t,k}) -
  H_{t,k+1} \right)}{H_{t,k+1}} - \pscal{B(\hatS_{t,k}) H_{t,k+1} }{H_{t,k+1}}
\nonumber \\ & \leq - \pscal{B(\hatS_{t,k}) \left( h(\hatS_{t,k}) - H_{t,k+1}
  \right)}{H_{t,k+1}} - v_{\min} \|H_{t,k+1}\|^2 \eqsp.
  \end{align}
Note that for $a,b \in \rset^q$ and  $\beta>0$,
  \[
\pscal{a}{b} \leq \frac{\beta^2}{2} \|a\|^2 + \frac{1}{2 \beta^2} \|b\|^2 \eqsp.
  \]
By \Cref{hyp:regV:bis}-\ref{hyp:regV:C1:vmax}, we have for any
$\beta_{t,k+1} >0$,
\begin{equation}\label{eq:tool3}
\left| \pscal{B(\hatS_{t,k}) \left( h(\hatS_{t,k}) - H_{t,k+1} \right)}{H_{t,k+1}}
\right| \leq \frac{\beta_{t,k+1}^2}{2} \, \|H_{t,k+1}\|^2 +
\frac{v^2_{\max}}{2 \beta_{t,k+1}^2} \| H_{t,k+1} - h(\hatS_{t,k}) \|^2 \eqsp.
\end{equation}
Combining \eqref{eq:tool1}, \eqref{eq:tool2} and \eqref{eq:tool3}
yield
\[
\lyap(\hatS_{t,k+1}) \leq \lyap(\hatS_{t,k}) - \Lambda_{t,k+1} \|H_{t,k+1}\|^2 +
\pas_{t,k+1} \frac{v^2_{\max}}{2 \beta_{t,k+1}^2} \| H_{t,k+1} - h(\hatS_{t,k})
\|^2 \eqsp,
\]
where for $\ell = 1, \ldots, \kin-1$,
\[
 \Lambda_{t,\ell} \eqdef \pas_{t,\ell} v_{\min} \left(1 -
 \frac{\beta_{t,\ell}^2}{2 v_{\min}} - \pas_{t,\ell} \frac{L_{\grad \lyap}}{2
   v_{\min}} \right) \eqsp.
 \]
By \Cref{prop:biasedfield},
\begin{align*}
\PE\left[\lyap(\hatS_{t,k+1}) \vert \F_{t,0} \right] & \leq \PE\left[
  \lyap(\hatS_{t,k}) \vert \F_{t,0} \right] - \Lambda_{t,k+1} \PE\left[
  \|H_{t,k+1}\|^2 \vert \F_{t,0} \right] \\ & - \pas_{t,k+1}
\frac{v^2_{\max}}{2 \beta_{t,k+1}^2} \frac{1}{\lbatch} \sum_{j=0}^k
\PE\left[\| \bars \circ \map (\hatS_{t,j}) -\bars \circ \map
  (\hatS_{t,j-1}) \|^2 \vert \F_{t,0}\right] \\ & + \pas_{t,k+1}
\frac{v^2_{\max}}{2 \beta_{t,k+1}^2} \frac{L^2}{\lbatch} \left(
\sum_{j=1}^k \pas_{t,j}^2 \, \PE\left[ \| H_{t,j}\|^2 \vert \F_{t,0}
  \right] + \|\hatS_{t,0} -\hatS_{t,-1} \|^2 \right)\eqsp;
\end{align*}
by taking the expectation, this yields
\begin{align*}
\PE\left[\lyap(\hatS_{t,k+1}) \right] & \leq \PE\left[
  \lyap(\hatS_{t,k}) \right] - \Lambda_{t,k+1} \PE\left[
  \|H_{t,k+1}\|^2 \right] \\ & - \pas_{t,k+1} \frac{v^2_{\max}}{2
  \beta_{t,k+1}^2} \frac{1}{\lbatch} \sum_{j=0}^k \PE\left[\| \bars
  \circ \map (\hatS_{t,j}) -\bars \circ \map (\hatS_{t,j-1}) \|^2
  \right] \\ & + \pas_{t,k+1} \frac{v^2_{\max}}{2 \beta_{t,k+1}^2}
\frac{L^2}{\lbatch} \left( \sum_{j=1}^k \pas_{t,j}^2 \, \PE\left[ \|
  H_{t,j}\|^2 \right] + \PE\left[\|\hatS_{t,0} -\hatS_{t,-1} \|^2 \right]\right)
\eqsp;
\end{align*}
By summing from time $k=0$ to $k=\kin-2$, we have (see
\Cref{lem:reorder})
\begin{align*}
\PE\left[\lyap(\hatS_{t+1,-1}) \right] &= \PE\left[\lyap(\hatS_{t,\kin-1}) \right]  \leq \PE\left[
  \lyap(\hatS_{t,0}) \right] - \Lambda_{t,\kin-1} \PE\left[
  \|H_{t,\kin-1}\|^2 \right] \\ & +  \frac{v^2_{\max}
  L^2}{2 \lbatch} \left( \sum_{\ell =0}^{\kin-2}
\frac{\pas_{t,\ell+1}}{\beta_{t,\ell+1}^2} \right) \PE\left[
  \|\hatS_{t,0} -\hatS_{t,-1} \|^2  \right] \\ & - \frac{v^2_{\max}}{2 \lbatch }
\sum_{k=0}^{\kin-2} \left( \sum_{\ell = k}^{\kin-2}
\frac{\pas_{t,\ell+1}}{\beta_{t,\ell+1}^2} \right) \PE\left[\| \bars \circ
  \map (\hatS_{t,k}) -\bars \circ \map (\hatS_{t,k-1}) \|^2 \right] \\ & +
\sum_{k=1}^{\kin-2} \left( \frac{L^2 v^2_{\max}}{2 \lbatch }
\pas_{t,k}^2 \left( \sum_{\ell = k}^{\kin-2}
\frac{\pas_{t,\ell+1}}{\beta_{t,\ell+1}^2} \right)- \Lambda_{t,k}\right) \,
\PE\left[ \| H_{t,k}\|^2 \right] \eqsp.
\end{align*}
With \eqref{eq:strict:contraction}, and using $H_{t,\kin} =
h(\hatS_{t,\kin-1})=h(\hatS_{t+1,-1})$; $\hatS_{1,0} = \hatS_{1,-1} = \hatS_\init$; and
for $t\geq 2$, $\hatS_{t,0} - \hatS_{t,-1} = \pas_{t-1,\kin}
h(\hatS_{t-1,\kin-1}) = \pas_{t-1,\kin} H_{t-1,\kin} = \pas_{t-1,\kin} H_{t,0}$:
\begin{align*}
& \PE\left[\lyap(\hatS_{t+1,0}) \right] - \PE\left[ \lyap(\hatS_{t,0})
    \right] \\ & \leq - \Lambda_{t,\kin-1} \PE\left[
    \|H_{t,\kin-1}\|^2 \right] + \frac{v^2_{\max} L^2}{2 \lbatch}
  \pas_{t-1,\kin}^2 \left( \sum_{\ell =0}^{\kin-2}
  \frac{\pas_{t,\ell+1}}{\beta_{t,\ell+1}^2} \right) \PE\left[ \|
    H_{t,0} \|^2 \right] \un_{t >1} \\ & - \frac{v^2_{\max}}{2 \lbatch
  } \sum_{k=0}^{\kin-2} \left( \sum_{\ell = k}^{\kin-2}
  \frac{\pas_{t,\ell+1}}{\beta_{t,\ell+1}^2} \right) \PE\left[\| \bars
    \circ \map (\hatS_{t,k}) -\bars \circ \map (\hatS_{t,k-1}) \|^2
    \right] \\ & + \sum_{k=1}^{\kin-2} \left( \frac{L^2 v^2_{\max}}{2
    \lbatch } \pas_{t,k}^2 \left( \sum_{\ell = k}^{\kin-2}
  \frac{\pas_{t,\ell+1}}{\beta_{t,\ell+1}^2} \right)-
  \Lambda_{t,k}\right) \, \PE\left[ \| H_{t,k}\|^2 \right]  -
  \pas_{t,\kin} v_{\min} \left(1 - \pas_{t,\kin} \frac{L_{\nabla
      \lyap}}{2 v_{\min}} \right) \PE\left[ \|H_{t+1,0}\|^2 \right]
  \\ & \leq -B_{t,k} \PE\left[\| \bars \circ \map (\hatS_{t,k}) -\bars
    \circ \map (\hatS_{t,k-1}) \|^2 \right] + \sum_{k=1}^{\kin-1}
  \left( \frac{L^2 v^2_{\max}}{2 \lbatch } \pas_{t,k}^2 \left(
  \sum_{\ell = k}^{\kin-2} \frac{\pas_{t,\ell+1}}{\beta_{t,\ell+1}^2}
  \right)- \Lambda_{t,k}\right) \, \PE\left[ \| H_{t,k}\|^2 \right]
  \\ & + \frac{v^2_{\max} L^2}{2 \lbatch} \pas_{t-1,\kin}^2 \left(
  \sum_{\ell =0}^{\kin-2} \frac{\pas_{t,\ell+1}}{\beta_{t,\ell+1}^2}
  \right) \PE\left[ \| H_{t,0} \|^2 \right] \un_{t >1} - \pas_{t,\kin}
  v_{\min} \left(1 - \pas_{t,\kin} \frac{L_{\nabla \lyap}}{2 v_{\min}}
  \right) \PE\left[ \|H_{t+1,0}\|^2 \right] \eqsp.
\end{align*}
We now sum from $t=1$ to $t=\kouter$.
\end{proof}

\begin{corollary}[of \Cref{prop:general:upperbound}]
  \label{coro:prop:general}
  Choose $\alpha>0$, $\beta >0$ such that
  \[
C(\alpha,\beta) \eqdef 1 - \frac{\beta^2}{2 v_{\min} } - \frac{\alpha}{2 v_{\min}} \frac{L_{\nabla
    \lyap}}{L} - \frac{\alpha^2
  v_{\max}^2}{2 \beta^2 v_{\min}} \frac{\kin}{\lbatch}
\]
is positive; and set
  \[
\pas_{t,k+1} \eqdef \frac{\alpha}{L} \eqsp, \qquad \beta_{t,k+1} \eqdef
\beta \eqsp.
\]
Then, for uniform random variables $\tau,\xi$ on $\{1, \cdots, \kouter
\}$ and $\{0, \cdots, \kin-1 \}$ respectively, independent from
$\F_{\kouter,\kin-1}$,
\[
\PE\left[ \|H_{\tau,\xi}\|^2 \right] \leq \frac{L}{\alpha v_{\min} C(\alpha,\beta)}
\frac{1}{\kin \kouter} \ \left( \PE\left[\lyap(\hatS_\init) \right] -
\min \lyap \right) \eqsp.
\]
  \end{corollary}
\begin{proof}
  We have
  \begin{align*}
A_{t,k} & \geq \frac{\alpha v_{\min}}{L} \left(1 - \frac{\beta^2}{2
  v_{\min} } - \frac{\alpha}{2 v_{\min}} \frac{L_{\nabla \lyap}}{L} -
\frac{\alpha^2 v_{\max}^2}{2 \beta^2 v_{\min}} \frac{\kin}{\lbatch}
\right) \eqsp, \\ B_{t,k} & \geq \frac{v_{\max}^2}{2 \lbatch}
\frac{\alpha}{L \beta^2} \kin \eqsp,
  \end{align*}
  from which the conclusion follows.
\end{proof}

\paragraph{Proof of \Cref{theo:rate:sqrtn}}
Let $\tau, \xi$ be uniform random variables resp. on $\{1,\cdots,
\kouter \}$ and $\{0, \cdots, \kin-1\}$. Since $\hatS_{1,-1} =
\hatS_{1,0}$ and for $t \geq 2$, $\hatS_{t,-1} = \hatS_{t-1, \kin-1}$,
then $\hatS_{t, \xi-1}$ is well defined. We write
\[
\PE\left[\|h(\hatS_{\tau,\xi-1})\|^2 \right] \leq 2 \PE\left[\|H_{\tau,\xi}\|^2
  \right] + 2 \, \PE\left[\|H_{\tau,\xi} - h(\hatS_{\tau,\xi-1})\|^2 \right] \eqsp.
\]
For the second term, we have
\begin{equation}\label{eq:randomstop}
\PE[\|H_{\tau,\xi} - h(\hatS_{\tau,\xi-1})\|^2 ] = \frac{1}{\kin \kouter}
\sum_{t=1}^{\kouter} \sum_{k=0}^{\kin-1} \PE[\|H_{t,k} - h(\hatS_{t,k-1})\|^2 ]
\end{equation}
by \Cref{prop:biasedfield}, since $\hatS_{1,0} = \hatS_{1,-1}$, the
RHS of \eqref{eq:randomstop} is upper bounded by
\[
 \frac{\alpha^2}{\lbatch } \frac{1}{\kouter} \sum_{t=1}^{\kouter} \sum_{k=0}^{\kin-1}
\PE\left[\|H_{t,k}\|^2 \right] \leq \frac{\alpha^2 \kin}{\lbatch }
\PE\left[\|H_{\tau,\xi}\|^2 \right] \eqsp.
\]
The proof is concluded by \Cref{coro:prop:general}:
\begin{equation} \label{eq:main:result:proof}
\PE[ \|h(\hatS_{\tau,\xi-1})\|^2] \leq \left( \frac{1}{\kin}+
\frac{\alpha^2}{\lbatch}\right) \frac{2 L}{\alpha v_{\min}
  C(\alpha,\beta)} \frac{1}{\kouter} \ \left( \PE[\lyap(\hatS_\init) ]
- \min \lyap \right) \eqsp.
\end{equation}
Let us choose $\beta>0$ so that $\beta \mapsto C(\alpha, \beta)$ is
maximal: for $A,B>0$, the function $x \mapsto x/A + B/x$ is minimal
at $x_\star \eqdef \sqrt{A B}$. This yields
\[
\beta^2(\alpha) \eqdef \alpha v_{\max} \sqrt{\frac{\kin}{\lbatch}} \eqsp,
\]
and
\[
v_{\min} \, C(\alpha, \beta(\alpha)) \eqdef v_{\min}  - \alpha \mu_\star \eqsp, \qquad \mu_\star \eqdef v_{\max} \sqrt{\frac{\kin}{\lbatch}}+ \frac{L_{\grad \lyap}}{2L} \eqsp.
\]
The function $\alpha \mapsto \alpha \, v_{\min} \, C(\alpha, \beta(\alpha))$ is maximal when $\alpha_\star \eqdef  v_{\min} /(2\mu_\star)$ thus yielding $\alpha_\star \, v_{\min}\, C(\alpha_\star, \beta(\alpha_\star)) = v^2_{\min}/(4 \mu_\star)$. By replacing $\beta \leftarrow \beta(\alpha)$ and $\alpha \leftarrow \alpha_\star$ in \eqref{eq:main:result:proof}, we have
 \begin{equation}
 \label{eq:bound:optimized}
\PE[ \|h(\hatS_{\tau,\xi-1})\|^2 ] \leq \left(\mu_\star+ \frac{\kin
  v_{\min}^2}{4 \mu_\star \lbatch}\right) \frac{8 L}{ v_{\min}^2 }
\frac{1}{\kin \, \kouter} \ \left( \PE[\lyap(\hatS_\init)] - \min
\lyap \right) \eqsp.
\end{equation}

 \subsection{On the Batch Size $\lbatch$ and Epoch Length $\kin$}
 \label{sec:choice:designparameters}
  Assume that $\lbatch = O(n^{\pa})$
 and $\kin = O(n^{\pc})$ for some $\pa, \pc \geq 0$. Let $\epsilon
 >0$.

 \paragraph{Case $\pa \geq \pc$.}
 When $n \to \infty$, $\mu_\star(\kin,\lbatch) =
 O(1)$.  Choose $\alpha \in \ooint{0,v_{\min}/\mu_\star(\kin,
   \lbatch)}$ such that $\alpha = O(n^{-\pd})$ for some $\pd \geq 0$.

 The RHS in \eqref{eq:upperbound:theo} is lower than $\epsilon$ by choosing
 \[
\kouter = O\left( \epsilon^{-1} n^{-\pc} \left(n^\pd + \frac{1}{n^{\pd + \pa - \pc}} \right) \right) \eqsp;
 \]
this implies that
 \[
K_{\operatorname{CE}}(n,\epsilon) = O \left(  n+(n + n^{\pa + \pc})
\kouter\right) \eqsp, \qquad K_{\operatorname{Opt}} (n,\epsilon) =
O\left(1+(1+n^{\pc}) \kouter \right) \eqsp.
 \]
In order to make $\kouter$ as small as possible, we choose $\pd = 0$
and $\pc$ as large as possible (\ie\ $\pa =\pc$).  Hence $\kouter =
O(\epsilon^{-1} n^{-\pa})$.  This implies that $K_{\operatorname{Opt}}
(n,\epsilon) = O(\epsilon^{-1})$. For fixed $\pa \geq 0$,
$K_{\operatorname{CE}}(n,\epsilon)$ is optimized by choosing $\pa \leq
1-\pa$, which implies $\pa \leq 1/2$. The largest value of $\pa$ will
provide the best rate for $\kouter$. Hence, the conclusion is
\[
\pa = \pc = 1/2, \qquad \pd = 0,
\]
which yields $\lbatch = O(\sqrt{n})$, $\kin = O(\sqrt{n})$,
$\kouter = O(\epsilon^{-1} n^{-1/2})$, $K_{\operatorname{CE}} (n,\epsilon) = O(n+\epsilon^{-1} \sqrt{n})$ and $K_{\operatorname{Opt}}(n,\epsilon) = O(\epsilon^{-1})$.

\paragraph{Case $\pa < \pc$.}
 When $n \to \infty$, $\mu_\star(\kin,\lbatch) = O(n^{(\pc -
   \pa)/2})$.  Choose $\alpha \in \ooint{0,v_{\min}/\mu_\star(\kin,
   \lbatch)}$ such that $\alpha = O(n^{-\pd})$ for some $\pd \geq (\pc
 - \pa)/2$.

The RHS in \eqref{eq:upperbound:theo} is lower than $\epsilon$ by choosing
 \[
\kouter = O\left( \epsilon^{-1} n^{-\pc} \left(n^\pd + \frac{1}{n^{\pd + \pa - \pc}} \right) \right) \eqsp;
 \]
we also have
 \[
K_{\operatorname{CE}}(n,\epsilon) = O \left( n+ (n + n^{\pa + \pc})
\kouter\right) \eqsp, \qquad K_{\operatorname{Opt}} (n,\epsilon) =
O\left(1+(1+n^{\pc}) \kouter \right) \eqsp.
 \]
In order to make $\kouter$ as small as possible, we choose $\pd =
(\pc-\pa)/2$ so $\kouter = O(\epsilon^{-1} n^{-(\pa+\pc)/2})$, and
then we choose $\pc+\pa$ as large as possible.  Hence This implies
that $K_{\operatorname{Opt}} (n,\epsilon) = O(\epsilon^{-1}
n^{(\pc-\pa)/2})$ and $K_{\operatorname{Opt}}(n,\epsilon)$ is
optimized by choosing $\pc - \pa$ as small as possible. Finally,
$K_{\operatorname{CE}}(n,\epsilon)$ is optimized with $\pa + \pc \leq
1$. Hence, the conclusion is: choose $\delta>0$ and set
\[
\pa = (1- \delta)/2, \qquad \pc = (1+\delta)/2, \qquad \pd = \delta/2,
\]
which yields $\lbatch = O(n^{1/2-\delta/2})$, $\kin =
O(n^{1/2+\delta/2})$, $\kouter = O(\epsilon^{-1} n^{-1/2})$,
$K_{\operatorname{CE}} (n,\epsilon) = O(n+\epsilon^{-1}
\sqrt{n})$ and $K_{\operatorname{Opt}}(n,\epsilon) =
O(\epsilon^{-1} n^{\delta/2})$.

\paragraph{Conclusion.} The above discussion shows that the best complexity in terms of the number
of computations of per-sample conditional expectations and the one in
terms of number of parameter updates are both optimized in the case
$\pa = \pc = 1/2$.

\section{Linear convergence rate of {\tt SPIDER-EM-PL}}
\label{sec:linear:cvgrate}
In this section, we establish a linear convergence rate of a slightly
modified version of {\tt SPIDER-EM}, see \autoref{algo:SPIDER-EM-PL},
the main modification being in the initialization.  The proof is
adapted from \cite[Theorem 5]{wang:etal:nips:2019}.
\setlength{\algomargin}{1.5em}
\begin{algorithm}[htbp]
  \KwData{$\kin \in \nset_\star$, $\kouter \in \nset_\star$, $\hatS_\init
    \in \Sset$, $\{\pas_{t,k+1}, t=1, \cdots, \kouter$ and $k = 0, \cdots, \kin-1\}$
    positive sequence.}
    \KwResult{A  {\tt SPIDER-EM-PL} sequence: $\hatS_{t,k},
    t=1, \cdots, \kouter$, $k=0, \ldots, \kin-1$}{ $\Smem_{1,0} = \bars \circ \map(\hatS_\init)$, $\hatS_{1,0} = \hatS_{1,-1} = \hatS_\init$}
    \;\For{$t=1, \ldots, \kouter$}{
  Sample    $\xi_t$  a uniform random variable on $\{1, \cdots, \kin-1\}$ \;
  \For{$k=0, \cdots, \xi_t-1$}
      {
         Sample a mini-batch $\batch_{t,k+1}$ in $\{1, \ldots,
      n\}$ of size $\lbatch$, with or without replacement \;
 $ \Smem_{t,k+1} =
      \Smem_{t,k} + \bars_{\batch_{t,k+1}} \circ \map( \hatS_{t,k} ) - \bars_{\batch_{t,k+1}} \circ \map( \hatS_{t,k-1} )$  \;
      $\hatS_{t,k+1} = \hatS_{t,k} +  \pas_{t,k+1} \big(\Smem_{t,k+1}- \hatS_{t,k}\big)$
}
     $\hatS_{t+1,0} = \hatS_{t+1,-1} = \hatS_{t,\xi_t}$  \;
$\Smem_{t+1,0} = \bars \circ \map(\hatS_{t,\xi_t
      })$}
    \caption{The {\tt SPIDER-EM-PL} algorithm. \label{algo:SPIDER-EM-PL}}
\end{algorithm}

By \Cref{prop:linearrate:general}, we have
\begin{proposition} \label{prop:linearrate}
  Assume \Cref{hyp:model}, \Cref{hyp:bars}, \Cref{hyp:Tmap} and
  \Cref{hyp:regV} and \Cref{hyp:regV:bis}. Set
  $L^2 \eqdef n^{-1} \sum_{i=1}^n L_i^2$. For any integers $t \geq 1$ and $K \geq 2$
\begin{align*}
  \sum_{\ell = 1}^{K-1}  & \delta_{t,\ell} \,  \PE\left[\| \hatS_{t,\ell} - \hatS_{t,\ell-1} \|^2 \vert \F_{t,0} \right]
    + \frac{v_{\min}}{2}\sum_{k=0}^{K-2} \pas_{t,k+1} \PE\left[ \|h(\hatS_{t,k})\|^2 \vert \F_{t,0} \right]   \\
  & \qquad \leq  \PE\left[\lyap(\hatS_{t,0}) - \lyap(\hatS_{t,K-1})  \vert \F_{t,0} \right]\eqsp,
\end{align*}
where (by convention, $\sum_{\ell=K-1}^{K-2} =0$),
\[
\delta_{t,\ell} \eqdef \left( \frac{v_{\min}}{2\pas_{t,\ell}} - \frac{L_{\grad \lyap}}{2} - \frac{v_{\max}}{2} \frac{L^2}{\lbatch} \sum_{k=\ell}^{ K-2}  \pas_{t,k+1}   \right) \eqsp.
\]
\end{proposition}

\begin{corollary}[of \Cref{prop:linearrate}]
\label{coro:inearrate}
For any $\pas>0$ such that
\[
\pas^2 + \frac{L_{\grad \lyap} \, \lbatch}{v_{\max}L^2 (K-1)} \pas - \frac{v_{\min} \lbatch}{v_{\max}L^2 (K-1)} <0 \eqsp,
\]
we have
\[
 \frac{v_{\min} \pas}{2}\sum_{k=0}^{K-1}  \PE\left[ \|h(\hatS_{t,k})\|^2 \vert \F_{t,0} \right]   \leq  \PE\left[\lyap(\hatS_{t,0}) - \lyap(\hatS_{t,K})  \vert \F_{t,0} \right]\eqsp.
\]
\end{corollary}
As a consequence of \Cref{coro:inearrate}, if $\xi_t$ is a uniform
random variable on $\{1, \cdots, \kin-1\}$ independent of the other random variables, then
\[
 \PE\left[ \|h(\hatS_{t,\xi_t})\|^2
\right] \leq   \frac{2}{v_{\min} \pas (\kin-1)} \PE\left[\lyap(\hatS_{t,0}) - \min \lyap \right] \eqsp.
\]
When the Polyak-Lojasiewicz inequality holds
\begin{equation}\label{eq:PL}
\exists \tau^\star >0 \  \text{such that} \ \forall s, \lyap(s) - \min \lyap \leq \tau^\star \, \|\nabla \lyap(s)\|^2 \eqsp,
\end{equation}
this yields by  \Cref{hyp:regV:bis}-\Cref{hyp:regV:C1:vmax}
\[
 \PE\left[ \|h(\hatS_{t,\xi_t})\|^2
\right] \leq  \frac{2}{v_{\min} \pas (\kin-1)} \PE\left[\lyap(\hatS_{t,0}) - \min \lyap \right] \leq \frac{2 \tau^\star v_{\max}^2}{v_{\min} \pas (\kin-1)} \PE\left[ \| h(\hatS_{t,0})\|^2\right] \eqsp.
\]
The above discussion establishes the following result.
\begin{theorem}
  Assume \Cref{hyp:model}, \Cref{hyp:bars}, \Cref{hyp:Tmap},
  \Cref{hyp:regV} and \Cref{hyp:regV:bis} and set
  $L^2 \eqdef n^{-1} \sum_{i=1}^n L_i^2$. Assume also that the
  Polyak-Lojasiewicz inequality \eqref{eq:PL} holds.  Fix
  $\kouter, \kin \in \nset_\star$, $\lbatch \in \nset_\star$; set
  $\pas_{t,k+1} \eqdef \pas$ for any $t \geq 1, k \geq 0$ for some $\pas>0$
  satisfying
\[
\pas^2 + \frac{L_{\grad \lyap} \,  \lbatch}{v_{\max}L^2 (\kin-1)} \pas - \frac{v_{\min} \lbatch}{v_{\max}L^2 (\kin-1)} < 0 \eqsp.
\]
Let $\{\hatS_{t,k}, t=1, \cdots, \kouter, k=0, \cdots, \xi_{t} \}$ be the sequence given by \autoref{algo:SPIDER-EM-PL}. Then
\[
\PE\left[ \|h(\hatS_{t+1,0})\|^2
\right]  = \PE\left[ \|h(\hatS_{t,\xi_t})\|^2
\right] \leq \frac{2 \tau^\star v_{\max}^2}{v_{\min} \pas (\kin-1)} \PE\left[ \| h(\hatS_{t,0})\|^2\right] \eqsp.
\]
\end{theorem}

\section{Mixture of Gaussian distributions}
\label{sec:MNIST}
In this section, we use the common notation $\{\hatS_\ell, \ell \geq 0\}$ for a path. For {\tt sEM-vr} and {\tt SPIDER-EM}, $\hatS_\ell$ stands for $\hatS_{t_\ell, k_\ell}$ where $t_\ell \geq 1$ and $k_\ell \in \{0, \cdots, \kin-1 \}$ are the   unique integers such that $\ell = (t_\ell-1) \kin + k_\ell$.

\subsection{The model}
Consider a mixture of Gaussian distributions on $\rset^p$,
\begin{equation}
  \label{eq:mixt:gauss:density}
y \mapsto \sum_{\ell=1}^g \alpha_\ell \, \mathcal{N}_p\left(\mu_\ell, \Sigma
\right)[y] \eqsp;
\end{equation}
$\mathcal{N}_p\left(\mu_\ell, \Sigma \right)[y]$ denotes the density
of a $\rset^p$-valued Gaussian distribution with expectation
$\mu_\ell$, covariance matrix $\Sigma$ and evaluated at $y \in
\rset^p$.  We consider a parametric statistical model indexed by
$\param \eqdef (\alpha_1, \ldots, \alpha_g, \mu_1, \ldots, \mu_g,
\Sigma)$ in $\Param$ where
\begin{equation} \label{eq:MNIST:def:Param}
\Param \eqdef \left\{ \alpha_\ell \geq 0, \sum_{\ell=1}^g \alpha_\ell
= 1 \right\} \times \rset^{pg} \times \mathcal{M}_p^+ \eqsp;
\end{equation}
$\mathcal{M}_p^+ $ denotes the set of positive definite $p \times p$ matrices.

Given $n$ examples $y_1, \ldots, y_n$ modeled as independent
realizations of a mixture of Gaussian distributions as described by
\eqref{eq:mixt:gauss:density}, the log-likelihood is
\[
\param \mapsto \sum_{i=1}^n \log  \sum_{\ell=1}^g \alpha_\ell \,
\mathcal{N}_p\left(\mu_\ell, \Sigma \right)[y_i] \eqsp.
\]
\Cref{prop:MNIST:defF} shows that the minimization of the
negative log-likelihood on $\Param$ is covered by the optimization
problem addressed in the paper.
\begin{proposition} \label{prop:MNIST:defF}
  Set $\Gamma \eqdef \Sigma^{-1}$, and define for $y \in \rset^p$ and
  $z \in \{1, \ldots, g \}$,
  \begin{equation*}
  \A_y  \eqdef \left[ \begin{matrix} \Id_{g}
      \\ \Id_g\otimes y \end{matrix}\right] \in \rset^{g(1+p) \times
  g}  \eqsp,  \qquad
 \rho(z)  \eqdef \left[ \begin{matrix} \1_{z=1} \\ \ldots
     \\ \1_{z=g} \end{matrix}\right]  \eqsp.
  \end{equation*}
 The negative normalized log-likelihood is of the form
  \eqref{eq:intro:F} with $\rho(y,z) =1$, $\s(y,z) \eqdef \A_y \,
  \rho(z)$ and
  \begin{align}
    \phi(\param) & \eqdef \left[ \begin{matrix}
        \ln \alpha_1 - 0.5 \mu_1^T \Gamma \mu_1 \\
        \ldots \\
        \ln \alpha_g - 0.5 \mu_g^T \Gamma \mu_g \\
        \Gamma \mu_1 \\
        \ldots \\
        \Gamma \mu_g
      \end{matrix} \right]\eqsp, \\ \psi(\param) & \eqdef \frac{p}{2} \ln(2 \pi) +\frac{1}{2}
    \mathrm{Tr}\left(\frac{\Gamma}{n} \sum_{i=1}^n y_i y_i^T \right) -
    \frac{1}{2}\ln \mathrm{det}(\Gamma) \eqsp.
\end{align}
\end{proposition}
\begin{proof}
  The likelihood of a single observation $y_i$ is given by
  \begin{align*}
\param & \mapsto \frac{1}{\sqrt{2\pi}^p} \sum_{z = 1}^g \alpha_z
\sqrt{\mathrm{det}(\Gamma)} \exp\left(- \frac{1}{2}(y_i-\mu_z)^T
\Gamma (y_i-\mu_z) \right) \\ & =
\frac{\sqrt{\mathrm{det}(\Gamma)}}{\sqrt{2\pi}^p} \exp\left(-
\frac{1}{2} y_i^T \Gamma y_i \right) \sum_{z = 1}^g
\exp\left(\sum_{\ell=1}^g \1_{z = \ell} \left\{ \ln \alpha_\ell - 0.5
\mu^T_\ell \Gamma \mu_\ell + \mu_\ell^T \Gamma y_i\right\}\right) \\ &
= \frac{\sqrt{\mathrm{det}(\Gamma)}}{\sqrt{2\pi}^p} \exp\left(-
\frac{1}{2} \mathrm{Tr}(\Gamma y_i y^T_i) \right) \sum_{z = 1}^g
\exp\left(\sum_{\ell=1}^g \1_{z = \ell} \{\ln \alpha_\ell - 0.5
\mu^T_\ell \Gamma \mu_\ell \} + \sum_{\ell=1}^g \pscal{ \Gamma
  \mu_\ell}{ y_i \1_{z = \ell}} \right) \\ & =
\frac{\sqrt{\mathrm{det}(\Gamma)}}{\sqrt{2\pi}^p} \exp\left(-
\frac{1}{2} \mathrm{Tr}(\Gamma y_i y^T_i) \right) \sum_{z = 1}^g
\exp\left( \pscal{s(y_i, z)}{\phi(\param)} \right)
  \end{align*}
  where we used that $\mathrm{Tr}(A u u^T) = u^T A u$. Since the
  observations are modeled as independent, the log-likelihood of the
  $n$ observations $y_1, \ldots, y_n$ is
  \[
\param \mapsto \frac{n}{2} \left( \log \mathrm{det}(\Gamma) - p
\log(2\pi) \right) - \frac{1}{2} \mathrm{Tr}(\Gamma \sum_{i=1}^n y_i
y^T_i) +   \sum_{i=1}^n  \log \sum_{z =1}^g \exp\left( \pscal{s(y_i,
  z)}{\phi(\param)} \right) \eqsp.
\]
This yields the expression of the negative normalized log-likeliood.
\end{proof}

The following statement gives the expression of the optimization map
$\map$. It relies on standard computations; the proof is omitted.
\begin{proposition}
  \label{prop:MNIST:mapT}
  Let $\phi,\psi$ and $\Param$ resp. given by \Cref{prop:MNIST:defF} and \eqref{eq:MNIST:def:Param}.  For
  any $s = (s_1, \ldots, s_{g + pg}) \in \rset^{g + p g}$ in the set
  \[\left(s_1 > 0, \ldots, s_g > 0,  \frac{1}{n} \sum_{i=1}^n y_i y_i^T - \sum_{\ell=1}^g s_\ell^{-1}  s_{g+(\ell-1)p+1: g+\ell p}  \  s_{g+(\ell-1)p+1: g+\ell p}^T  \ \text{positive definite}
  \right)
  \]
  the minimizer of $\param \mapsto -\pscal{s}{\phi(\param)} +
  \psi(\param)$ under the constraint that $\param \in \Param$, exists
  and is unique and is given by
  \begin{align*}
    \alpha_\ell & \eqdef \frac{s_\ell}{\sum_{u=1}^g s_u} \eqsp, \qquad
    \ell=1, \ldots, g \eqsp, \\ \mu_\ell & \eqdef \frac{1}{s_\ell}
    s_{g+(\ell-1)p+1: g+\ell p} \eqsp, \qquad \ell=1, \ldots, g \eqsp, \\
    \Sigma^{-1} & \eqdef \frac{1}{n} \sum_{i=1}^n y_i y_i^T - \sum_{\ell=1}^g s_\ell \mu_\ell \mu_\ell^T \eqsp.
    \end{align*}
\end{proposition}
\Cref{prop:MNIST:barsi} provides the expression of the
conditional probabilities $z \mapsto p(z \vert y_i; \param)$ on $\{1,
\ldots, g \}$; as a corollary of this statement, we also have the
expression of the per sample conditional expectations
\[
\bars_i(\param) \eqdef \sum_{z=1}^g s(y_i,z) \ p(z \vert y_i; \param)
\eqsp,
\]
for all $i=1, \ldots, n$.
\begin{proposition}
  \label{prop:MNIST:barsi}
 For any $y \in \rset^p$, $ z \in \{1, \ldots, g \}$ and $\param \in
 \Param$ where $\Param$ is defined by \eqref{eq:MNIST:def:Param}, we
 have
 \begin{equation}
   \label{eq:MNIST:condproba}
p(z \vert y ; \param) \eqdef \frac{\alpha_z \ \mathcal{N}_p(\mu_z,
  \Sigma)[y]}{\sum_{u=1}^g \alpha_u \ \mathcal{N}_p(\mu_u, \Sigma)[y]}
\eqsp,
\end{equation}
and
\[
\sum_{z=1}^g s(y,z) \ p(z \vert y; \param) = \left[ \begin{matrix} p(1
    \vert y; \param) \\ \ldots \\ p(g \vert y; \param) \\ y \ p(1 \vert
    y; \param) \\ \ldots \\ y \ p(g \vert y; \param)
  \end{matrix}\right] \eqsp,
\]
where $s(y,z)$ is defined in \Cref{prop:MNIST:defF}.
\end{proposition}
As a corollary of this statement, we have
\begin{align}
\bars_i(\param) & \eqdef \left[ \begin{matrix} p(1 \vert y_i; \param)
    \\ \ldots \\ p(g \vert y_i; \param) \\ y_i \ p(1 \vert y_i; \param)
    \\ \ldots \\ y_i \ p(g \vert y_i; \param)
  \end{matrix}\right]  = \A_{y_i} \, \left[ \begin{matrix} p(1 \vert y_i; \param)
    \\ \ldots \\ p(g \vert y_i; \param)
  \end{matrix}\right]  \eqsp, \nonumber\\
\bars(\param) & \eqdef \left[ \begin{matrix} n^{-1} \sum_{i=1}^n p(1 \vert y_i; \param)
    \\ \ldots \\ n^{-1} \sum_{i=1}^n p(g \vert y_i; \param) \\ n^{-1} \sum_{i=1}^n y_i \ p(1 \vert y_i; \param)
    \\ \ldots \\ n^{-1} \sum_{i=1}^n y_i \ p(g \vert y_i; \param)
  \end{matrix}\right] = \frac{1}{n} \sum_{i=1}^n \A_{y_i} \,  \left[ \begin{matrix} p(1 \vert y_i; \param)
    \\ \ldots \\ p(g \vert y_i; \param)
  \end{matrix}\right]\eqsp, \label{eq:MNIST:bars:LC}
\end{align}
where the probability $p(\cdot \vert y; \param)$ is given by
\eqref{eq:MNIST:condproba}.

\subsection{On the Assumption \Cref{hyp:Tmap}}
Let $\A_y$ be the matrix defined in \Cref{prop:MNIST:defF}. It is proved in \cite[Section
  5]{fort:gach:moulines:2020} that $\map(s) \in \Param$ if
\[
s \in \mathcal{S} \eqdef \left\{s = \frac{1}{n} \sum_{i=1}^n \A_{y_i}
\, \rho_i, \rho_i =(\rho_{i,1}, \ldots, \rho_{i,g}) \in (\rset_+)^g,
\sum_{\ell=1}^g \rho_{i,\ell} =1 \right\} \eqsp.
\]
 The following statement shows that the {\tt SPIDER-EM} sequence
 $\{\hatS_k, k \geq 0 \}$ is at least in
\[
\widetilde{\mathcal{S}} \eqdef \left\{s = \frac{1}{n} \sum_{i=1}^n
\A_{y_i} \, \rho_i, \rho_i =(\rho_{i,1}, \ldots, \rho_{i,g}) \in
\rset^g, \sum_{\ell=1}^g \rho_{i,\ell} =1 \right\} \eqsp.
\]
\begin{proposition}
  \label{prop:MNIST:welldefinedT}
  Assume that $\hatS_\init \in \mathcal{S}$.  Then, for any $t \in \nset$,
  $\Smem_{t,0} \in \mathcal{S}$ and for any $k \geq 0$, $\hatS_{t,k}
  \in \widetilde{\mathcal{S}}$ and $\Smem_{t,k} \in
  \widetilde{\mathcal{S}}$.
  \end{proposition}
\begin{proof}
It is trivially seen from \eqref{eq:MNIST:bars:LC} that $\Smem_{t,0} \in \mathcal{S}$ for any $t \in \nset$.  Define $\rho^{(t,0)}_i
\in (\rset_+)^g$ and $\hat \rho^{(t,0)}_i \in (\rset_+)^g$ such that
\[
\Smem_{t,0} = \frac{1}{n} \sum_{i=1}^n \A_{y_i} \, \rho_i^{(t,0)} \eqsp,
\qquad \hatS_{t,0} = \frac{1}{n} \sum_{i=1}^n \A_{y_i} \, \hat
\rho_i^{(t,0)} \eqsp;
\]
note that by \eqref{eq:MNIST:bars:LC}, $\sum_{\ell =1}^g
\rho_{i,\ell}^{(t,0)} = 1$ and by assumption, $\sum_{\ell =1}^g \hat
\rho_{i,\ell}^{(t,0)} = 1$.

From \autoref{line:algo2:updateSmem} of \autoref{algo:SPIDER-EM}, we
have when $k < \kin-1$,
\begin{align*}
\Smem_{t,k+1} &= \frac{1}{n} \sum_{i=1}^n \A_{y_i} \left( \rho_i^{(t,k)} +
\frac{n}{\lbatch} \1_{i \in \batch_{t,k+1}} \left\{ p(\cdot \vert y_i;
\map(\hatS_{t,k})) -p(\cdot \vert y_i; \map(\hatS_{t,k-1})) \right\} \right)
\end{align*}
where $p(\cdot \vert y; \param)$ is defined by \eqref{eq:MNIST:condproba}, thus
implying that
\[
\rho^{(t,k+1)}_i = \rho_i^{(t,k)} + \frac{n}{\lbatch} \1_{i \in
  \batch_{t,k+1}} \left\{ p(\cdot \vert y_i; \map(\hatS_{t,k})) -p(\cdot
\vert y_i; \map(\hatS_{t,k-1})) \right\} \eqsp.
\]
Hence by a trivial induction, $\sum_{\ell=1}^g \rho^{(t,k+1)}_{i,\ell} =
1$ for any $i=1, \ldots, n$. From \autoref{line:algo2:updateShat} and
\autoref{line:algo2:updateShatbis} of \autoref{algo:SPIDER-EM}, we
have for any $k \geq 0$,
\begin{align*}
\hatS_{t,k+1} &= \frac{1}{n} \sum_{i=1}^n \A_{y_i} \left( (1-\pas_{t,k+1})
\hat \rho_i^{(t,k)} + \pas_{t,k+1} \rho_i^{(t,k+1)} \right)
\end{align*}
 thus
implying that
\[
\hat \rho^{(t,k+1)}_i = (1-\pas_{t,k+1}) \hat \rho_i^{(t,k)} + \pas_{t,k+1}
\rho_i^{(t,k+1)} \eqsp.
\]
Here again, by a trivial induction, we have $\sum_{\ell=1}^g
\hat \rho^{(t,k+1)}_{i,\ell} = 1$ for any $i=1, \ldots, n$.
  \end{proof}

\subsection{Numerical Analysis}

\subsubsection{The data set}
We consider $n=6 \times 10^4$ observations in $\rset^{p}$, $p=20$; modeled
as independent observations from a mixture of Gaussian distributions
with $g=12$ components. These data are obtained from the MNIST data
training set available at http://yann.lecun.com/exdb/mnist.

The set contains $n = 6 \times 10^4$ examples of size $28 \times 28$;
among these pixels, $67$ are constant over all the images and are removed yielding to observations of
length $717$. A PCA is performed in order to reduce the dimensionality
to $p=20$ features.

\subsubsection{The algorithms}
We compare {\tt EM}, {\tt iEM}, {\tt Online EM}, {\tt FIEM} and {\tt
 sEM-vr} implemented as described in \autoref{algo:EM} to
\autoref{algo:SEMVR}. The map $\map$ is given by \Cref{prop:MNIST:mapT}.

The design parameters $\lbatch, \pas_{t,k+1}$ are fixed to
\begin{itemize}
  \item $\lbatch =100$,
\item for all the algorithms except {\tt iEM}, the step size is constant and equal to $5 \,
  10^{-3}$. In {\tt iEM}, $\pas_{k+1}=1$.
\end{itemize}

{\bf Initialization.} For all the algorithms and all the paths, the
same initial value $\hatS_\init$ is considered. It is obtained as follows:
we run the random initialization technique described in
\cite{kwedlo:2015} in order to obtain $\param_\init \in \Param$, and
then we set $\hatS_\init \eqdef \bars(\param_\init)$. Below,
$\hatS_\init$ is such that $-\lyap(\hatS_\init) = -58.3097$ (the
constant term $p \log(2\pi)/2$ is omitted in this evaluation, and in
any evaluation of the log-likelihood given below).

{\bf Mini-batch.} The mini-batches are independent, and sampled at
random in $\{1, \ldots, n\}$ with replacement.  For a fair comparison
of the algorithms, they share the same seed; another seed is used for
{\tt FIEM} which requires a second sequence of minibatches $\{
\overline{\batch}_{k+1}, k \geq 0\}$.

{\bf An epoch.} In the analyses below, {\em an epoch} is defined as
the selection of $n$ examples:
\begin{itemize}
  \item For {\tt EM}, an epoch is one iteration $\hatS_k \to
    \hatS_{k+1}$. It necessitates the computation of $n$ conditional
    expectations $\bars_i$ and of a single optimization $\map(\hatS)$.
\item For {\tt iEM} and {\tt Online EM}, an epoch is $n/\lbatch$ iterations $\hatS_k
  \to \hatS_{k+1}$. It necessitates the computation of $n$ conditional
  expectations $\bars_i$ and of $n/\lbatch$ optimizations
  $\map(\hatS)$.
\item For {\tt FIEM}, an epoch is $n /\lbatch$ iterations $\hatS_k \to
  \hatS_{k+1}$. It necessitates the computation of $2 n$ conditional
  expectations $\bars_i$ and of $n/\lbatch$ optimizations
  $\map(\hatS)$.
\item For {\tt sEM-vr} and {\tt SPIDER-EM}, an epoch is either one
  iteration $\hatS_{t, -1} \to \hatS_{t,0}$ or $n/\lbatch$
  iterations $\hatS_{t,k} \to \hatS_{t,k+1}$ for $k < \kin-1$. They
  resp. necessitate the computation of $n$ and $2 n /\lbatch$
  conditional expectations $\bars_i$ and of $1$ and $n/\lbatch$
  optimizations $\map(\hatS)$.
 \end{itemize}

{\bf Hybrid methods.} Since {\tt FIEM}, {\tt sEM-vr} and {\tt
  SPIDER-EM} are variance reduction methods w.r.t. {\tt Online EM}, we
advocate to combine them with few steps of {\tt Online EM}. Here, we
start with $\kswitch =2$ epochs of {\tt Online EM} and obtain
$\hatS_1, \hatS_2$; before switching to {\tt FIEM}, {\tt sEM-vr} and
{\tt SPIDER-EM}.

{\bf Value of $\kmax$.} The number $\kmax$ is fixed in order to
compare the algorithms with the same number of epochs equal to
$150$. For {\tt EM}, $\kmax = 150$; for {\tt Online EM} and {iEM},
$\kmax = 150 \, n /\lbatch$; for {\tt FIEM}, $\kmax =    (150-\kswitch) \, n /
\lbatch$; for {\tt sEM-vr}, $\kouter =  (150-\kswitch)/2$ and $\kin =
1+n/\lbatch$; and for {\tt SPIDER-EM}, $\kouter = (150-\kswitch)/2$ and $\kin =
1+n/\lbatch$.

\subsubsection{Experimental Results}
We first analyze the behavior of the functional $\lyap$ along a path
of the algorithm. We display on \Cref{fig:LogLikelihoodMean} a Monte
Carlo approximation, computed from $40$ independent runs, of the
expectation of the normalized log-likelihood  as a function of the
number of epochs. Different algorithms are considered: {\tt EM}
remains trapped in a local extremum while the stochastic EM algorithms
succeed in exiting to a better limiting point. {\tt Online EM} is far
more variable than {\tt iEM}, {\tt FIEM}, {\tt sEM-vr} and {\tt
  SPIDER-EM}. The convergence of {\tt iEM} is longer, when compared to
{\tt FIEM}, {\tt sEM-vr} and {\tt SPIDER-EM}.

\begin{figure}[t]
\centering
\includegraphics[width=0.4\textwidth]{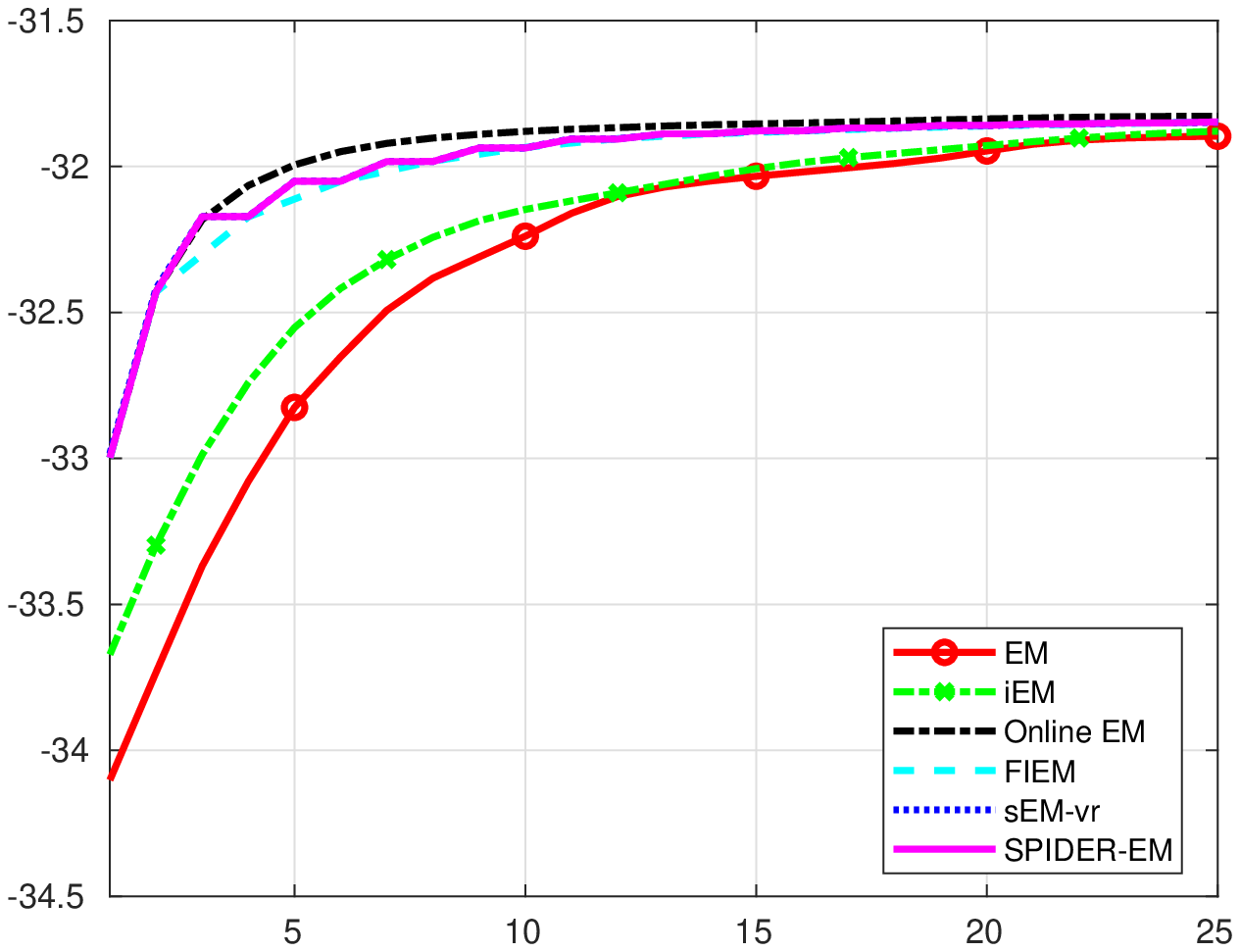}\includegraphics[width=0.4\textwidth]{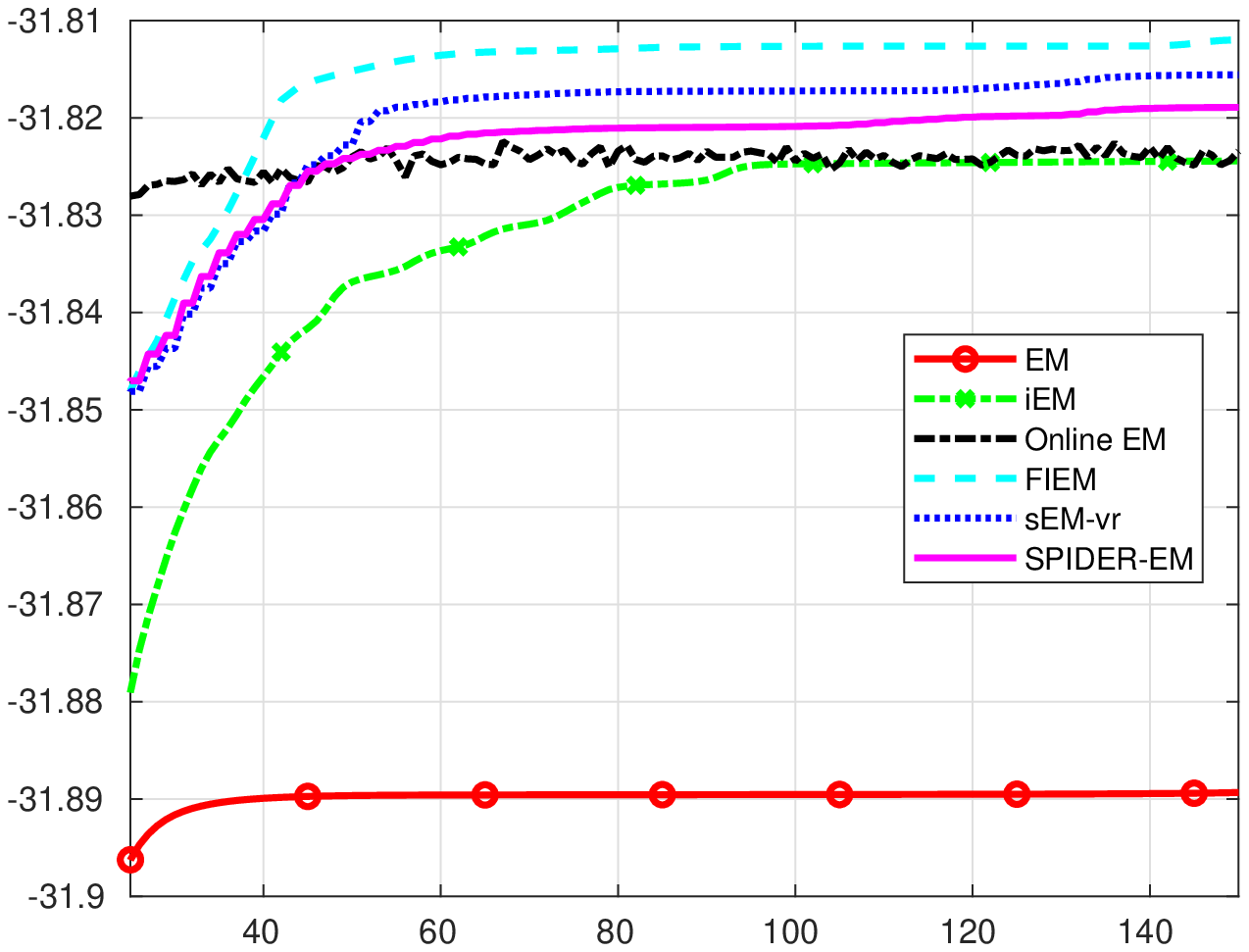}
\caption{Monte Carlo approximation (computed over $40$ independent
  runs) of $- \PE [\lyap(\hatS_\ell) ] = - \PE [F \circ
    \map(\hatS_\ell)]$ against the number of epochs.
  [left] Epochs $1$ to $25$; [right] epochs $25$ to $150$.}
\label{fig:LogLikelihoodMean}
\end{figure}

On \Cref{fig:LogLikelihoodZoom} and \Cref{fig:LogLikelihoodZoomBis},
for each of the algorithms {\tt FIEM}, {\tt sEM-vr} and {\tt
  SPIDER-EM}, four different realizations of a path of the normalized
likelihood are displayed as a function of the number of epochs. These
four sets of curves differ from the selection of the sequence of
mini-batches. The staircase behavior of the paths of {\tt sEM-vr} and
{\tt SPIDER-EM} comes from the two successive kinds of epoch: one
corresponds to a single optimization and a full scan of the data set
and the other one corresponds to $n/\lbatch$ optimizations and the use
of $n/\lbatch$ minibatches; the largest increase of $\lyap$
corresponds to the second type of epoch.  Based on this criterion, the
three algorithms are equivalent.

\begin{figure}[t]
\centering
  \includegraphics[width=0.4\textwidth]{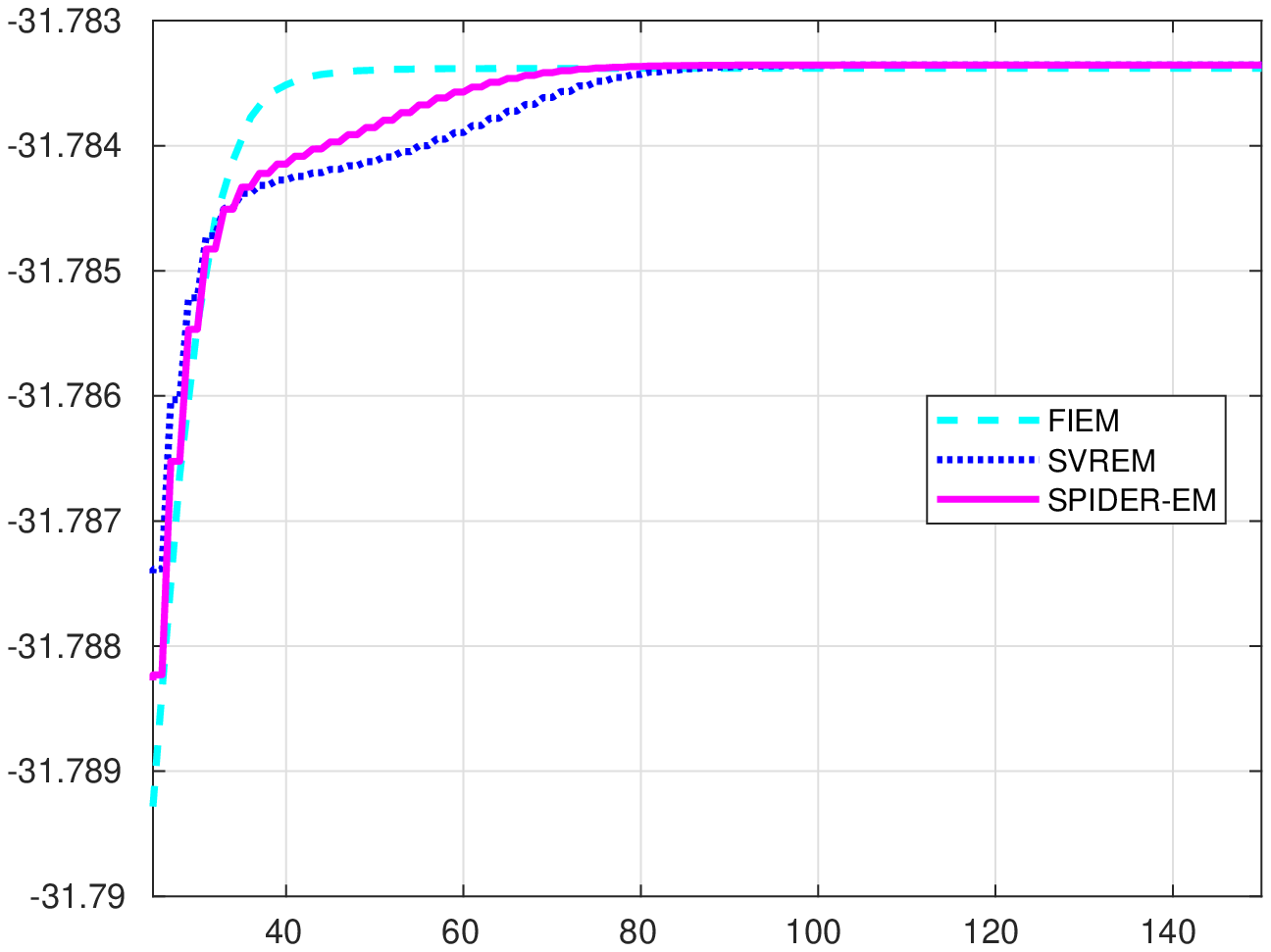}\includegraphics[width=0.4\textwidth]{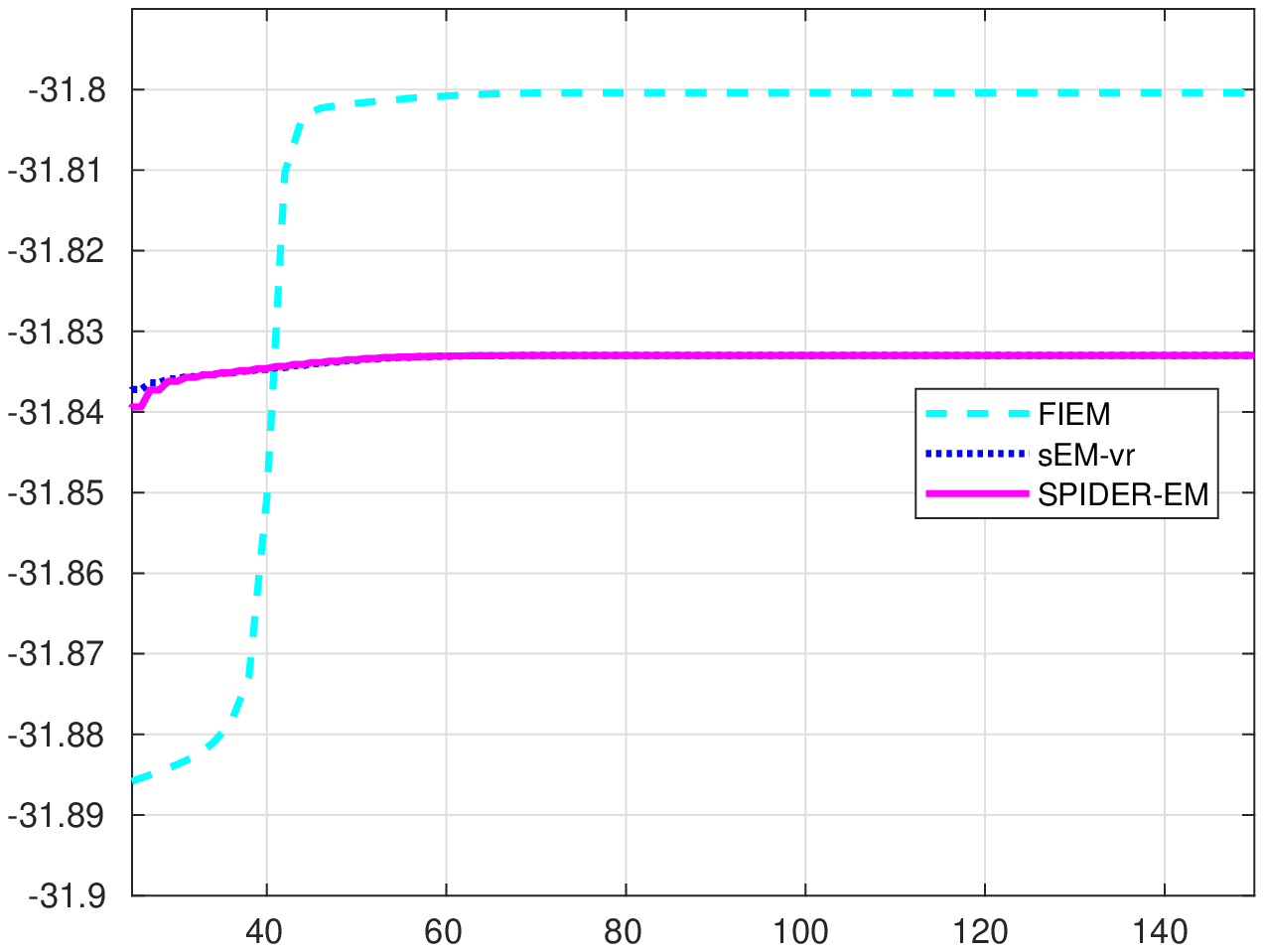}
\caption{The objective function $-\lyap(\hatS_\ell) = -F \circ
  \map(\hatS_\ell)$ against the number of epochs along two
  (left, right) independent runs of {\tt FIEM}, {\tt sEM-vr} and {\tt
    SPIDER-EM}. The first $25$ epochs are discarded.}
\label{fig:LogLikelihoodZoom}
\end{figure}

\begin{figure}[t]
\centering
\includegraphics[width=0.4\textwidth]{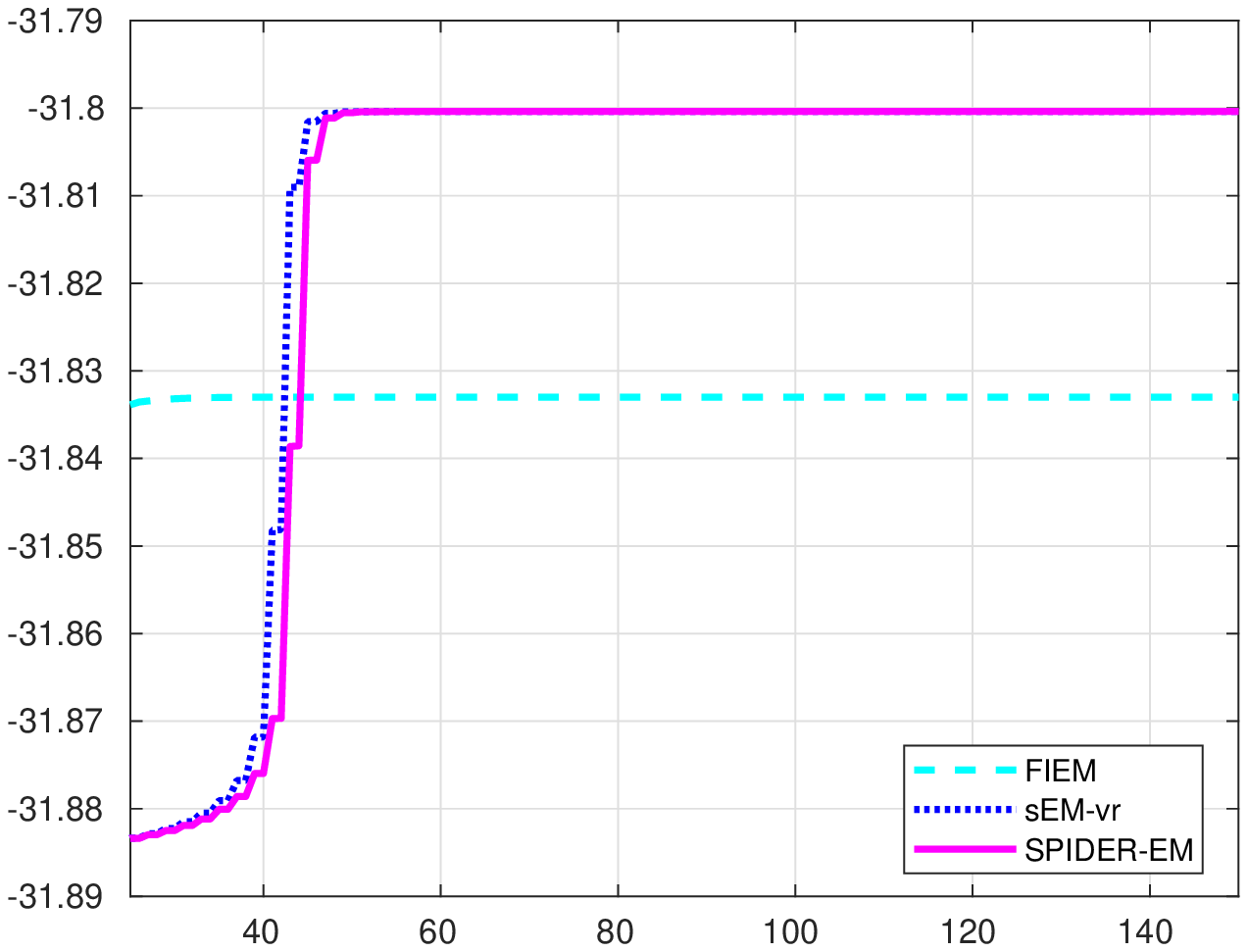}\includegraphics[width=0.4\textwidth]{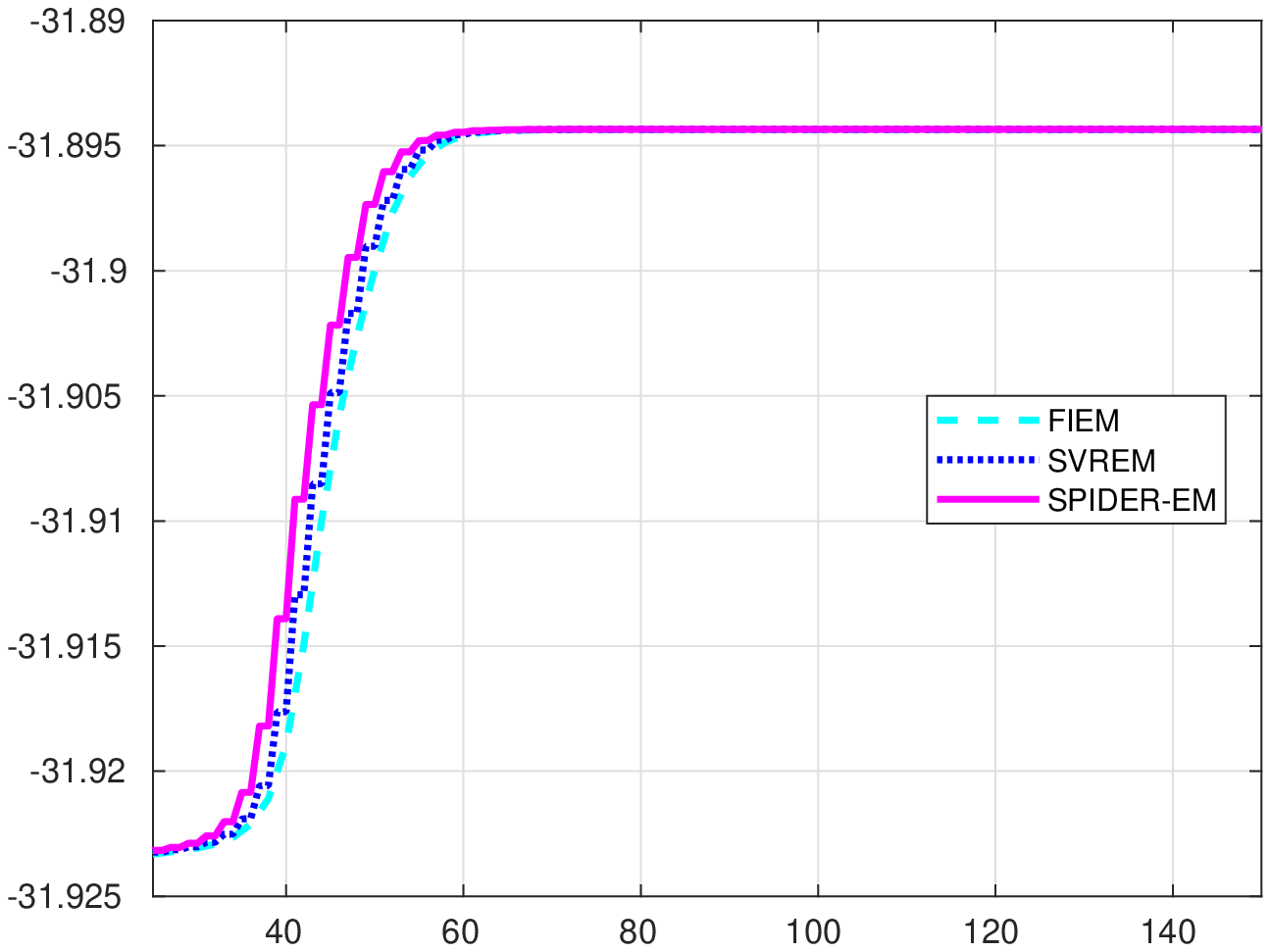}
\caption{The objective function $-\lyap(\hatS_\ell) = -F
  \circ \map(\hatS_\ell)$ against the number of epochs along two
 (left,right) independent runs of {\tt FIEM}, {\tt sEM-vr} and {\tt SPIDER-EM}.
  The first $25$ epochs are discarded.}
\label{fig:LogLikelihoodZoomBis}
\end{figure}

\Cref{fig:EstimWeight} displays the evolution of the $g=12$ iterates
$\{\alpha_1, \ldots, \alpha_g \}$ along a path of many
algorithms. \Cref{fig:EstimEigenvalues} display the evolution of
the $p=20$ eigenvalues of the covariance matrix $\Sigma$ along a path
of many algorithms.  Here again, we observe a strong variability of
{\tt Online EM} when compared to the other algorithms.

\begin{figure}[t]
\centering
  \includegraphics[width=0.8\textwidth]{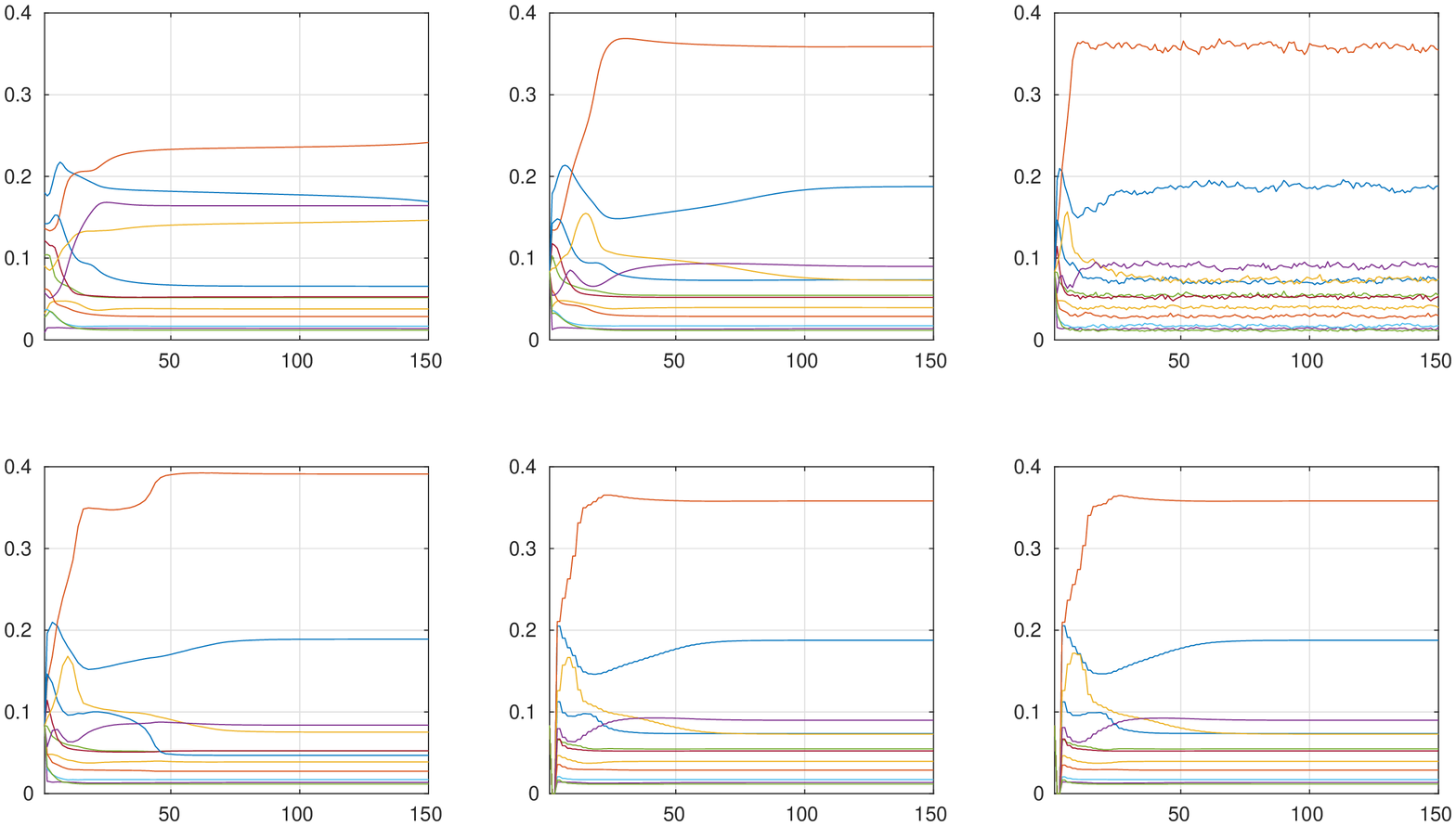}
\caption{Evolution of the $g=12$ iterates $\alpha_k = (\alpha_{k,1}, \ldots,
  \alpha_{k,g})$ against the number of epochs, for {\tt EM}, {\tt
    iEM} and {\tt Online EM} on the top from left to right; {\tt
    FIEM}, {\tt sEM-vr} and {\tt SPIDER-EM} on the bottom from left ro right.}
\label{fig:EstimWeight}
\end{figure}

\begin{figure}[t]
\centering
  \includegraphics[width=0.75\textwidth]{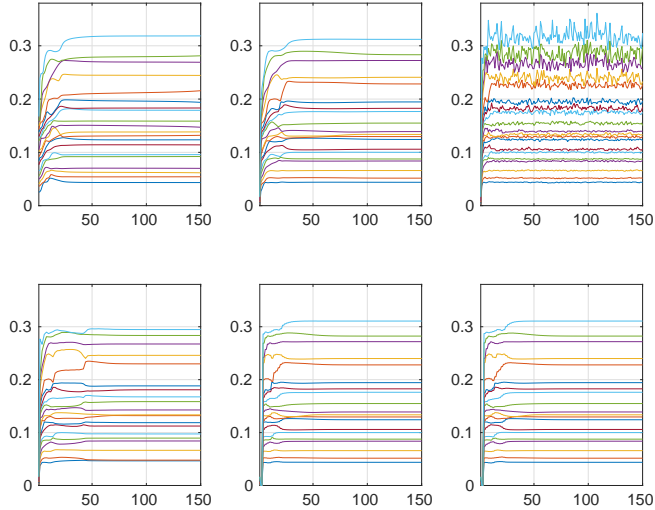}
\caption{Evolution of the $p=20$ eigenvalues of the iterates
  $\{\Sigma_\ell, \ell \geq 0\}$ against the number of epochs $\ell$, for {\tt EM}, {\tt
    iEM} and {\tt Online EM} on the top from left to right; {\tt
    FIEM}, {\tt sEM-vr} and {\tt SPIDER-EM} on the bottom from left ro
  right.}
\label{fig:EstimEigenvalues}
\end{figure}

\Cref{fig:MeanFieldIEM-40path} and \Cref{fig:MeanFieldSVREM-40path}
display $40$ independent realizations of the squared norm of the mean
field $h$ as a function of the number of epochs for different
algorithms. It may be seen that {\tt Online EM} has a strong
variability and {\tt FIEM}, {\tt sEM-vr}, {\tt SPIDER-EM} succeed in
reducing this variability. {\tt FIEM} converges more rapidly than {\tt
  iEM}, and they achieve the same level of accuracy (here not better
than $10^{-6}$). {\tt sEM-vr} and {\tt SPIDER-EM} have the same level
of accuracy, which is most often far smaller than the one reached by
{\tt FIEM} (more than $75 \%$ of the paths reached an accuracy level
of $10^{-10}$ after $150$ epochs). Based on this criterion, we will
definitively advocate the use of {\tt sEM-vr} or {\tt SPIDER-EM} when
compared to {\tt iEM}, {\tt Online EM} and {\tt FIEM}.

\begin{figure*}[h]
  \includegraphics[width=0.33\textwidth]{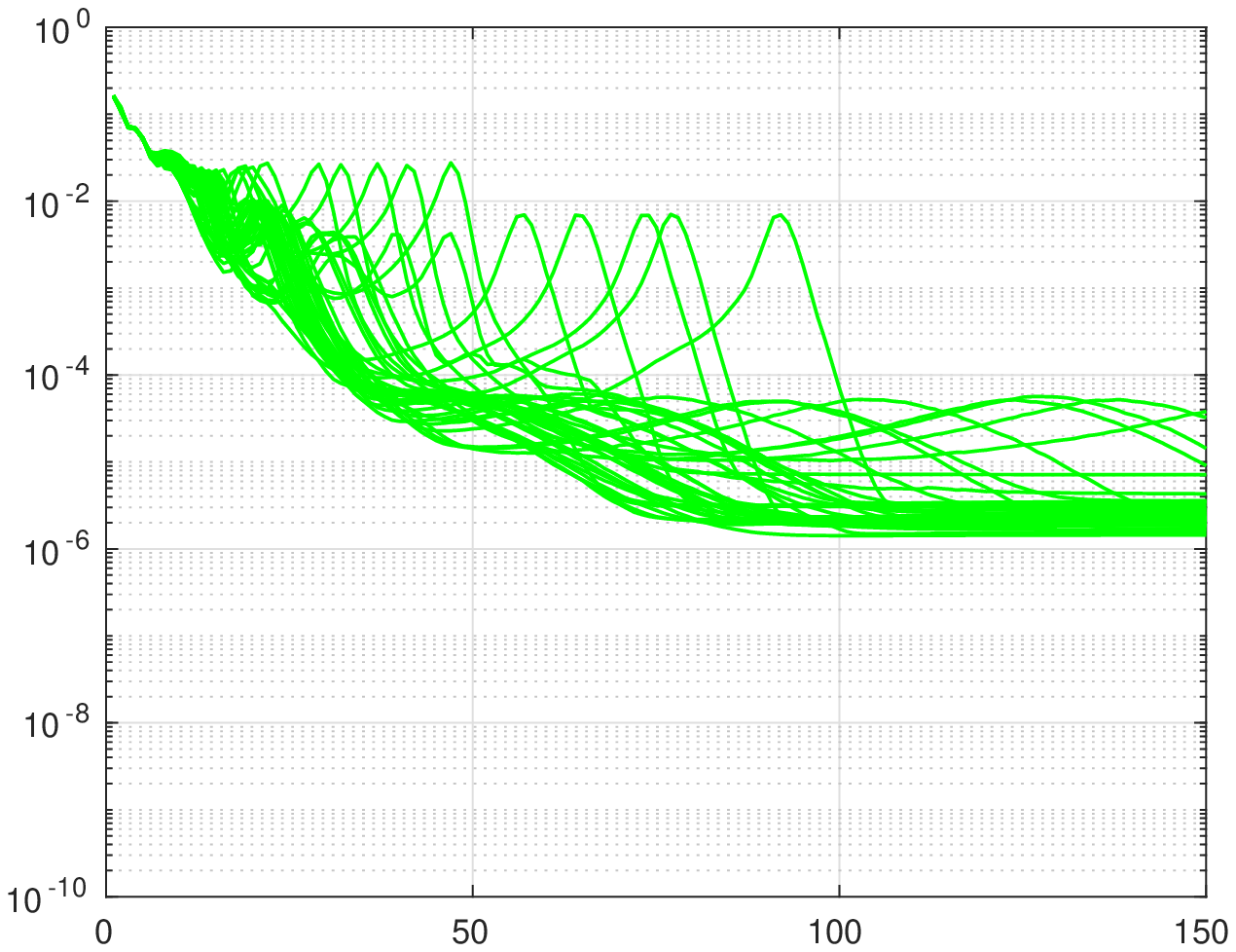} \includegraphics[width=0.33\textwidth]{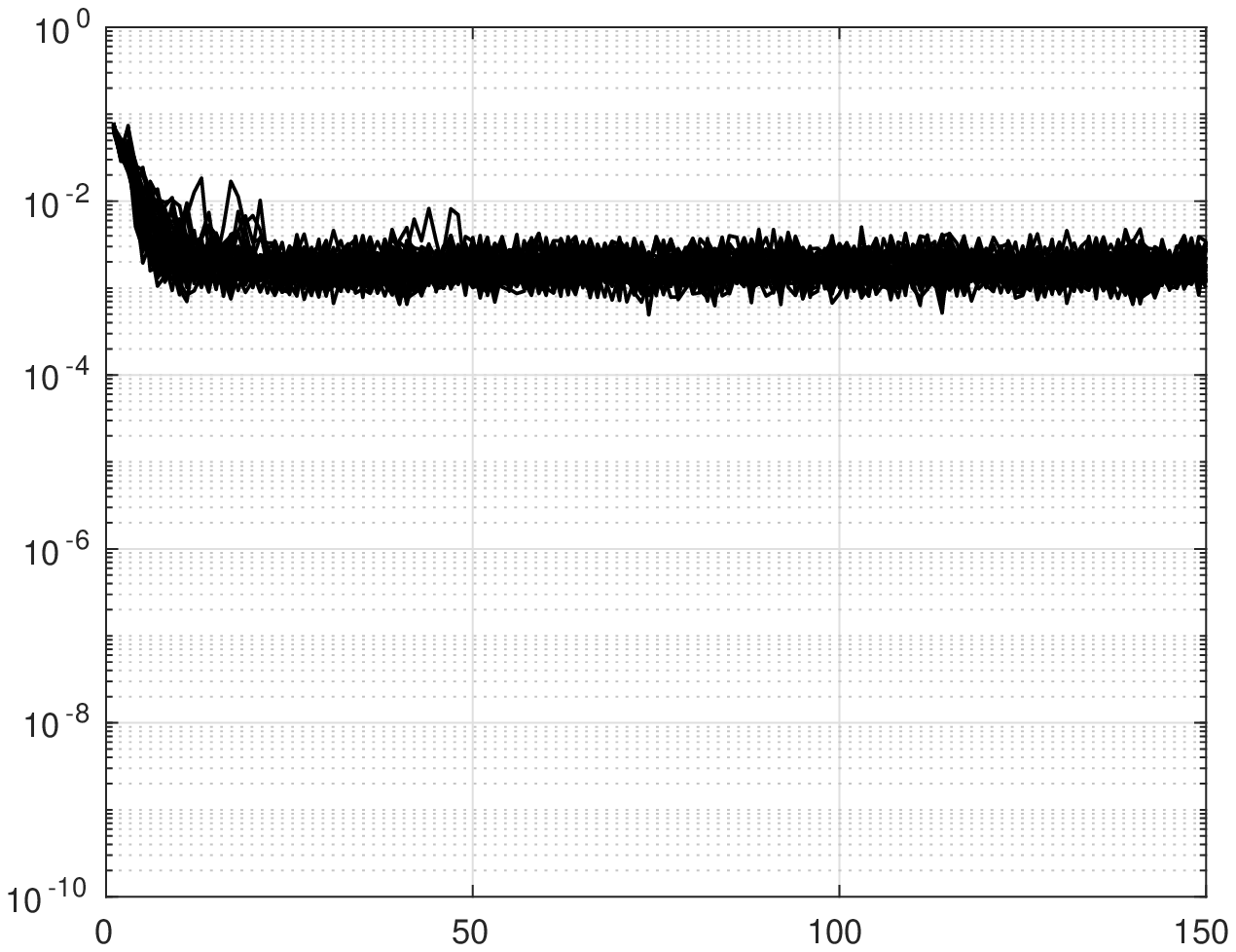} \includegraphics[width=0.33\textwidth]{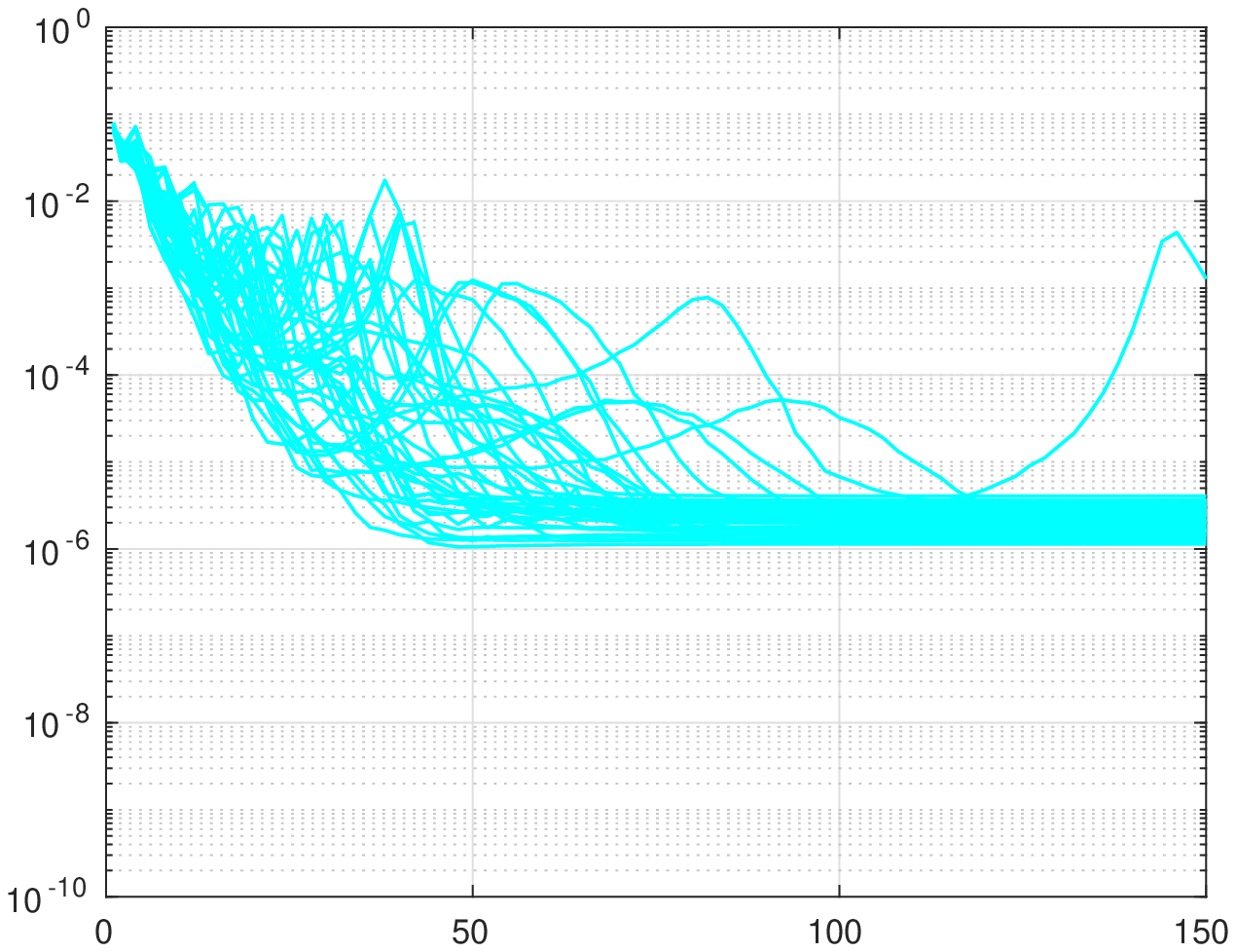}
\caption{[left] We display $40$ independent realizations of the
    squared norm of the mean field $\ell \mapsto \|h(\hatS_\ell)\|^2$ as a
    function of the number of epochs, along a {\tt iEM} path. [center]
    same analysis for {\tt Online EM}. [right] same analysis for {\tt
      FIEM}.}
\label{fig:MeanFieldIEM-40path}
\end{figure*}
\begin{figure*}[h]
\centering
  \includegraphics[width=0.4\textwidth]{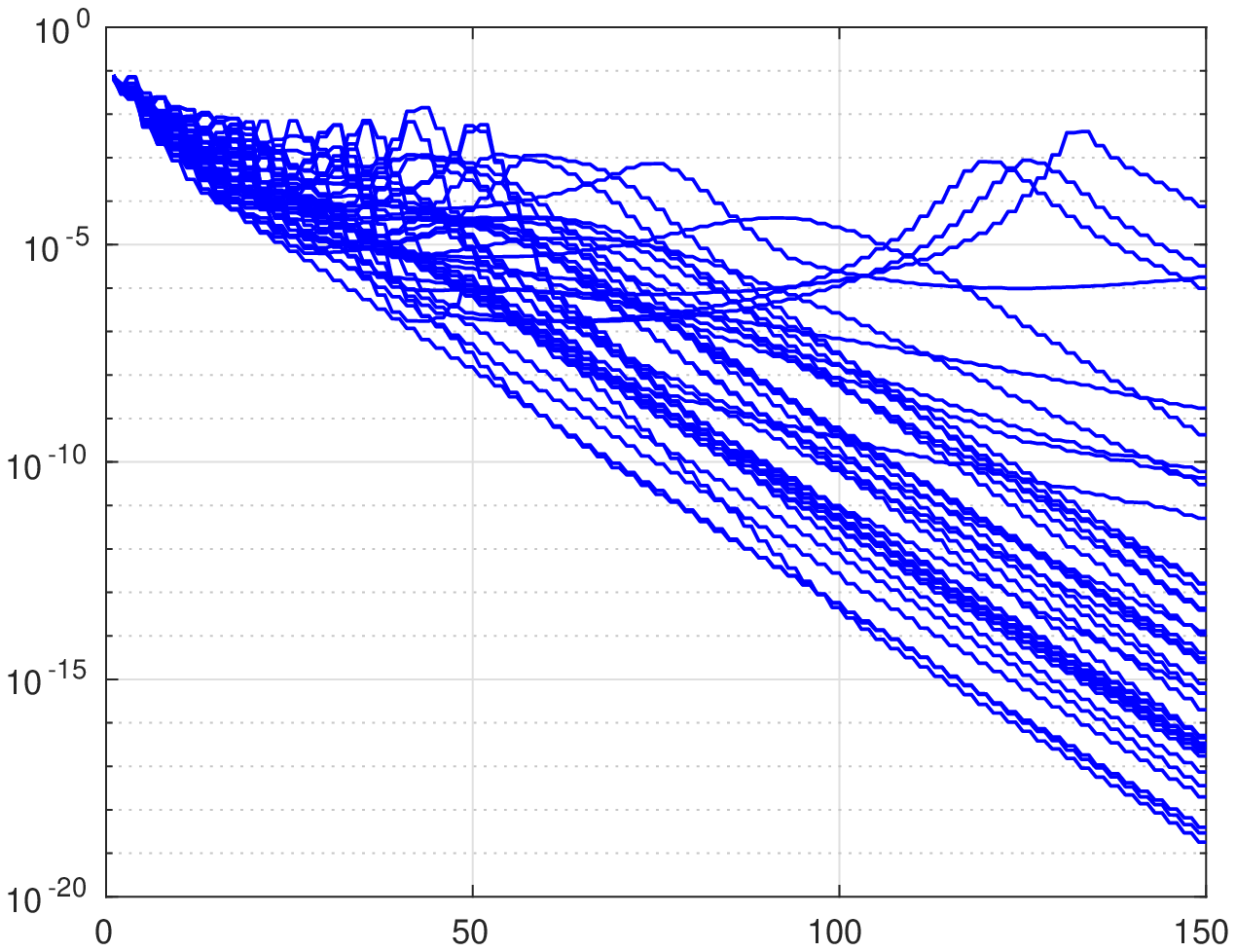} \includegraphics[width=0.4\textwidth]{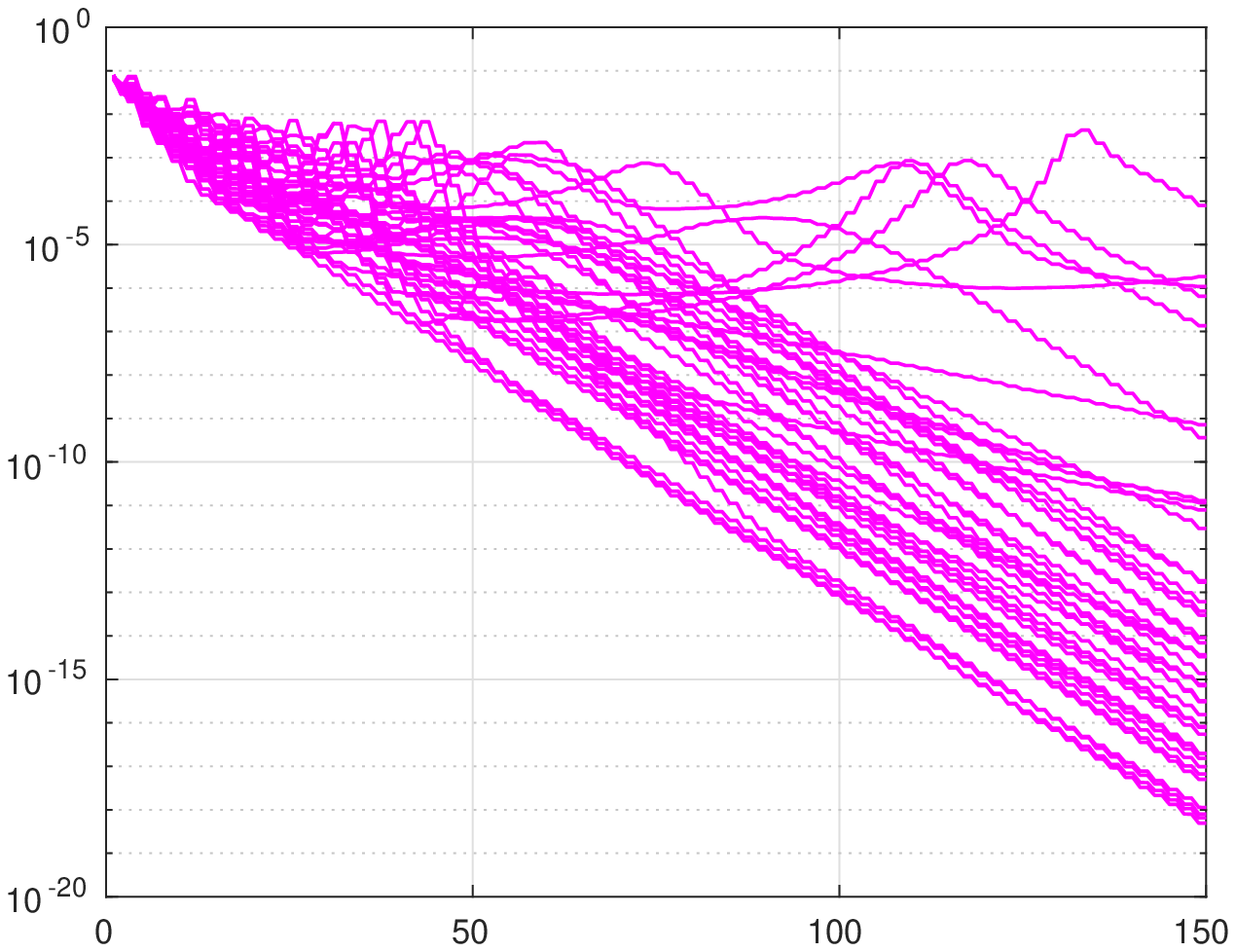}
\caption{[left] We display $40$ independent realizations of the
  squared norm of the mean field $\ell \mapsto \|h(\hatS_\ell)\|^2$ as a
  function of the number of epochs, along a {\tt sEM-vr} path.
  [right] same analysis for {\tt SPIDER-EM}.}
\label{fig:MeanFieldSVREM-40path}
\end{figure*}

\Cref{fig:MeanFieldBoxplot-20epoch} and
\Cref{fig:MeanFieldBoxplot-80epoch} display the boxplots of $40$
independent realizations of $\|h(\hatS_\ell)\|^2$ at time in  $\{20,
40, 60, 80, 110 \}$ epochs for different algorithms. In
\Cref{fig:MeanFieldBoxplot-80epoch}, {\tt Online EM} is not
displayed since it is too large (compare the third plot on
\Cref{fig:MeanFieldBoxplot-20epoch} and the first one on
\Cref{fig:MeanFieldBoxplot-80epoch}). The quantities $\{
\|h(\hatS_\ell)\|^2, \ell \geq 0 \}$ are the key informations for deriving
the complexity bounds in \Cref{theo:rate:sqrtn}.  The plots below
show again that for small, medium and large values of the number of
epochs $k$, {\tt sEM-vr} and {\tt SPIDER-EM} provide the best results.
\begin{figure*}[h]
  \includegraphics[width=0.33\textwidth]{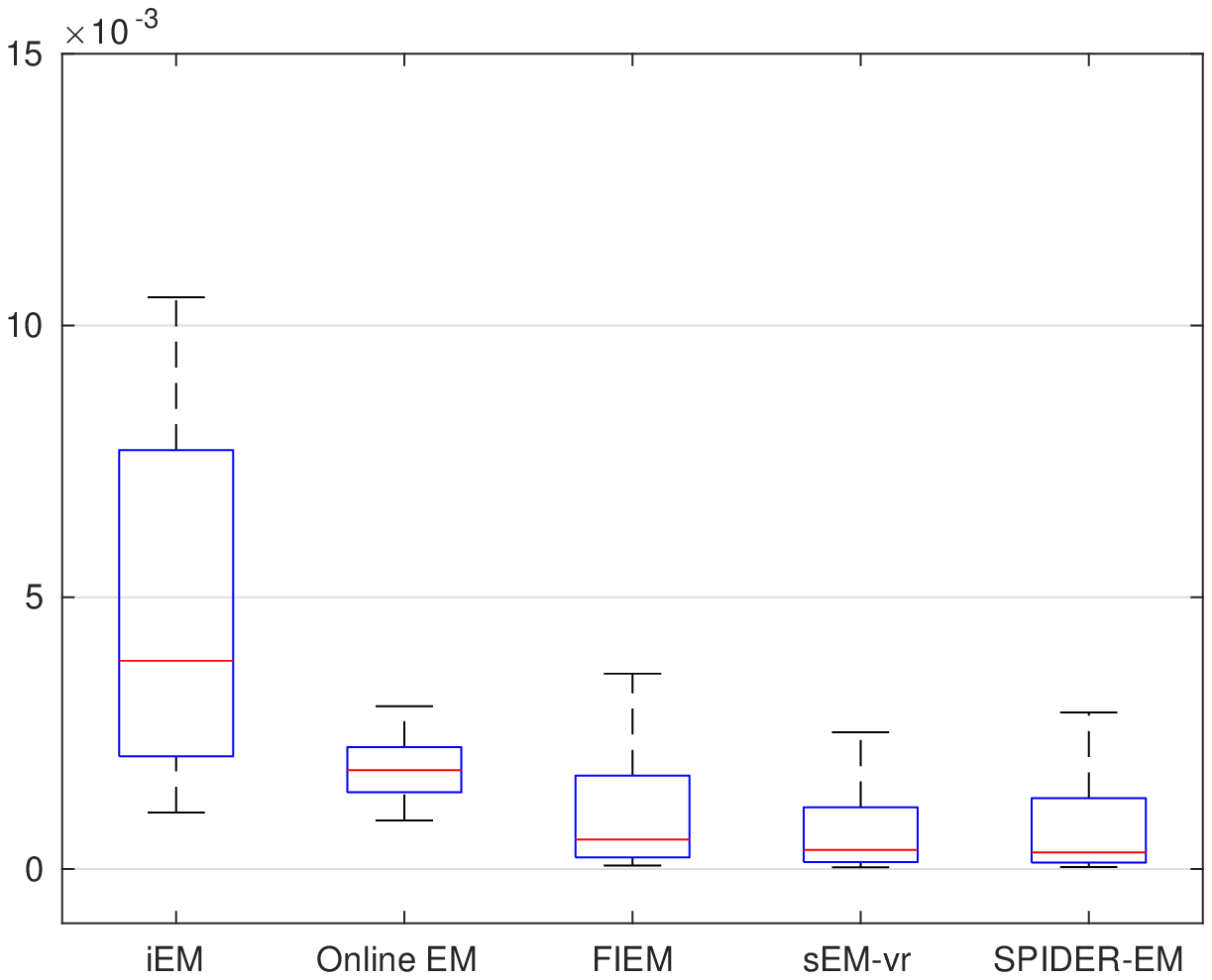}
  \includegraphics[width=0.33\textwidth]{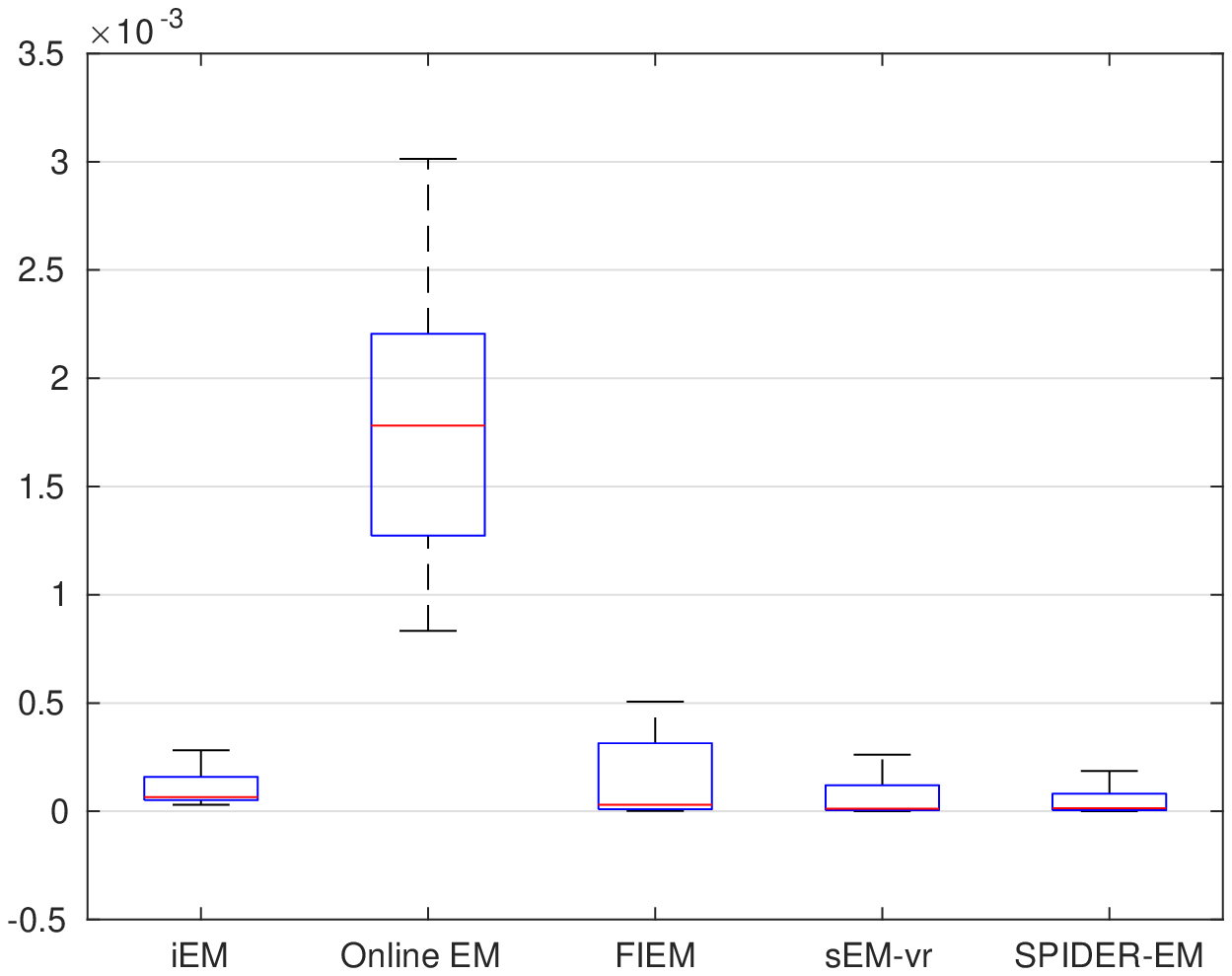}
  \includegraphics[width=0.33\textwidth]{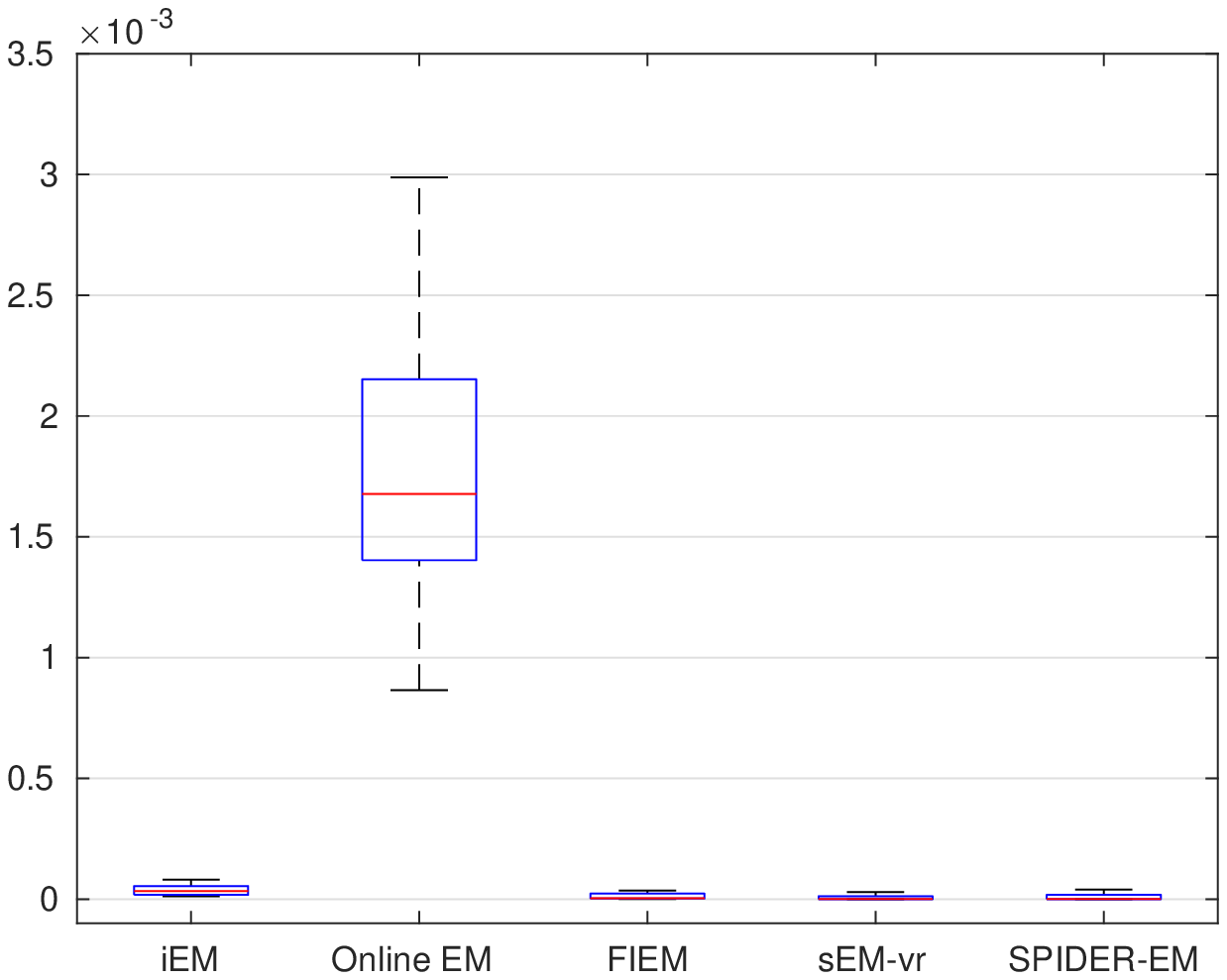}
\caption{Boxplots of $40$ independent points of $\|h(\hatS_\ell)\|^2$
  [left] at time $20$ epochs; [center] at time $40$ epochs; [right]
  at time $60$ epochs. The outliers are removed. }
\label{fig:MeanFieldBoxplot-20epoch}
\end{figure*}
\begin{figure*}[h]
 \includegraphics[width=0.33\textwidth]{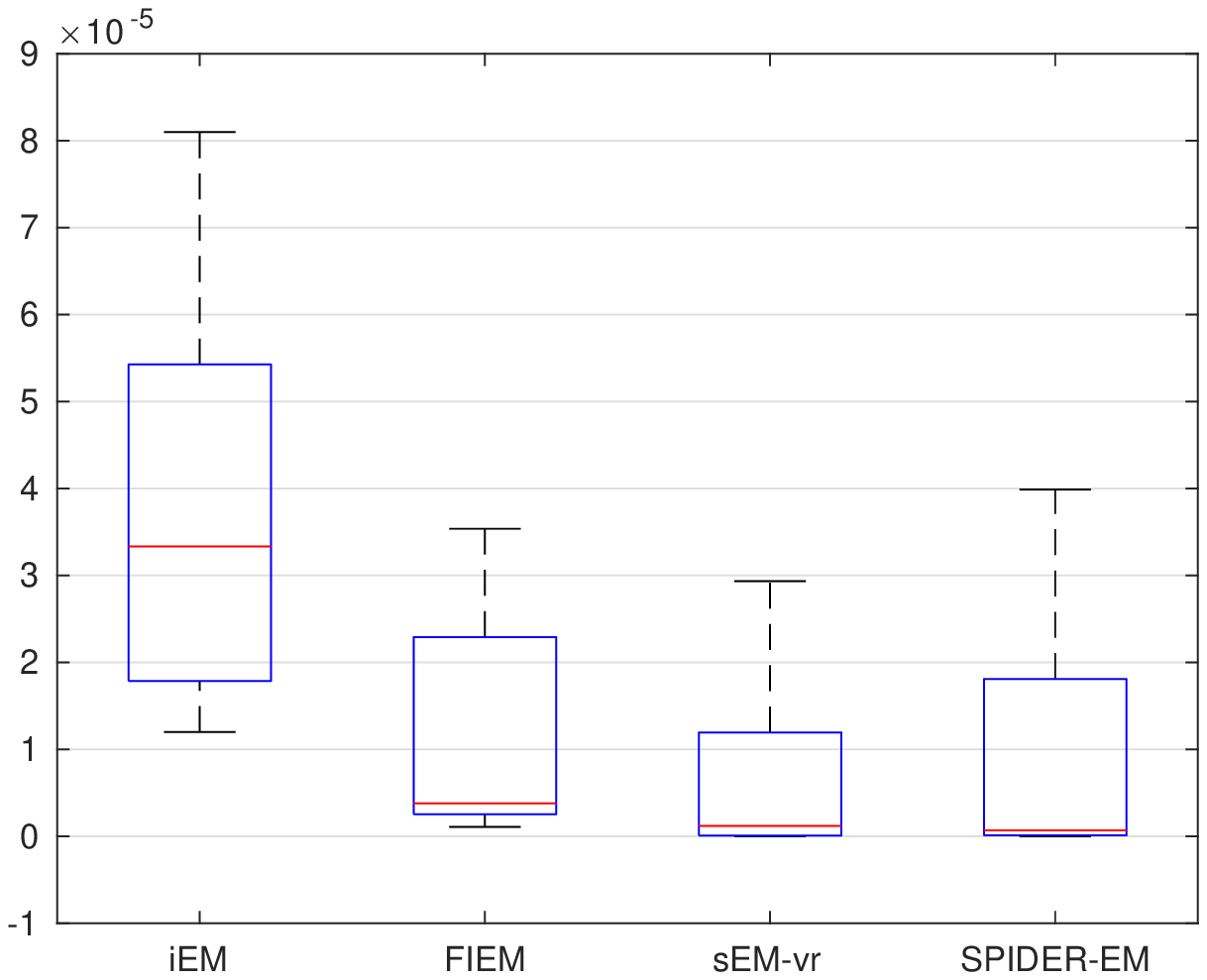}\includegraphics[width=0.33\textwidth]{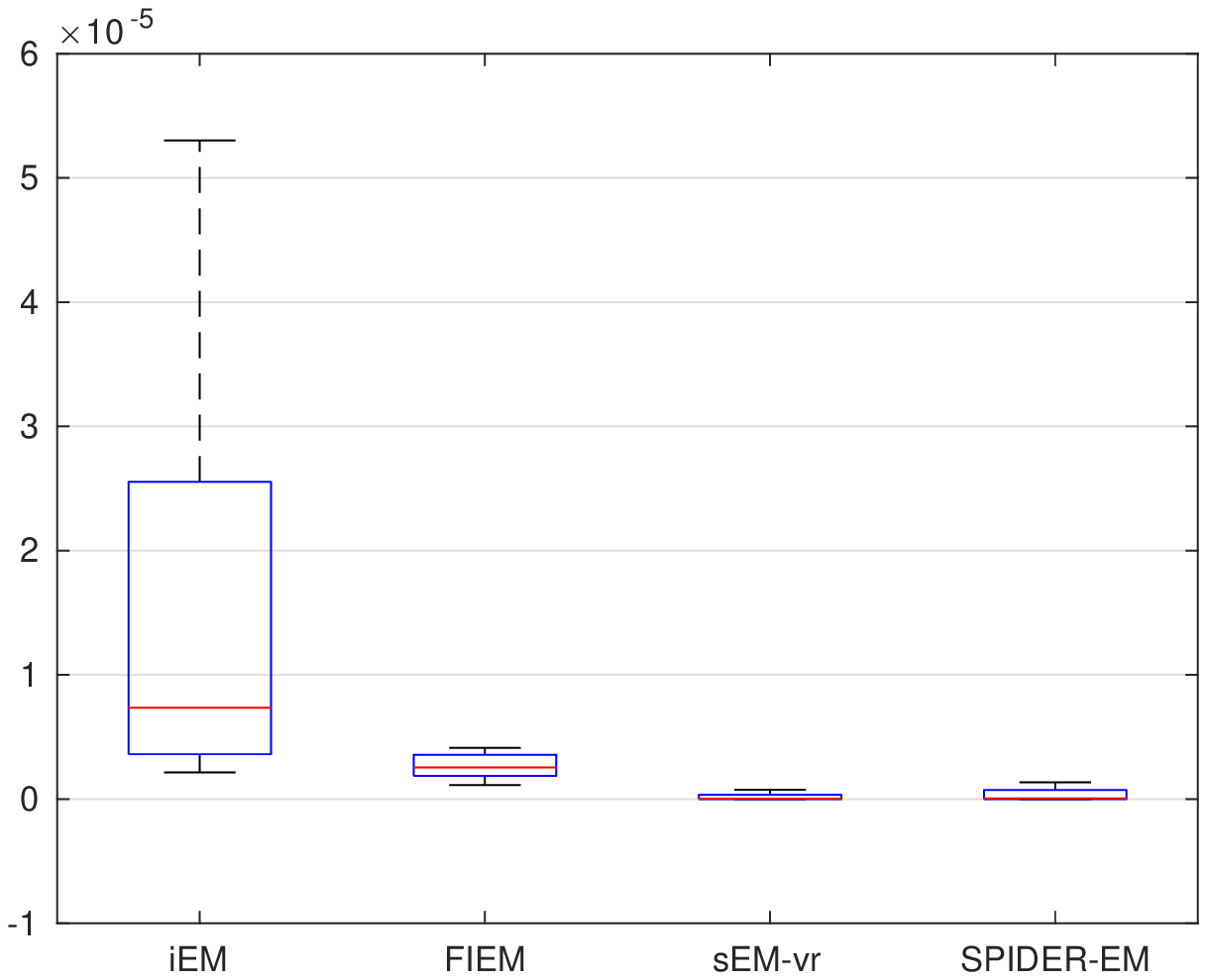}\includegraphics[width=0.33\textwidth]{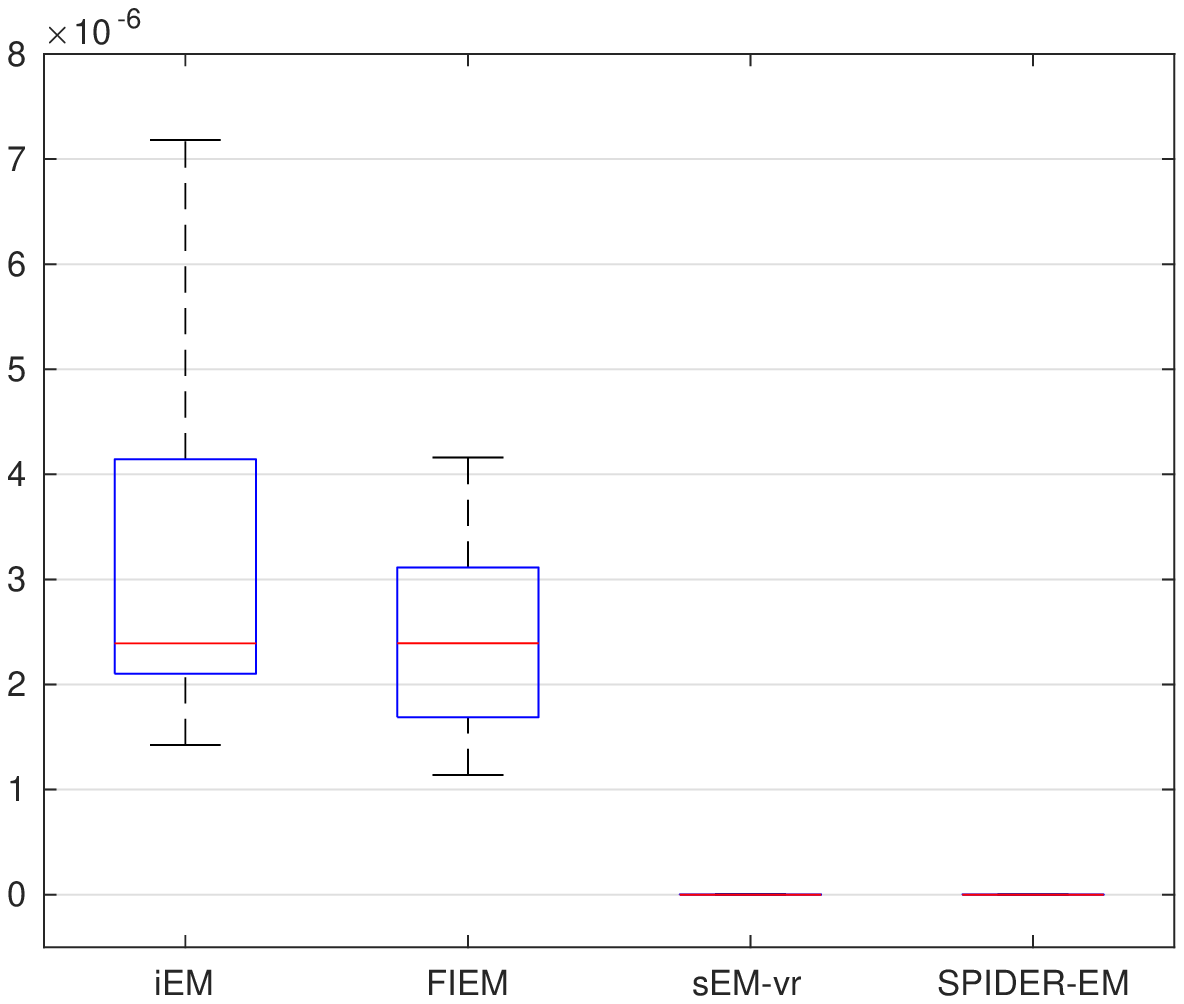}
\caption{Boxplots of $40$ independent points of $\|h(\hatS_\ell)\|^2$
  [left] at time $60$ epochs; [center] at time $80$ epochs; [right]
  at time $110$ epochs. The outliers are removed.}
\label{fig:MeanFieldBoxplot-80epoch}
\end{figure*}

\clearpage


\begin{thebibliography}{27}
\providecommand{\natexlab}[1]{#1}
\providecommand{\url}[1]{\texttt{#1}}
\expandafter\ifx\csname urlstyle\endcsname\relax
  \providecommand{\doi}[1]{doi: #1}\else
  \providecommand{\doi}{doi: \begingroup \urlstyle{rm}\Url}\fi

\bibitem[Balakrishnan et~al.(2017)Balakrishnan, Wainwright, and
  Yu]{balakrishnan2017}
S.~Balakrishnan, M.~J. Wainwright, and B.~Yu.
\newblock {Statistical guarantees for the EM algorithm: From population to
  sample-based analysis}.
\newblock \emph{Ann. Statist.}, 45\penalty0 (1):\penalty0 77--120, 2017.

\bibitem[Benveniste et~al.(1990)Benveniste, M{\'{e}}tivier, and
  Priouret]{benveniste:etal:1990}
A.~Benveniste, M.~M{\'{e}}tivier, and P.~Priouret.
\newblock \emph{{Adaptive Algorithms and Stochastic Approximations}}.
\newblock Springer Verlag, 1990.

\bibitem[Borkar(2008)]{borkar:2008}
V.~S. Borkar.
\newblock \emph{Stochastic approximation}.
\newblock Cambridge University Press, Cambridge; Hindustan Book Agency, New
  Delhi, 2008.
\newblock A dynamical systems viewpoint.

\bibitem[Bottou and Le~Cun(2004)]{bottou:lecun:2003}
L.~Bottou and Y.~Le~Cun.
\newblock Large scale online learning.
\newblock In S.~Thrun, L.~K. Saul, and B.~Sch\"{o}lkopf, editors,
  \emph{Advances in Neural Information Processing Systems 16}, pages 217--224.
  MIT Press, 2004.

\bibitem[B{\"u}hlmann et~al.(2016)B{\"u}hlmann, Drineas, Kane, and van~der
  Laan]{buhlmann2016handbook}
P.~B{\"u}hlmann, P.~Drineas, M.~Kane, and M.~van~der Laan.
\newblock \emph{Handbook of Big Data}.
\newblock Chapman \& Hall/CRC Handbooks of Modern Statistical Methods. CRC
  Press, 2016.
\newblock ISBN 9781482249088.

\bibitem[Capp\'{e} and Moulines(2009)]{cappe:moulines:2009}
O.~Capp\'{e} and E.~Moulines.
\newblock On-line expectation-maximization algorithm for latent data models.
\newblock \emph{J. R. Stat. Soc. Ser. B Stat. Methodol.}, 71\penalty0
  (3):\penalty0 593--613, 2009.

\bibitem[Chen et~al.(2018)Chen, Zhu, Teh, and Zhang]{chen:etal:2018}
J.~Chen, J.~Zhu, Y.~Teh, and T.~Zhang.
\newblock Stochastic expectation maximization with variance reduction.
\newblock In S.~Bengio, H.~Wallach, H.~Larochelle, K.~Grauman, N.~Cesa-Bianchi,
  and R.~Garnett, editors, \emph{Advances in Neural Information Processing
  Systems 31}, pages 7967--7977. 2018.

\bibitem[Defazio et~al.(2014)Defazio, Bach, and
  Lacoste-Julien]{Defazio:bach:2014}
A.~Defazio, F.~Bach, and S.~Lacoste-Julien.
\newblock Saga: A fast incremental gradient method with support for
  non-strongly convex composite objectives.
\newblock In Z.~Ghahramani, M.~Welling, C.~Cortes, N.~D. Lawrence, and K.~Q.
  Weinberger, editors, \emph{Advances in Neural Information Processing Systems
  27}, pages 1646--1654. Curran Associates, Inc., 2014.

\bibitem[Delyon et~al.(1999)Delyon, Lavielle, and Moulines]{Delyon:etal:1999}
B.~Delyon, M.~Lavielle, and E.~Moulines.
\newblock Convergence of a stochastic approximation version of the {EM}
  algorithm.
\newblock \emph{Ann. Statist.}, 27\penalty0 (1):\penalty0 94--128, 1999.
\newblock ISSN 0090-5364.
\newblock \doi{10.1214/aos/1018031103}.
\newblock URL \url{https://doi.org/10.1214/aos/1018031103}.

\bibitem[Dempster et~al.(1977)Dempster, Laird, and Rubin]{Dempster:em:1977}
A.~Dempster, N.~Laird, and D.~Rubin.
\newblock {Maximum Likelihood from Incomplete Data via the EM Algorithm}.
\newblock \emph{J. Roy. Stat. Soc. B Met.}, 39\penalty0 (1):\penalty0 1--38,
  1977.

\bibitem[Fang et~al.(2018)Fang, Li, Lin, and Zhang]{fang:etal:2018}
C.~Fang, C.~Li, Z.~Lin, and T.~Zhang.
\newblock {SPIDER: Near-Optimal Non-Convex Optimization via Stochastic
  Path-Integrated Differential Estimator}.
\newblock In S.~Bengio, H.~Wallach, H.~Larochelle, K.~Grauman, N.~Cesa-Bianchi,
  and R.~Garnett, editors, \emph{Advances in Neural Information Processing
  Systems 31}, pages 689--699. Curran Associates, Inc., 2018.

\bibitem[Fort et~al.(2020)Fort, Gach, and Moulines]{fort:gach:moulines:2020}
G.~Fort, P.~Gach, and E.~Moulines.
\newblock {Fast Incremental Expectation Maximization for non-convex
  optimization: non asymptotic convergence bounds}.
\newblock Technical report, HAL-02617725v1, 2020.

\bibitem[Ghadimi and Lan(2013)]{ghadimi:lan:2013}
S.~Ghadimi and G.~Lan.
\newblock {Stochastic First- and Zeroth-Order Methods for Nonconvex Stochastic
  Programming}.
\newblock \emph{SIAM J. Optimiz.}, 23\penalty0 (4):\penalty0 2341--2368, 2013.

\bibitem[H{\"a}rdle et~al.(2018)H{\"a}rdle, Lu, and Shen]{hardle2018handbook}
W.~H{\"a}rdle, H.~H.-S. Lu, and X.~Shen.
\newblock \emph{Handbook of big data analytics}.
\newblock Springer, 2018.

\bibitem[Johnson and Zhang(2013)]{johnson:zhang:2013}
R.~Johnson and T.~Zhang.
\newblock {Accelerating Stochastic Gradient Descent using Predictive Variance
  Reduction}.
\newblock In C.~J.~C. Burges, L.~Bottou, M.~Welling, Z.~Ghahramani, and K.~Q.
  Weinberger, editors, \emph{Advances in Neural Information Processing Systems
  26}, pages 315--323. Curran Associates, Inc., 2013.

\bibitem[Karimi et~al.(2019{\natexlab{a}})Karimi, Lavielle, and
  Moulines]{karimi:lavielle:moulines:2019}
B.~Karimi, M.~Lavielle, and E.~Moulines.
\newblock {On the Convergence Properties of the Mini-Batch EM and MCEM
  Algorithms}.
\newblock Technical report, hal-02334485, 2019{\natexlab{a}}.

\bibitem[Karimi et~al.(2019{\natexlab{b}})Karimi, Miasojedow, Moulines, and
  Wai]{Karimi:miasojedow:2019}
B.~Karimi, B.~Miasojedow, E.~Moulines, and H.-T. Wai.
\newblock {Non-asymptotic Analysis of Biased Stochastic Approximation Scheme}.
\newblock In \emph{COLT}, 2019{\natexlab{b}}.

\bibitem[Karimi et~al.(2019{\natexlab{c}})Karimi, Wai, Moulines, and
  Lavielle]{karimi:etal:2019}
B.~Karimi, H.-T. Wai, E.~Moulines, and M.~Lavielle.
\newblock {On the Global Convergence of (Fast) Incremental Expectation
  Maximization Methods}.
\newblock In H.~Wallach, H.~Larochelle, A.~Beygelzimer, F.~d'Alch\'{e} Buc,
  E.~Fox, and R.~Garnett, editors, \emph{Advances in Neural Information
  Processing Systems 32}, pages 2837--2847. Curran Associates, Inc.,
  2019{\natexlab{c}}.

\bibitem[Kwedlo(2015)]{kwedlo:2015}
W.~Kwedlo.
\newblock {A new random approach for initialization of the multiple restart EM
  algorithm for Gaussian model-based clustering}.
\newblock \emph{Pattern Anal. Applic.}, 18:\penalty0 757--770, 2015.

\bibitem[McLachlan and Krishnan(2008)]{maclachlan:2008}
G.~McLachlan and T.~Krishnan.
\newblock \emph{{The EM algorithm and extensions}}.
\newblock Wiley series in probability and statistics. Wiley, 2008.

\bibitem[Neal and Hinton(1998)]{Neal:hinton:1998}
R.~M. Neal and G.~E. Hinton.
\newblock {A View of the EM Algorithm that Justifies Incremental, Sparse, and
  other Variants}.
\newblock In M.~I. Jordan, editor, \emph{Learning in Graphical Models}, pages
  355--368. Springer Netherlands, Dordrecht, 1998.

\bibitem[Ng and McLachlan(2003)]{Ng:mclachlan:2003}
S.~K. Ng and G.~J. McLachlan.
\newblock On the choice of the number of blocks with the incremental {EM}
  algorithm for the fitting of normal mixtures.
\newblock \emph{Stat. Comput.}, 13\penalty0 (1):\penalty0 45--55, 2003.

\bibitem[Nguyen et~al.(2020)Nguyen, Forbes, and McLachlan]{nguyen:etal:2020}
H.~Nguyen, F.~Forbes, and G.~McLachlan.
\newblock {Mini-batch learning of exponential family finite mixture models}.
\newblock \emph{Stat. Comput.}, 2020.

\bibitem[Nguyen et~al.(2017)Nguyen, Liu, Scheinberg, and
  Tak\'{a}\v{c}]{nguyen:liu:etal:2017}
L.~M. Nguyen, J.~Liu, K.~Scheinberg, and M.~Tak\'{a}\v{c}.
\newblock Sarah: A novel method for machine learning problems using stochastic
  recursive gradient.
\newblock In \emph{Proceedings of the 34th International Conference on Machine
  Learning - Volume 70}, ICML’17, page 2613–2621. JMLR.org, 2017.

\bibitem[{Reddi} et~al.(2016){Reddi}, {Sra}, {P{\'o}czos}, and
  {Smola}]{reddi:etal:2016}
S.~{Reddi}, S.~{Sra}, B.~{P{\'o}czos}, and A.~{Smola}.
\newblock {Fast Incremental Method for Smooth Nonconvex Optimization}.
\newblock In \emph{2016 IEEE 55th Conference on Decision and Control (CDC)},
  pages 1971--1977, 2016.

\bibitem[Robbins and Monro(1951)]{robbins1951stochastic}
H.~Robbins and S.~Monro.
\newblock A stochastic approximation method.
\newblock \emph{The annals of mathematical statistics}, pages 400--407, 1951.

\bibitem[Wang et~al.(2019)Wang, Ji, Zhou, Liang, and
  Tarokh]{wang:etal:nips:2019}
Z.~Wang, K.~Ji, Y.~Zhou, Y.~Liang, and V.~Tarokh.
\newblock {SpiderBoost and Momentum: Faster Stochastic Variance Reduction
  Algorithms}.
\newblock In H.~Wallach, H.~Larochelle, A.~Beygelzimer, F.~d'~Alch\'{e}-Buc,
  E.~Fox, and R.~Garnett, editors, \emph{Advances in Neural Information
  Processing Systems 32}, pages 2406--2416. 2019.

\end{thebibliography}
\end{document}